\documentclass[11pt]{article} 

\newif\ifarxiv
\arxivtrue        

\newif\ificml
\ifarxiv\icmlfalse\else\icmltrue\fi

\usepackage{microtype}
\usepackage{graphicx}
\usepackage{booktabs} 
\usepackage{tablefootnote}
\usepackage[margin=1in]{geometry}
\usepackage{pifont}
\usepackage[dvipsnames]{xcolor}
\usepackage{algorithm}
\usepackage{amsmath}

\usepackage[noend]{algpseudocode} 
\restylefloat{algorithm}
\usepackage[most]{tcolorbox}
\usepackage{hyperref} 
\hypersetup{
    colorlinks,
    linkcolor={blue!100!black!70},
    citecolor={green!50!black!90},
    urlcolor={blue!75!black}
}

\usepackage{amsthm}
\usepackage{epstopdf}
\graphicspath{{Figures/}}
\usepackage{wrapfig}
\usepackage{enumitem}
\algrenewcommand\algorithmicrequire{\textbf{Input:}}
\algrenewcommand\algorithmicensure{\textbf{Output:}}  

\usepackage{color-edits} 
\usepackage{nicefrac} 
\usepackage{authblk}
\usepackage{float}     
\usepackage{placeins}  
\usepackage{subcaption}
        
\usepackage{makecell}
\usepackage{siunitx}
\usepackage{adjustbox}
\usepackage{dashrule}
\usepackage{amssymb}  
\usepackage{verbatim} 
\usepackage{amsxtra}  

\usepackage{amsfonts}
\usepackage{color}
\usepackage{bbm}

\usepackage{subcaption}

\usepackage{wasysym}

\usepackage{letltxmacro}

\def\DHLhksqrt#1#2{%
\setbox0=\hbox{$#1\sqrt{#2\,}$}\dimen0=\ht0
\advance\dimen0-0.2\ht0
\setbox2=\hbox{\vrule height\ht0 depth -\dimen0}%
{\box0\lower0.4pt\box2}}

\usepackage{mathtools}
\mathtoolsset{showonlyrefs=true}

\newcommand{\isi}[1]{}
\newcommand{\isiincl}[2]{}

\newcommand{\googleincl}[2]{}
\newcommand{\googleinclabs}[3]{}

\usepackage{tikz}
\usetikzlibrary{shapes,arrows,decorations.markings,positioning}

\tikzstyle{plant} = [draw, fill=red!5, rectangle, 
    minimum height=3em, minimum width=6em]
\tikzstyle{block} = [draw, fill=blue!5, rectangle, 
    minimum height=3em, minimum width=6em]
\tikzstyle{sum} = [draw, fill=yellow!10, circle, node distance=1cm]
\tikzstyle{coord} = [coordinate]
\tikzstyle{gain} = [draw, fill=red!5, regular polygon, regular polygon sides=3, shape border rotate=-90]
\tikzstyle{pinstyle} = [pin edge={to-,thick,black}]

\tikzstyle{BitPipe} = [thick, decoration={markings,mark=at position
   1 with {\arrow[semithick]{open triangle 60}}},
   double distance=1.4pt, shorten >= 5.5pt,
   preaction = {decorate},
   postaction = {draw,line width=1.4pt, white,shorten >= 4.5pt}]
   
\usetikzlibrary{fit,backgrounds}

\usepackage{epsfig, graphicx, psfrag}

\usepackage[mathcal]{euscript}
\usepackage[cal=boondox]{mathalfa}

\DeclareMathAlphabet{\pazocal}{OMS}{zplm}{m}{n}

\usepackage{amsthm}

\newtheorem{thm}{Theorem}

\newtheorem{lem}{Lemma}
\newtheorem{claim}{Claim}

\newtheorem{prop}{Proposition}

\theoremstyle{definition}
\newtheorem{defn}{Definition}
\newtheorem*{defn*}{Definition}

\newtheorem*{scheme*}{Scheme}

\theoremstyle{plain}
\newtheorem{remark}{Remark}

\providecommand{\theoremref}[1]{Theorem~\ref{#1}}
\providecommand{\sectionref}[1]{Section~\ref{#1}}

\providecommand{\lemmaref}[1]{Lemma~\ref{#1}}
\providecommand{\appendixref}[1]{Appendix~\ref{#1}}

\providecommand{\figureref}[1]{Figure~\ref{#1}}

\providecommand{\definitionref}[1]
{Definition~\ref{#1}}
\providecommand{\algorithmref}[1]{Algorithm~\ref{#1}}
\providecommand{\propositionref}[1]{Proposition~\ref{#1}}

\newcommand{\iid}{i.i.d.}

\newcommand{\reals}{\mathbb{R}}

\newcommand{\nats}{\mathbb{N}}

\newcommand{\bm}[1]{\mbox{\boldmath{$#1$}}}

\newcommand{\old}[1]{}
\newcommand{\rem}[1]{}

\newcommand{\eps}{{\varepsilon}}

\providecommand{\brI}{\mathbb{I}}

\newcommand{\sx}{{\mathsf{x}}}

\newcommand{\deter}[1]{\mathrm{det}\left(#1 \right)}

\newcommand{\abs}[1]{\left| #1 \right|}
\newcommand{\Norm}[1]{\left\| #1 \right\|}

\providecommand{\comment}[1]{}

\providecommand{\norm}[1]{\Norm{#1}}

\newcommand{\beqn}[1]{\begin{eqnarray}\label{#1}}
\newcommand{\eeqn}{\end{eqnarray}}
\newcommand{\beq}[1]{\begin{equation}\label{#1}}
\newcommand{\eeq}{\end{equation}}

\newcommand{\argmin}{\mathrm{argmin}}

\newcommand{\simiid}{\overset{\mathrm{iid}}{\sim}}

\makeatletter
\newcommand{\vast}{\bBigg@{4}}
\newcommand{\Vast}{\bBigg@{5}}
\makeatother

\newcommand\indep{\protect\mathpalette{\protect\indepT}{\perp}}
\def\indepT#1#2{\mathrel{\rlap{$#1#2$}\mkern3mu{#1#2}}}

\providecommand{\KLDA}[2]{{D_{\alpha} \left( #1 \middle\| #2 \right)}}

\providecommand{\bbE}{\mathbb{E}}

\providecommand{\E}[1]{\bbE \left[ #1 \right]}

\newcommand{\pA}{\pazocal{A}}
\newcommand{\pB}{\pazocal{B}}
\newcommand{\pC}{\pazocal{C}}

\newcommand{\pG}{\pazocal{G}}

\newcommand{\pM}{\pazocal{M}}
\newcommand{\pN}{\pazocal{N}}

\newcommand{\pS}{\pazocal{S}}

\newcommand{\pX}{\pazocal{X}}
\newcommand{\pY}{\pazocal{Y}}
\newcommand{\pZ}{\pazocal{Z}}

\newcommand{\DP}{{\normalfont DP}~}
 
\newcommand{\RDP}{{\normalfont RDP}~} 
\newcommand{\Renyi}{{\normalfont R\'enyi}~} 
\newcommand{\RenyiDP}{{\normalfont \Renyi\hspace{-.3em}-\DP}~}

\renewcommand{\epsilon}{\varepsilon}
\renewcommand{\eps}{\varepsilon}

\allowdisplaybreaks

\makeatletter
\def\@maketitle{%
  \newpage
  \begin{center}%
  \let \footnote \thanks
    {\LARGE \bf \@title \par}%
  \end{center}%
  \par
  \vskip 0.5em}
\makeatother
\title{The Fast Mixing Mechanism for Differential Privacy}

\date{} 

\usepackage{multirow}
\usepackage{booktabs}
\usepackage[authoryear]{natbib}
\usepackage{tikz}
\usepackage{bbm}
\usepackage{xspace}
\usepackage[table]{xcolor}
\definecolor{FastIHMColorA}{RGB}{0,0,128}        
\definecolor{FastIHMColorB}{RGB}{65,105,225}     
\definecolor{FastIHMColorC}{RGB}{0,191,255}      
\newcommand{\FastMix}{\hyperref[alg:fast_mixing]{\texttt{FastMix}}\xspace}
\newcommand{\FastIHM}{\hyperref[alg:private_fast_hessian_sketch]{\texttt{FastIHM}}\xspace}
\newcommand{\MultipleFastMixing}{\hyperref[alg:gauss_fast_mix_eta]{\texttt{MultipleFastMixing}}\xspace}
\newcommand{\filledpentagon}{%
\tikz[scale=0.12, baseline=-0.6ex]
\filldraw[draw=none] (90:1) -- (18:1) -- (-54:1) -- (-126:1) -- (162:1) -- cycle;
}
\begin{document}
\maketitle
\begin{center}
{\large
\begin{tabular}{ccc}
Omri Lev$^{\dagger}$ &
Moshe Shenfeld$^{\star,\ddagger}$ &
Vishwak Srinivasan$^{\dagger,\ddagger}$ \\[0.25em]
\multicolumn{3}{c}{
Katrina Ligett$^{\star}$ \hspace{2.5em} Ashia C.\ Wilson$^{\dagger}$
}
\end{tabular}
}

\vskip 0.75em

\normalsize
\begin{tabular}{c}
$^{\dagger}$Massachusetts Institute of Technology \\[0.25em]
$^{\star}$The Hebrew University of Jerusalem
\end{tabular}
\end{center}

\maketitle 

\begin{abstract}
    Randomized sketching is a central tool for compressing large-scale optimization problems while preserving accuracy. In particular, sketches that are based on structured matrices, such as the Hadamard matrix, can be applied efficiently and often yield solutions that approximate those of the original problem at much lower computational cost. In differential privacy (DP), Gaussian sketching has been used to solve DP linear regression, beginning with \citet{sheffet2017differentially, sheffet2019old} and later refined by \citet{lev2025gaussianmix, lev2026near}. However, although these methods achieve strong utility guarantees, they usually do not improve runtime over classical DP approaches. In this work, we introduce a new DP sketching mechanism based on fast transforms, which, in certain cases, matches the runtime of classical fast sketching methods. We prove state-of-the-art privacy guarantees for this mechanism and show that, in favorable regimes, they match those of the Gaussian sketch up to a constant factor. As an application, we combine this mechanism with recent sketch-based methods for DP linear regression to obtain a new algorithm with strong utility and improved runtime. We establish privacy and accuracy guarantees for this algorithm, yielding, to the best of our knowledge, the first fast method for DP ordinary least squares.
\end{abstract}


\begingroup
  \renewcommand\thefootnote{}%
  \footnotetext{Correspondence: \texttt{omrilev@mit.edu}.%
    \ \ Code available at \url{https://github.com/omrilev1/FastMix}.}%
    \renewcommand\thefootnote{$\ddagger$}%
  \footnotetext{Equal contribution.}%
\endgroup

\section{Introduction}
\label{s:intro}

Modern technologies to collect and process data have led to the development of large-scale datasets that require computationally efficient methods to analyze them. In recent years, \emph{randomized sketching} has proven to be an attractive approach for designing computationally efficient data analysis tools; see \citet{woodruff2014sketching} for an overview.
Roughly speaking, these sketching methods transform the data via random projections, and the data analysis is conducted on this transformed data.
These random projections are developed with two key goals in mind: (a) reducing the size of the data, and consequently the computational cost of the data analysis, and (b) preserving essential structures of the original data, therefore ensuring that the conclusion of the data analysis does not differ vastly from the original data analysis.
This motivates the development of random projections that address both desiderata, and there has been significant progress made towards this recently \citep{drineas2006sampling, sarlos2006improved, ailon2009fast, ailon2013almost, drineas2011faster, pilanci2015randomized, pilanci_hessiansketch, pilanci2017newton, cohen2016optimal_amm, avron2017sharper, derezinski2021newton}. 


Simultaneously, with the advent of these large datasets --- particularly those built using personal data, there has been increased attention towards preserving the privacy of the participating individuals.
Differential privacy (DP) \citep{dwork2006calibrating} provides a rigorous theoretical framework for the development of mechanisms that provably protect privacy, and DP mechanisms have been deployed in practice of late \citep{apple2017learning,  census, facebook2020protecting, snap2022differential}. Here, the tradeoff to keep in mind is between preserving privacy and the quality of the data analysis as compared to the non-private data analysis; for instance, a completely garbled output achieves complete privacy, but fundamentally disregards the data analysis problem.

A relatively underexplored direction is to view the randomness of randomized sketching as a means of providing DP. A notable example is the work of \citet{blocki2012johnson}, which proved that random projections can provide differential privacy, and the work of \citet{sheffet2017differentially, sheffet2019old}, which builds on this work and demonstrates how Gaussian sketching can be applied to private linear regression.
More recent papers \citep{li2025sketched, lev2025gaussianmix, lev2026near} refine the privacy analysis for Gaussian sketching, develop improved algorithms for private linear regression, and demonstrate gains over standard baselines both theoretically and empirically.
Yet these methods do not provide the main computational benefit usually associated with sketching. Specifically, if $\mathsf{S}\in\reals^{k\times n}$ is a dense Gaussian matrix and $X\in\reals^{n\times d}$, computing $\mathsf{S}X$ costs $O(knd)$ operations, often comparable to solving the original problem. In contrast, modern non-private sketching methods are often built using sparse or structured transforms that are both computationally fast and preserve structures of the data, such as the subsampled randomized Hadamard transform (SRHT).
Based on the above discussion, our work is motivated by the natural hypothesis:
\begin{center}
\emph{Can we design a fast DP mechanism that preserves properties essential for data analysis?}
\end{center}

This question was first studied by \citet{upadhyay2014randomness}, who proposed a DP sketching mechanism that can be applied significantly faster than the Gaussian sketch. However, the performance of their method is usually unsatisfactory; while their mechanism works in principle, we demonstrate that it is impractical for typical privacy regimes, incurring a substantial slack relative to the privacy guarantees of the Gaussian sketch. Thus, from a practical standpoint, this question remains open. 

\subsection{Our Contributions}
\label{s:intro_contrib}

We introduce a new family of fast DP sketches, which closes a gap between non-private fast sketching methods and recently studied private sketches.
Our fast DP sketching mechanism (presented in \sectionref{s:fast_mixing}) applies two sketches, first a large fast sketch that compresses the dataset, and second a small Gaussian sketch which provides DP.
This approach is analogous to prior works on non-private fast sketching for approximate matrix multiplication (see \citet[Appendix~A.3]{cohen2016optimal_amm} for example).
For an input matrix \(X \in \reals^{n \times d}\) and a target sketch size \(k\), this scheme is faster than the dense Gaussian sketch by a factor 
\(\widetilde{O}\left(\min\left\{\frac{n}{k}, k\right\}\right)\).
To show that this is indeed private, we provide R\'{e}nyi-DP (a stronger notion of DP) guarantee for our proposed scheme.

We then apply our mechanism to DP ordinary least squares (DP-OLS). Instantiating it with SRHT sketches, we derive privacy and accuracy guarantees quantifying the price of fast, private sketching relative to non-fast constructions. For well-conditioned design matrices, these guarantees match those of state-of-the-art slow methods for DP-OLS. Experiments on large-scale linear regression datasets show that our method achieves state-of-the-art runtime with accuracy close to non-fast baselines. In contrast, \citet{upadhyay2014randomness} do not apply their fast sketching mechanism to DP-OLS, nor do they provide accuracy guarantees comparable to those of state-of-the-art DP-OLS methods. Thus, to our knowledge, our work provides the first algorithm for DP-OLS relying on fast sketching.

\section{Problem Setup and Background}
\label{s:background}

\paragraph{Notation.} 
Random variables are in sans-serif (e.g., \(\mathsf{X}, \mathsf{y}\)), and their realizations in serif (e.g., \(X, y\)). 
The set \(\{1, \ldots, n\}\) is denoted by \([n]\).
The $\ell_{2}$ norm of \(v \in \reals^{d}\) is \(\|v\|\).
Let \(X \in \reals^{n \times d}\).
We use \(\lambda^{X}_{\min}, \lambda^{X}_{\max}, \|X\|_{\mathrm{op}}\) to denote \(\lambda_{\min}(X^{\top}X)\), \(\lambda_{\max}(X^{\top}X)\), and \(\sqrt{\lambda_{\max}^{X}}\) respectively.
The all-zeros column vector in $\reals^{d}$ is \(\vec{0}_{d}\) and the $k \times k$ identity matrix is $\brI_k$.
The Gaussian distribution with mean \(\mu\) and (co)variance \(\Sigma\) is \(\pN(\mu, \Sigma)\), and 
$\pN(0,\brI_{k_1 \times k_2})$ denotes a \(k_{1} \times k_{2}\) random matrix with \iid{} standard normal entries. Throughout, $\widetilde{O}(\cdot)$ hides polylogarithmic factors only in quantities whose dependence in the displayed bound is at most polynomial. The standard basis for \(\reals^{d}\) is denoted by $\left\{e_i\right\}^{d}_{i=1}$, where $e_i \in \reals^{d}$ are the columns of \(\brI_{d}\). Additional notation is in \appendixref{app:notation}.

\paragraph{Problem Setup.}
We are interested in designing differentially private mechanisms that operate over tabular datasets.
These are represented as a matrix $X \in \reals^{n\times d}$. We make a bounded-domain assumption, namely, that there exists a known constant $\textsc{C}_X$ such that $\norm{x_i} \le \textsc{C}_X$ for all $i \in [n]$, as in several prior analyses in both private and non-private settings (e.g., \citet{shamir2015sample, wang_adassp}).
Specifically, from a privacy perspective, this condition can be enforced by ``clipping'' the data. Throughout, we assume that $\textsc{C}_X = 1$ without loss of generality. Since $X'= \nicefrac{X}{\textsc{C}_{\hspace{-0.1em}X}}$ and $\theta'=\textsc{C}_{X}\theta$ give $X\theta = X'\theta'$, any bound for $\textsc{C}_{X}=1$ extends to general $\textsc{C}_X$ by rescaling.
\subsection{Differential Privacy}
\label{s:bckg_DP}

Differential privacy \citep{dwork2006calibrating} is a property of a randomized mechanism \(\pM\) that quantifies how different the output distributions of \(\pM\) are on ``neighboring'' datasets. Here, datasets $X, X'$ are \emph{neighbors} if \(X'\) is formed by removing an element from \(X\), or vice versa.\footnote{For simplicity, we identify row removal with replacement by \(\vec{0}_{d}\), so the dimension remains constant. This notion is sometimes called \emph{zero-out} neighboring.} We denote this relation by \(X \simeq X'\). Throughout, we use the following standard definitions.

\begin{defn}[$(\eps,\delta)$-DP]
\label{defn:eps_delta_DP}
A randomized mechanism $\pM$ is said to satisfy \emph{$(\eps, \delta)$-differential privacy} if for all \(X, X'\) such that $X' \simeq X$ and measurable subsets $\pS \subseteq \mathrm{Range}(\pM)$,
\begin{equation}
    \label{eq:eps_delta_DP}
    \Pr(\pM(X) \in \pS) \leq e^{\eps} \cdot \Pr(\pM(X') \in \pS) + \delta.
\end{equation}
\end{defn}

Our work leverages the \RenyiDP \citep{mironov2017renyi} variant of differential privacy.

\begin{defn}[$(\alpha,r(\alpha))$-RDP]
\label{defn:RDP}  
A randomized mechanism $\pM$ is said to satisfy $(\alpha,r(\alpha))$-\RDP for some $\alpha > 1$ if for all $X, X'$ such that \(X \simeq X'\),
\begin{equation}
    \label{eq:RDP}
    \KLDA{\pM(X)}{\pM(X')} \leq r(\alpha),
\end{equation}
where $\KLDA{\mathbb{P}}{\mathbb{Q}} \hspace{-0.2em}\coloneqq\hspace{-0.2em} \frac{1}{\alpha - 1}\hspace{-0.15em}\log\hspace{-0.15em}\big(\mathbb{E}_{\sx \sim \mathbb{Q}}\left[\mathbb{P}^{\alpha}(\sx)\mathbb{Q}^{-\alpha}(\sx)\right]\hspace{-0.15em}\big)$ denotes the $\alpha$-\Renyi divergence \citep{renyi1961measures}.\footnote{Given random variables $\mathsf{p} \sim \mathbb{P}$ and $\mathsf{q} \sim \mathbb{Q}$, $\KLDA{\mathsf{p}}{\mathsf{q}}$ denotes the $\alpha$-\Renyi divergence $\KLDA{\mathbb{P}}{\mathbb{Q}}$. }
\end{defn}

One can translate \((\alpha, r(\alpha))\)-RDP to classical \((\varepsilon, \delta)\)-DP as demonstrated below.

\begin{prop}[{\citet[Proposition 12]{canonne2020discrete}}]
\label{prop:Renyi_classical_translate}
If $\pM$ satisfies $(\alpha, r(\alpha))$-\RDP, then it also satisfies $(\eps, \delta)$-\DP for any $0 < \delta < 1$, where $\eps = r(\alpha) + \log\left(1 - \frac{1}{\alpha}\right) - \frac{\log(\alpha\delta)}{(\alpha - 1)}$.  
\end{prop}

Both $(\alpha,r(\alpha))$-RDP and $(\eps,\delta)$-DP are preserved under post-processing and degrade gracefully under composition. In particular, post-processing ensures that if a mechanism $f$ satisfies either definition, then so does $g\circ f$ for any (possibly randomized) $g$~\citep{Dwork_AlgFoundationd_DP, mironov2017renyi}.
\subsection{Subsampled Randomized Hadamard Transform}
\label{s:SRHT}

Our fast DP sketch is based on the Subsampled Randomized Hadamard Transform (SRHT), which we discuss in this subsection.
This is a standard tool for fast dimensionality-reduction of large matrices via matrix multiplication \citep{tropp2011improved,ailon2013almost}.
\begin{defn}[SRHT]
    \label{defn:SRHT}
    Let \(n\) be a power of \(2\). An SRHT matrix \(\mathsf{S}_{\text{SRHT}} \in \reals^{k \times n}\), \(k \le n\), is defined as
    \begin{equation}
        \label{eq:SRHT_Def}
        \mathsf{S}_{\mathrm{SRHT}} := \sqrt{\frac{n}{k}}\mathsf{P}_{k\times n}H_n\mathsf{B}_n
    \end{equation}
    where \(\mathsf{B}_{n} = \mathrm{diag}(\mathsf{b}_{1}, \ldots, \mathsf{b}_{n})\) and \(\mathsf{b}_{i}\) are drawn independently from a Rademacher distribution,
    \(H_{n} \in \reals^{n \times n}\) is the normalized Hadamard matrix \citep{walsh1923closed} that satisfies \(H_{n}^{\top}H_{n} = \brI_{n}\), and \(\mathsf{P}_{k \times n}\) is formed by randomly sampling \(k\) rows of \(\brI_{n}\) uniformly without replacement.

\end{defn}

We provide background on Hadamard matrices and properties of SRHT matrices in \appendixref{app:helpers}.
The SRHT is particularly appealing from a computational perspective; computing \(S_{\mathrm{SRHT}}X\) for \(X \in \reals^{n \times d}\) takes \(O(nd \log(k))\) operations \citep{ailon2009fast}, which scales exponentially better in \(k\) than the cost of left multiplying \(X\) by an arbitrary, $k\times n$ dense matrix. Statistically, the SRHT also ensures that the minimal singular value and the norms of the rows are approximately preserved, while controlling the correlations between the columns of the output.
These properties are based on prior results of \citet{pilanci2015randomized} and \citet{cohen2016optimal_amm}, and are used in the analysis of our fast DP sketching mechanism. We discuss these in more detail in \appendixref{app:proof_fihm}. 

\begin{remark}
\normalfont{The ``power of 2'' condition on $n$ in SRHT is due to the standard recursive construction of Hadamard matrices, which applies under this condition. Other fast orthogonal transforms, such as DCT- or DFT-based constructions \citep{avron2010blendenpik, boutsidis2013improved}, remove this restriction and can be used to derive similar results. However, their unequal column norms require extra normalization in our privacy analysis, which degrades the signal strength.}
\end{remark}

\subsection{Differential Privacy via Gaussian Sketches}
\label{s:bckg_gaussmix}

As a prelude to our fast DP sketching mechanism, we first discuss the ``slow'' version which is based on a dense sketching matrix, specifically defined by \(\mathsf{S}_{\mathrm{G}} \sim \pN(0, \brI_{k \times n})\) for a target sketch size \(k\) and input matrix \(X \in \reals^{n \times d}\).
The use of such Gaussian sketches for information-theoretic privacy dates back to the work of \citet{zhou2007}, and \citet{blocki2012johnson} adapted it for differentially-private covariance estimation.
Later work by \citet{sheffet2017differentially,sheffet2019old} studies the Gaussian sketching mechanism in isolation; this analysis was recently sharpened by \citet{lev2025gaussianmix}, who provide a tighter \Renyi differential privacy analysis of the mechanism.
In the following lemma, we present the mechanism and its privacy guarantee from \citet{lev2025gaussianmix}.
\begin{prop}[\citep{lev2025gaussianmix}]
    \label{prop:Privacy_First}
    Let \(X\in\reals^{n\times d}\) such that \(\|x_{i}\| \leq 1\) for all \(i \in [n]\) and let
    \begin{equation}
        \label{eq:base_gauss_mix}
        \pM(X) := \mathsf{g}^{\top}X + \sigma \xi^{\top}~, \quad \text{with} \quad \mathsf{g} \sim \pN(\vec{0}_n, \brI_{n}),~ \xi \sim \pN(\vec{0}_d, \brI_{d}).
    \end{equation}
    Assume that we have knowledge of \(\overline{\lambda}_{\min}\) such that \(\lambda^{X}_{\min} \geq \overline{\lambda}_{\min}\), and that \(\gamma := \sigma^{2} + \overline{\lambda}_{\min} > 1 \).
    Then, for every $\alpha \in (1,\gamma)$, $\pM(X)$ is $(\alpha,\varphi(\alpha;\gamma))$-\RDP, where
    \begin{equation}
    \label{eq:privacy_calib}
        \varphi(\alpha; \gamma) \coloneq \frac{\alpha}{2(\alpha - 1)} \log\left(1 - \frac{1}{\gamma}\right) - \frac{1}{2(\alpha - 1)} \log\left(1 - \frac{\alpha}{\gamma}\right).
    \end{equation}
\end{prop}

\propositionref{prop:Privacy_First} bounds the divergence incurred by multiplying \(X\) by a single Gaussian vector. For a Gaussian sketch \(\mathsf{S}_{\mathrm{G}}\) with \(k\) independent Gaussian rows, each row corresponds to one such release; hence, by composition in RDP, the divergence bound for \(\mathsf{S}_{\mathrm{G}}X\) is \(k\) times the bound in \eqref{eq:privacy_calib}. To interpret the role of $\gamma$, note that it can be viewed as a lower bound on the squared minimal singular value of the augmented matrix $\widetilde{X}:= [X^{\top}, \sigma \brI_d]^{\top}.$ Thus, given any lower bound $\lambda^{X}_{\min} \geq \overline{\lambda}_{\min}$, the term $\sigma\xi$ can be interpreted as lifting the minimal singular value to a desired level $\gamma$. Gaussian sketches can yield state-of-the-art accuracy guarantees for DP linear regression \citep{lev2026near}, improving over classical DP constructions such as \citet{wang_adassp}; however, they do not improve the asymptotic runtime. Indeed, forming $\mathsf{S}_{\mathrm{G}}X$ costs $O(nkd)$ in the worst case. Since typically $k=\Omega(d)$, this is at least $O(nd^{2})$, which is often comparable to solving the underlying optimization problem.
 
\section{The Fast Mixing Mechanism for Differential Privacy}
\label{s:fast_mixing}

\begin{figure*}[t]
    \centering
    \begin{subfigure}[t]{0.37\textwidth}
        \centering
        \includegraphics[width=\linewidth]{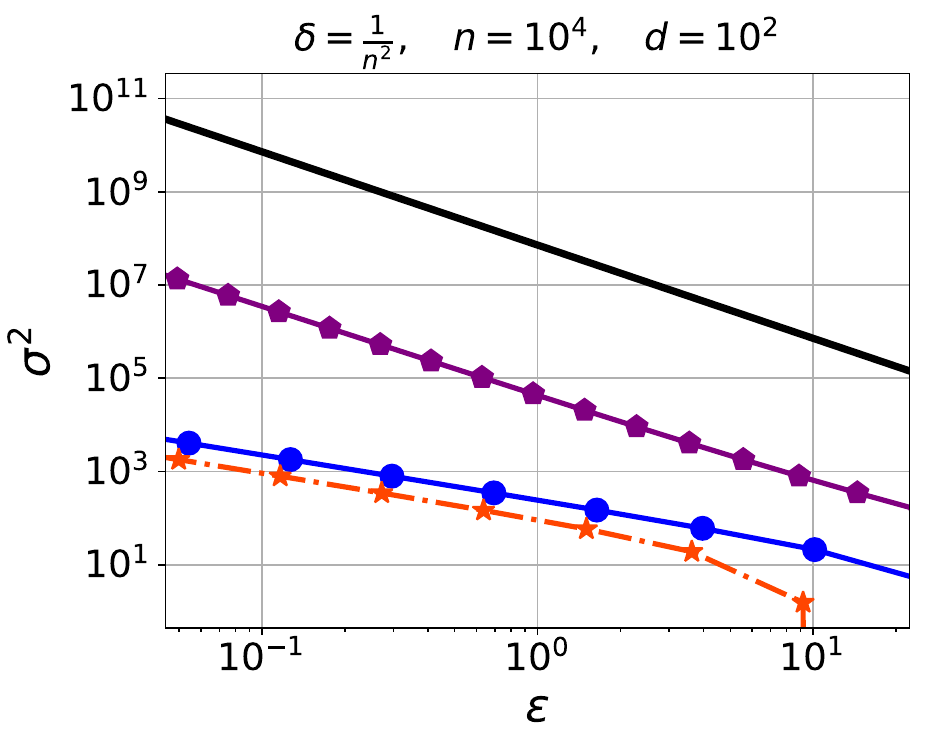}
        \label{fig:compare_jalaj_ours_n_10000}
    \end{subfigure}
    \hskip3mm
    \begin{subfigure}[t]{0.37\textwidth}
        \centering
        \includegraphics[width=\linewidth]{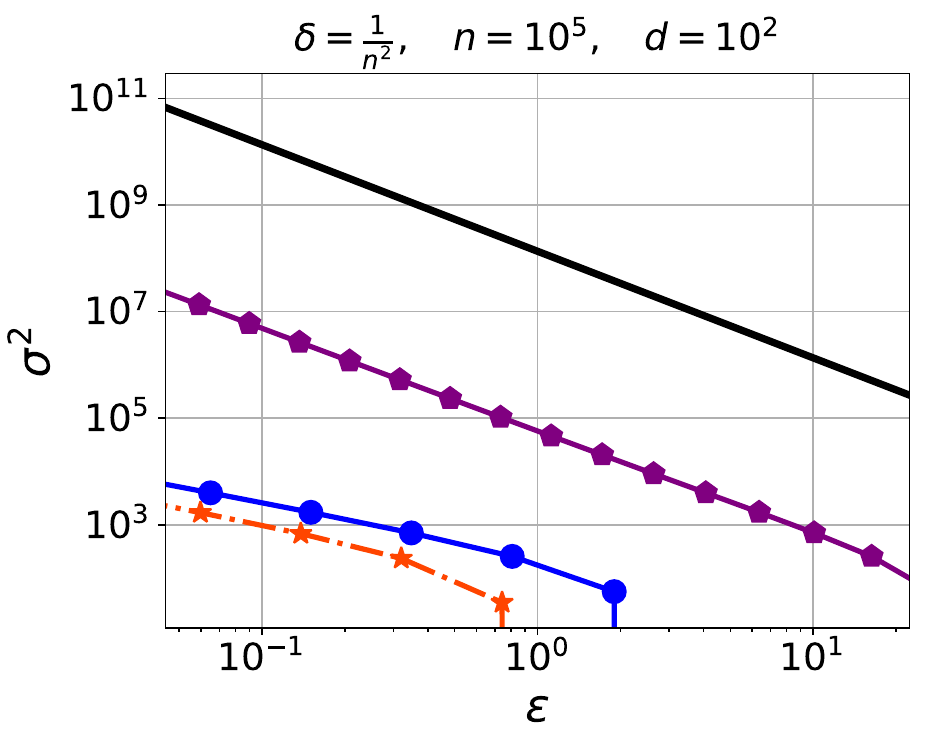}
        \label{fig:compare_jalaj_ours_n_100000}
    \end{subfigure}
    \vspace*{-0.25em}
    {\footnotesize
    \colorlet{myorange}{orange!45!red}
    \colorlet{mypurple}{Plum}
    \colorlet{myblue}{blue}
    \setlength{\tabcolsep}{6pt}
    \renewcommand{\arraystretch}{1.1}
    \begin{tabular}{ll}
    \textcolor{black}{\rule[1pt]{24pt}{1.2pt}} \; Upadhyay '14, Theorem 16
    &
    \textcolor{myblue}{\rule[2pt]{11pt}{1.2pt}}\!\! \textcolor{myblue}{\large $\bullet$}\!\textcolor{myblue}{\rule[2pt]{10pt}{1.2pt}} \; FastMix (ours): $\textsc{C}^2_{S_{\mathrm{f}}X} = 2\textsc{C}^2_{X}$
    \\
    \textcolor{mypurple}{\rule[2pt]{11pt}{1.2pt}}\!\textcolor{mypurple}{$\filledpentagon$}\!\textcolor{mypurple}{\rule[2pt]{11pt}{1.2pt}} \; FastMix (ours): $\textsc{C}^2_{S_{\mathrm{f}}X} \to \infty$
    &
    \textcolor{myorange}{\rule[1.5pt]{10pt}{1.2pt}}\!\! \textcolor{myorange}{$\bigstar$}\!\textcolor{myorange}{\rule[1.5pt]{10pt}{1.2pt}}
    \; GaussMix [Lev et al.\ '25]
    \end{tabular}
    }

    \caption{The additive noise scale $\sigma$ required to attain target $\eps$ for \lemmaref{lem:privacy_lemma} compared with \citet{upadhyay2014randomness} and GaussMix \citep[Lemma.~1]{lev2025gaussianmix} for \(n \in \{10^{4}, 10^{5}\}\).
    We set $\lambda^{X}_{\min} = \frac{n}{8d}\textsc{C}^2_X$ and $\lambda^{S_{\mathrm{f}}X}_{\min} = \frac{5}{6}\lambda^{X}_{\min}$. We plot the noise required by \FastMix with $\textsc{C}^2_{S_{\mathrm{f}}X} = 2\textsc{C}^2_X$, and in the extreme case $\textsc{C}^2_{S_{\mathrm{f}}X}\to\infty$. The curves use the exact characterization instead of \eqref{eq:divergence_upperbound_basic} (see \appendixref{app:full_proof_privacy} for details).}
    \label{fig:compare_jalaj_ours}
\end{figure*}

In this section, we present our key development: a DP mechanism with worst-case runtime matching SRHT for sufficiently large $n$. It builds on the simpler mechanism
\begin{equation}
    \label{eq:sketch_fast_ailon_modified}
    \pM_{S_{\mathrm{f}}}(X) = \mathsf{g}^{\top}S_{\mathrm{f}}X + \sigma \xi^{\top}
\end{equation}
where \(S_{\mathrm{f}} \hspace{-0.15em}\in\hspace{-0.15em} \reals^{k \hspace{-0.05em}\times\hspace{-0.05em} n}\) is a prespecified matrix, \(\mathsf{g} \hspace{-0.15em}\sim\hspace{-0.15em} \pN(\vec{0}_{k}\hspace{-0.05em},\hspace{-0.05em} \brI_{k})\), \(\xi \hspace{-0.15em}\sim\hspace{-0.15em} \pN(\vec{0}_{d}\hspace{-0.05em},\hspace{-0.05em} \brI_{d})\), and \(\sigma \hspace{-0.15em}>\hspace{-0.15em} 0\). 
The following lemma proves the RDP guarantees of \(\pM_{S_{\mathrm{f}}}\), generalizing the result of \citet[Lemma 1]{lev2025gaussianmix}.

\begin{lem} 
    \label{lem:privacy_lemma}
    Let \(X\) be a dataset such that \(\|x_{i}\| \leq 1\) for all \(i \in [n]\).
    Suppose \(S_{\mathrm{f}} \in \reals^{k \times n}\) is a matrix such that \(\|S_{\mathrm{f}}e_{i}\| \leq 1\) for all \(i \in [n]\).
    Define \(\textsc{C}_{S_{\mathrm{f}}X}^{2} \coloneq 1 + 2 \cdot \underset{i \in [n]}{\max}\  \|X^{\top}S_{\mathrm{f}}^{\top}S_{\mathrm{f}}e_i - x_i\|\). 
    Assume that 
    \begin{equation}
        \label{eq:gamma_lower_bound_fast}
        \gamma \coloneq \textsc{C}^{-2}_{S_{\mathrm{f}}X}\cdot \left(\lambda^{S_{\mathrm{f}}X}_{\min} + \sigma^2\right) > \frac{25}{8}.
    \end{equation}
    Then 
    \begin{equation}
        \label{eq:divergence_upperbound_basic}
        \underset{X'\simeq X}{\max} \ \KLDA{\pM_{S_{\mathrm{f}}}(X)}{\pM_{S_{\mathrm{f}}}(X')} \leq \frac{25\alpha}{32\gamma^2}, \ \ \mathrm{for \ all} \ \alpha \in \left(1, \frac{8\gamma}{25}\right).
    \end{equation}
\end{lem}

\textbf{Proof Sketch} \hskip3mm The proof follows by calculating the $\alpha$-\Renyi divergence between $\pM_{S_{\mathrm{f}}}(X)$ and $\pM_{S_{\mathrm{f}}}(X')$. This is possible since $\mathsf{g}^{\top}S_{\mathrm{f}}X$ and $\mathsf{g}^{\top}S_{\mathrm{f}}X'$ are both zero mean Gaussian vectors with covariances $X^{\top}S^{\top}_{\mathrm{f}}S_{\mathrm{f}}X$ and $X'^{\top}S^{\top}_{\mathrm{f}}S_{\mathrm{f}}X'$.
The proof then follows by using the closed-form formula for the \Renyi divergence between multivariate Gaussian vectors, and by crucially noting that since $X'\simeq X$, the covariances differ by a rank-$2$ update.
We then use the assumption made on $\gamma$ in \eqref{eq:gamma_lower_bound_fast} to derive an upper bound that is independent of $X$.
The full proof is presented in \appendixref{app:proof_of_aux_privacy_lemma}. 

\bigskip 

The quantity \(\textsc{C}_{S_{\mathrm{f}}X}\) can be interpreted as the row bound for \(S_{\mathrm{f}}X\).
Similarly to the Gaussian sketching mechanism, for every $\alpha > 1$ and fixed $\textsc{C}_{S_{\mathrm{f}}X}$, this upper bound decays to zero as $\gamma$ increases. Thus, \eqref{eq:divergence_upperbound_basic} can be made arbitrarily small by increasing $\gamma$, which in turn is done by increasing the noise level $\sigma^2$. Thus, $\gamma$ parallels a similar quantity in \citet[Lemma.~1]{lev2025gaussianmix}, calculated for the matrix $S_{\mathrm{f}}X$ instead of $X$.
The mechanism \(\pM_{S_{\mathrm{f}}}\) can be viewed as applying a premultiplication to the augmented matrix $\big[(S_{\mathrm{f}}X)^{\top}, \sigma \brI_d\big]^{\top}$ by a Gaussian random vector of length \(k + d\).
We also recover the result of \citet[Lemma.~1]{lev2025gaussianmix} when \(S^{\top}_{\mathrm{f}}S_{\mathrm{f}} = \brI_{n}\), in which case \(\pM_{S_{\mathrm{f}}}\) coincides with \(\pM\) in distribution.

\begin{figure*}[t]
    \centering
    \begin{minipage}[t]{\textwidth}
        \begin{algorithm}[H]
        \caption{\FastMix} 
        \begin{algorithmic}[1]
            \Require Dataset $X\in\reals^{n\times d}$, parameters $k_1, \gamma, \tau, \omega$, sketch $S_{\mathrm{f}} \in \reals^{k_2\times n}$.
            \State $Z \leftarrow S_{\mathrm{f}}X, \widehat{m}\leftarrow \max\limits_{i \in [n]} \norm{Z^{\top}S_{\mathrm{f}}e_i - x_i}, \widehat{\Delta}\gets\max\limits_{i, j\in [n]} \abs{e^{\top}_i\left(S_{\mathrm{f}}^{\top}S_{\mathrm{f}}- \brI_n\right)e_j}$ and $\widehat{\Lambda} \gets 2 - \min\limits_{i \in [n]} \norm{S_{\mathrm{f}}e_i}^2$.  \label{line:aux_computes_FastMix} 
            \State $\widetilde{m} \gets \max\left\{0, \widehat{m} + \omega \widehat{\Delta} (\tau - \mathsf{z}_1)\right\}$ for $\mathsf{z}_1 \sim \mathrm{Lap}(0, 1)$.  \label{line:compute_m_tilde}
            \State $\widetilde{\lambda}\gets \max\left\{0, \lambda^{Z}_{\min} - \omega \left(\widehat{\Lambda} + 2\widetilde{m}\right) (\tau - \mathsf{z}_2)\right\}$ for $\mathsf{z}_2 \sim \mathrm{Lap}(0, 1)$. \label{line:compute_lambda_tilde}
             
            \State $\widetilde{\eta} \gets \sqrt{\max\left\{0, \gamma\left(1 + 2\widetilde{m}\right) - \widetilde{\lambda}\right\}}$. \label{line:compute_eta_tilde}
            \State \textbf{Output:} $\mathsf{S}_{\mathrm{G}}Z + \widetilde{\eta}\xi$ \ with $\mathsf{S}_{\mathrm{G}} \sim \pN\left(0, \brI_{k_1\times k_2}\right), \xi \sim \pN\left(0, \brI_{k_1\times d}\right),$ and $\widetilde{\eta}$.
        \end{algorithmic}
        \label{alg:fast_mixing}
        \end{algorithm}
    \end{minipage}
\end{figure*}

The simpler mechanism \(\pM_{S_{\mathrm{f}}}\) is the basis for our proposed fast mixing mechanism \FastMix which we discuss next.
While \eqref{eq:gamma_lower_bound_fast} provides guidance for choosing \(\sigma\), \(\gamma\) is not a private quantity by definition.
A key step in \FastMix is in selecting \(\sigma\) for \(\pM_{S_{\mathrm{f}}}\) privately, which is achieved by private estimation of \(\textsc{C}_{S_{\mathrm{f}}X}^{2}\) and \(\lambda_{\min}^{S_{\mathrm{f}}X}\).
We do so by identifying the sensitivity (i.e., the maximal change for neighboring datasets) of \(X \mapsto \textsc{C}_{S_{\mathrm{f}}X}^{2}\) and \(X \mapsto \lambda_{\min}^{S_{\mathrm{f}}X}\) and apply the Laplace mechanism to ensures privacy (Lines.~\ref{line:compute_m_tilde} and \ref{line:compute_lambda_tilde}).
After this, we calibrate the noise level with these private estimates and apply \(\pM_{S_{\mathrm{f}}}\).

We recapitulate our goals: we are interested in designing a mechanism that is both fast and has a comparable noise budget to the dense Gaussian sketch.
The latter can be satisfied if the transform \(X \mapsto S_{\mathrm{f}}X\) approximately preserves \(\lambda_{\min}^{X}\) and \(\textsc{C}_{X}^{2}\).
Curiously, we find that SRHT provides a means to achieve both of these goals.
First, as discussed a priori, the SRHT is faster than premultiplying \(X\) by a dense Gaussian matrix.
Second, for any $\chi \in (0,1]$, when the target sketch size \(k\) satisfies $k\chi^2 = \Theta\left(\mathrm{rank}(X)\cdot\log^4(n)\right)$ implies that
\begin{equation*}
    \abs{\lambda_{\min}^{S_{\mathrm{f}}X} - \lambda^{X}_{\min}} \leq \chi\cdot \lambda_{\min}^{X} \qquad\text{and}\qquad \textsc{C}_{S_{\mathrm{f}}X}^2 - \textsc{C}_X^2 \leq 2\chi\cdot \textsc{C}_X\norm{X}_{\mathrm{op}}.
\end{equation*}
See \appendixref{app:proof_fihm} for an extended discussion of this.
Our main theorem for \FastMix is stated next.
\begin{thm}
    \label{thm:Privacy_First}
    Let $\delta \hspace{-0.15em}\in\hspace{-0.15em} (\hspace{-0.05em}0,\hspace{-0.05em}1\hspace{-0.05em})\hspace{-0.05em}, \hspace{-0.05em} k_1 \hspace{-0.15em}\geq\hspace{-0.15em} 1$ and $S_{\mathrm{f}}$ matrix satisfying $\norm{S_{\mathrm{f}}e_i} \hspace{-0.15em}\leq\hspace{-0.15em} 1$ for all $i\hspace{-0.15em}\in\hspace{-0.15em} [n]$. Then,
    \begin{enumerate}[leftmargin=1.25em]
        \item For every $\gamma \hspace{-0.2em}>\hspace{-0.2em} \nicefrac{25}{8}$ and $\omega\hspace{-0.2em}>\hspace{-0.2em} 0$, the output of \algorithmref{alg:fast_mixing} with $\tau \hspace{-0.2em}\geq\hspace{-0.2em} \log(\nicefrac{3}{2\delta})$ is $(\eps(\omega\hspace{-0.05em},\hspace{-0.05em} \gamma\hspace{-0.05em},\hspace{-0.05em} k_1\hspace{-0.05em},\hspace{-0.05em} \delta)\hspace{-0.05em},\hspace{-0.05em}\delta)$-\DP with respect to $X$ where 
        \begin{align}
            \label{eq:eps_bound_theorem1}
            \eps(\omega, \gamma, k_1, \delta) &\leq \frac{2}{\omega} + 
                \begin{cases}
                    \frac{25k_1}{32\gamma^2}+2\sqrt{\frac{25k_1\log(\nicefrac{3}{\delta})}{32\gamma^2}},&\log(\nicefrac{3}{\delta})\leq \left(\frac{8\gamma}{25} - 1\right)^2\cdot \frac{25k_1}{32\gamma^2},\\
                    \frac{k_1}{4\gamma}+\frac{25\log(\nicefrac{3}{\delta})}{8\gamma-25},&\mathrm{otherwise}.
                \end{cases}
        \end{align}
        \item For $S_{\mathrm{f}}\in \mathrm{SRHT}(k_2, n)$, \algorithmref{alg:fast_mixing} runs in $O\left(nd\log(n) + k_1k_2d + k_2 d^2 + d^3\right).$
    \end{enumerate}
\end{thm}

\textbf{Proof Sketch} \hskip3mm 
We note that \FastMix comprises three private-mechanism calls: calculating $\widetilde{m}$, calculating $\widetilde{\lambda}$, and computing the output $\mathsf{S}_{\mathrm{G}}Z + \widetilde{\eta}\xi$. First, note that $\widetilde{m} \geq \widehat{m}$ and $\widetilde{\lambda} \leq \lambda^{Z}_{\min}$ w.p. at least $1 \hspace{-0.15em}-\hspace{-0.15em} \nicefrac{2\delta}{3}$, which holds since $\tau \geq \log(\nicefrac{3}{2\delta})$ and since $\mathsf{z}_1$ and $\mathsf{z}_2$ are Laplace variables. Conditioned on this event, note that the sensitivity of $\widehat{m}$ is upper bounded by $\widehat{\Delta}$ (\lemmaref{lem:sensitivity_M}) and the sensitivity of $\lambda^{Z}_{\min}$ is upper bounded by $\widehat{\Lambda} \hspace{-0.15em}+\hspace{-0.15em} 2\widehat{m} \leq \widehat{\Lambda} \hspace{-0.15em}+\hspace{-0.15em} 2\widetilde{m}$ (\lemmaref{lem:sensitivity_F_eigenvalues}). Thus, the privacy properties of the Laplace mechanism apply to the release of $\widehat{m}$, while for the release of $\widetilde{\lambda}$, the privacy guarantees hold since we use a private upper bound on the local sensitivity. Then, since $S_{\mathrm{f}}$ has unit column-norms, the conditions of \lemmaref{lem:privacy_lemma} hold. Multiplying the guarantees of \lemmaref{lem:privacy_lemma} by $k_1$, converting to $(\eps,\delta)$-DP using \propositionref{prop:Renyi_classical_translate},  and combining with the guarantees of the steps that use the Laplace mechanism yields the overall privacy guarantees. The runtime argument follows using the fact that the runtime complexity of sketching with an SRHT costs $O(nd\log(n))$, and using properties of the Hadamard matrix for computing the different quantities efficiently. The full proof, along with the tightest bound on $\eps$ obtained by using closed-form guarantees obtained for \lemmaref{lem:privacy_lemma}, is in \appendixref{app:prof_theorem_1}. 

\begin{remark}
    \normalfont{\FastMix uses the Laplace mechanism to compute $\widetilde{m}$ and $\widetilde{\lambda}$, for slightly better empirical performance. An analogous construction can be obtained with the Gaussian mechanism.}
\end{remark}
\paragraph{Role of $S_{\mathrm{f}}$.}
The matrix $S_{\mathrm{f}}$ is treated as deterministic, and our privacy analysis does not rely on any randomness in its construction.
The quantity \(\widehat{\Delta} = \max_{i \neq j} |e_{i}^{\top}S_{\mathrm{f}}^{\top}S_{\mathrm{f}}e_{j}|\) is referred to as the \emph{coherence} of \(S_{\mathrm{f}}\).
The application of \(S_{\mathrm{f}}\) can be viewed as a computationally efficient pre-processing step that also ``accelerates'' the application of the Gaussian sketch.
Since $S_{\mathrm{f}}$ is assumed to be public, any quantity that depends only on $S_{\mathrm{f}}$ (such as $\widehat{\Delta}$ and $\widehat{\Lambda}$) need not itself be privatized. 
As discussed previously, the SRHT is viable from a computational and privacy standpoint, and additionally satisfies \(\widehat{\Lambda} = 1\).
Additionally, for the SRHT, we can rely on  
classical sketching results to obtain high-probability approximation guarantees for downstream applications, as shown in \sectionref{s:apps}.

\paragraph{Analytical Bound.} \
As shown in \appendixref{app:tcdp_bound}, \lemmaref{lem:privacy_lemma} implies $\eps \hspace{-0.2em}=\hspace{-0.2em} O\hspace{-0.2em}\left(\hspace{-0.2em}\tfrac{\textsc{C}^4_{S_{\mathrm{f}}X}k_1}{(\lambda^{S_{\mathrm{f}}X}_{\min} + \sigma^2)^2} \hspace{-0.2em}+\hspace{-0.2em} \tfrac{\textsc{C}^2_{S_{\mathrm{f}}X}\sqrt{k_1\log(\nicefrac{1}{\delta})}}{\lambda^{S_{\mathrm{f}}X}_{\min} + \sigma^2}\hspace{-0.2em}\right)$ whenever $\lambda^{S_{\mathrm{f}}X}_{\min} + \sigma^2 > a_0$ and $k_1 > a_1 \cdot \log(\nicefrac{1}{\delta})$ for constants $a_0$ and $a_1$. Moreover, for any $\omega \geq \frac{3}{\eps}$, the releases $\widetilde{\lambda}$ and $\widetilde{m}$ contribute at most $\frac{2\eps}{3}$ to $\eps$, so \FastMix matches this bound up to constant degradation in $\eps$. In comparison, the pure Gaussian sketch with $k_1$ rows satisfies $\eps \hspace{-0.2em}=\hspace{-0.2em} O\hspace{-0.2em}\left(\hspace{-0.2em}\tfrac{k_1}{(\lambda^{X}_{\min} + \sigma^2)^2} \hspace{-0.2em}+\hspace{-0.2em} \tfrac{\sqrt{k_1\log(\nicefrac{1}{\delta})}}{\lambda^{X}_{\min} + \sigma^2}\hspace{-0.2em}\right)$ \citep[Section~4]{lev2025gaussianmix}. Thus, \FastMix incurs two degradations: each factor $\tfrac{\sqrt{k_1}}{\lambda^{X}_{\min} + \sigma^2}$ is multiplied by $\textsc{C}^2_{S_{\mathrm{f}}X} \geq 1$, increasing $\eps$; and $\lambda^{X}_{\min}$ is replaced by $\lambda^{S_{\mathrm{f}}X}_{\min}$. As discussed, the latter typically changes $\lambda^{X}_{\min}$ only by a constant multiplicative factor, whereas the former can be as large as $O\hspace{-0.2em}\left(\hspace{-0.2em}\sqrt{\frac{\mathrm{rank}(X)\log^4(n)}{k_2}\hspace{-0.1em}\cdot\hspace{-0.1em}n}\right)$.

\paragraph{Comparison to Existing Literature.} \
To our knowledge, little prior work has studied DP via fast sketches. The closest related mechanism is due to \citet{upadhyay2014randomness}, whose construction also first applies an SRHT, but then uses a sparse Gaussian outer sketch. Their privacy bound scales as $O\hspace{-0.2em}\left(\hspace{-0.2em}\log(\nicefrac{1}{\delta})\sqrt{\tfrac{k\log(\nicefrac{1}{\delta})}{\lambda^{S_{\mathrm{f}}X}_{\min}+\sigma^2}}\hspace{-0.15em}\right)$, where $S_{\mathrm{f}}$ is the SRHT. This is looser than \FastMix\ by a factor \(O\hspace{-0.2em}\left(\hspace{-0.2em}\frac{\log(\nicefrac{1}{\delta})\left(\lambda^{S_{\mathrm{f}}X}_{\min}+\sigma^2\right)^{\nicefrac{1}{2}}}{\textsc{C}^2_{S_{\mathrm{f}}X}}\hspace{-0.2em}\right)\). Their guarantee resembles our analysis in \appendixref{app:guarantee_CFX_infinity}, corresponding to $\textsc{C}^2_{S_{\mathrm{f}}X}\to\infty$, with further constant and $\log(\nicefrac{1}{\delta})$ loss from analyzing directly in $(\eps,\delta)$-DP rather than via \RDP. By \figureref{fig:compare_jalaj_ours}, our approach typically reduces $\sigma^2$ by more than two orders of magnitude.

\paragraph{Choosing $k_2$.} \ While \theoremref{thm:Privacy_First} does not depend explicitly on \(k_2\), it depends on \(k_2\) implicitly through $\lambda^{Z}_{\min}, \widehat{m}$ and $\widehat{\Delta}$. Note that, for matching the guarantee of the Gaussian sketch, our goal is to make $\lambda^{Z}_{\min} \approx \lambda^{X}_{\min}$ and to keep $\widehat{m}$ and $\widehat{\Delta}$ small. Without additional assumptions, choosing $S_{\mathrm{f}} \hspace{-0.25em}\sim\hspace{-0.25em} \mathrm{SRHT}(k_2, n)$ yields that for any \(\chi\hspace{-0.25em}\in\hspace{-0.25em}(0,1]\) s.t. $\frac{k_2\chi^2}{\log^4(n)} \hspace{-0.2em}\geq\hspace{-0.2em} \max\left\{c_0\hspace{-0.2em}\cdot\hspace{-0.2em} \mathrm{rank}(X), c_2\hspace{-0.2em}\cdot\hspace{-0.2em}\log(\nicefrac{c_1}{\varrho})\right\}$, $\lambda^{Z}_{\min} \hspace{-0.2em}\geq\hspace{-0.2em} (1\hspace{-0.2em}-\hspace{-0.2em}\chi) \lambda^{X}_{\min}$ and $\widehat{m} \hspace{-0.2em}\leq\hspace{-0.2em} \frac{\chi}{2} \hspace{-0.2em}\norm{X}_{\mathrm{op}}$ w.p. at least $1\hspace{-0.2em}-\hspace{-0.2em}\varrho$ for constants \(c_0,c_1, c_2\) (\lemmaref{lem:SRHT_preserves_lambda_min}). Moreover, if $k_2\chi^2 \hspace{-0.2em}\geq\hspace{-0.2em} \widetilde{c}_0\hspace{-0.2em}\cdot\hspace{-0.2em} \log\left(\nicefrac{\widetilde{c}_1n^2}{\chi\varrho}\right)\hspace{-0.2em}\cdot\hspace{-0.2em} \log\left(\nicefrac{\widetilde{c}_2n^2}{\varrho}\right)$ then $\widehat{\Delta} \hspace{-0.2em}\leq\hspace{-0.2em} \chi$ w.p. at least $1\hspace{-0.2em}-\hspace{-0.2em}\varrho$ for another three constants $(\widetilde{c}_0, \widetilde{c}_1, \widetilde{c}_2)$ (\lemmaref{lem:nelson_prob}). Consequently, choosing \(k_2\) above these values ensures that $\lambda^{Z}_{\min}$ and $\textsc{C}^2_{S_{\mathrm{f}}X}$ are close to $\lambda^{X}_{\min}$ and $\textsc{C}^2_{X}$ up to factors $1+\chi$ and $1 + \frac{\chi}{2\textsc{C}_X}\norm{X}_{\mathrm{op}}$, respectively, making the degradation relative to \propositionref{prop:Privacy_First} bounded and allowing control over the privacy guarantee of \FastMix. 

\paragraph{Complexity.} \ Setting $k_1\chi^2 \hspace{-0.15em}=\hspace{-0.15em} O(\max\hspace{-0.15em}\left\{d, \log(\nicefrac{1}{\varrho})\right\})$ and $k_2\chi^2 \hspace{-0.15em}=\hspace{-0.15em} \widetilde{O}\hspace{-0.15em}\left(\max\hspace{-0.15em}\left\{\mathrm{rank}(X), \log^2(\nicefrac{1}{\varrho})\right\}\right)$ yields runtime of $\widetilde{O}\hspace{-0.15em}\left(\hspace{-0.15em}nd \hspace{-0.15em}+\hspace{-0.15em} d^3\hspace{-0.15em}+\hspace{-0.15em} \frac{d}{\chi^4}\hspace{-0.15em}\cdot\hspace{-0.15em} \max\hspace{-0.15em}\left\{\hspace{-0.1em}\mathrm{rank}(X), \log^2(\nicefrac{1}{\varrho})\hspace{-0.1em}\right\}\hspace{-0.15em}\cdot\hspace{-0.15em} \max\hspace{-0.15em}\left\{\hspace{-0.1em}d, \log(\nicefrac{1}{\varrho})\hspace{-0.1em}\right\}\hspace{-0.15em}\right)$. As discussed, these values allow controlled degradation relative to \propositionref{prop:Privacy_First}, and thus represent typical choices of $k_1$ and $k_2$. Whenever the dominating term is $nd$\textemdash{}corresponding to the regime where sketching is most beneficial\textemdash{}this matches, up to logarithmic factors, the cost of sketching with SRHT.

\begin{algorithm}[!t]
    \caption{Fast IHM} 
    \begin{algorithmic}[1]
        \Require  Dataset $(X, Y)$; noise parameters $(\gamma, \tau, \sigma,\omega)$; sketch sizes $(k_1,k_2)$; iterations $T$; clip $\textsc{C} > 0$.
        \State Initialize $\widehat{\theta}_0 \gets \vec{0}_d$.
        \State Compute $\eta, \big\{\widetilde{X}_{t}\big\}^{T-1}_{t=0} \gets \MultipleFastMixing(X, 1, \gamma, \tau, \omega, k_2, T)$. 
        \algrenewcommand\algorithmicindent{0.75em}            
        \For {$t=0,\ldots, T-1$: }
            \State Sample $\mathsf{S}_{\mathrm{G}, t} \sim \pN\left(0, \brI_{k_1\times k_2}\right), \xi_t \sim \pN\left(0, \brI_{k_1\times d}\right), \zeta_t \sim \pN(\vec{0}_d, \brI_d)$.
            \State Calculate $\widehat{X}_t = \mathsf{S}_{\mathrm{G}, t}\widetilde{X}_t + \eta \xi_t$. \label{line:fast-sketch-X}  
            \State Calculate $\widetilde{G}_t = X^{\top}\mathrm{clip}_{\textsc{C}}(Y - X\widehat{\theta}_t) - \eta^2 \widehat{\theta}_t + \sigma \zeta_t$. \label{line:tildeg-calc}
            \State Update $\widehat{\theta}_{t+1}=\widehat{\theta}_t + (\tfrac{1}{k_1}\widehat{X}_t^{\top}\widehat{X}_t)^{-1}\widetilde{G}_t$.\label{line:update_step} 
        \EndFor 
        \State \textbf{Output:} $\widehat{\theta}_{T}$
    \end{algorithmic}
    \label{alg:private_fast_hessian_sketch}
\end{algorithm}

\section{Application: Fast DP-OLS}
\label{s:apps}
We now present an application of \FastMix\ to DP-OLS. Using \FastMix, we obtain a new algorithm with state-of-the-art runtimes. We demonstrate empirical runtime improvements of factors up to $\approx 3$ for large-scale datasets. As we show, in certain cases, these runtime improvements come without paying for accuracy. Recall that the non-private OLS baseline is given by
{
\setlength{\abovedisplayskip}{3pt}
\setlength{\belowdisplayskip}{3pt}
\setlength{\abovedisplayshortskip}{2pt}
\setlength{\belowdisplayshortskip}{2pt}
\begin{equation*}
    \theta^{\star} := \underset{\theta}{\argmin} \ \norm{Y - X\theta}^2
\end{equation*}
}
where $X\hspace{-0.2em}\in\hspace{-0.2em} \reals^{n\hspace{-0.05em}\times\hspace{-0.05em} d}$ and $Y\hspace{-0.2em}\in\hspace{-0.2em} \reals^{n}$. Its DP version seeks a regressor $\widehat{\theta}$ that satisfies $(\eps,\delta)$-DP (\definitionref{defn:eps_delta_DP}) while minimizing $\|Y \hspace{-0.25em}-\hspace{-0.25em} X\widehat{\theta}\|^2$. We focus on the regime $n\hspace{-0.25em}\geq\hspace{-0.25em} d$. To obtain a private and fast solution, we adapt the iterative Hessian mixing (IHM) algorithm of \citep{lev2026near} to use \FastMix. Our algorithm is presented in \algorithmref{alg:private_fast_hessian_sketch}; the next theorem provides its privacy and accuracy guarantees.
\begin{thm}
    \label{thm:main_fihm}
    Let $\widehat{\theta}$ be the output of \algorithmref{alg:private_fast_hessian_sketch}. Then, there exist $(\gamma, \tau, \sigma, \omega)$ such that 
    \begin{enumerate}
        \item The output $\widehat{\theta}$ is $(\eps, \delta)$-\DP with respect to $(X,Y)$. 
        \item Let \(\varrho \hspace{-0.15em}\in\hspace{-0.15em} (0\hspace{-0.05em},\hspace{-0.05em}1]\). Then, for every \(\chi \hspace{-0.15em}\in\hspace{-0.15em} \bigl(0,\hspace{-0.1em}\frac{1}{4}\bigr]\) there exist \(k_1\hspace{-0.05em},\hspace{-0.05em} k_2\hspace{-0.05em},\hspace{-0.05em} T,\) and \(\textsc{C}\) s.t. w.p. at least \(1-\varrho\):
        \vspace{-0.2em}
        \begin{equation}
            \label{eq:iterative_fast_hessian_sketch_bound_empirical}
            \norm{Y - X\widehat{\theta}_T}^2 - \norm{Y - X\theta^{*}}^2 \leq \begin{cases}
                \widetilde{O}\left(\gamma_{\mathrm{f}}\cdot \textsc{C}^2_{\mathrm{f}}\left(1 + \overline{m}\right) - \lambda^{X}_{\min}\textsc{C}^2_{\mathrm{f}}\right)    &\mathrm{if} \ \lambda^{X}_{\min} \leq \gamma_{\mathrm{f}}\\
                \widetilde{O}\left(\frac{\gamma^2_{\mathrm{f}}\textsc{C}^2_{\mathrm{f}}}{\lambda^{X}_{\min}\left(1 + \overline{m}\right)^2}\right) & \mathrm{otherwise}
            \end{cases}
        \end{equation}
        where 
        \vspace{-1.0em}
        \begin{equation}   
            \label{eq:asymptotic_fastIHM}
            g^{X} \coloneqq \left(\frac{\eps d}{\sqrt{\log(\nicefrac{1}{\delta})}}\left(\lambda^{X}_{\max} - \lambda^{X}_{\min}\right)\right)^{\nicefrac{2}{3}}, \textsc{C}^2_{\mathrm{f}}:= \textsc{C}^2_{Y} + \norm{\theta^{*}}^2, \overline{m} \coloneqq \chi \left(\sqrt{\lambda^{X}_{\max}} + \frac{\log(\nicefrac{1}{\delta})}{\eps}\right) 
        \end{equation}
        and $\gamma_{\mathrm{f}} \coloneqq\frac{\sqrt{\max\left\{d, \log\left(\nicefrac{1}{\varrho}\right), g^{X}\right\}\log(\nicefrac{1}{\delta})}\left(1 +  \overline{m}\right)}{\eps}$.
        \item Under the previous $\chi, \varrho, k_1, k_2, T$ and $\textsc{C}$, the algorithm runs in 
        \begin{equation}
            \widetilde{O}\left(nd + \frac{d}{\chi^2}\cdot \max\left\{d, g^{X}, \log(\nicefrac{1}{\varrho})\right\}\cdot \left(d + \frac{1}{\chi^2}\cdot \max\left\{\mathrm{rank}(X), g^{X}, \log^2(\nicefrac{1}{\varrho})\right\}\right)\right).
        \end{equation}
    \end{enumerate}
\end{thm}

\textbf{Proof Sketch} \hskip3mm  The proof adapts the analysis of \citet{lev2026near} to \FastMix, where the effective sketch concatenates two sketches. Using properties of SRHT, this roughly incurs a loss of factor $1-\chi$ in $\lambda_{\min}^{X}$ and $1 + \overline{m}$ in $\gamma$. Substituting these in the asymptotic analysis of \citet{lev2026near} yields the stated argument. The runtime complexity is dominated by $T$ steps of \FastMix (splitted between \texttt{MultipleFastMixing} and Line.~\ref{line:fast-sketch-X}), and $T$ iterations of solving an OLS problem of size $k_1\times d$. Summing these and substituting the asymptotic arguments of the parameters $k_1,k_2,\textsc{C}$ and $T$ used for deriving \eqref{eq:asymptotic_fastIHM} yields the stated argument. The full proof is in \appendixref{app:proof_fihm}.

The guarantee \eqref{eq:iterative_fast_hessian_sketch_bound_empirical} parallels \citep[Theorem.~1]{lev2026near}. The main difference is that the bound is multiplied by $1 + \overline{m}$ whenever $\lambda^{X}_{\min} \hspace{-0.15em}\leq\hspace{-0.15em} \gamma_{\mathrm{f}}$, and that the threshold $\gamma_{\mathrm{f}}$ is inflated by the same factor. This is because the Gaussian sketch is applied to $\{\mathsf{S}_{\mathrm{H}, t}X\}_t$ rather than directly to $X$. This degradation can be controlled by increasing $k_2$ at the cost of a runtime increase, as demonstrated in \sectionref{s:sims}. For runtime, whenever $d\chi^2 \geq \max\left\{\mathrm{rank}(X), g^{X}, \log^2(\nicefrac{1}{\varrho})\right\}$, the overall runtime evaluates to $\widetilde{O}\left(nd + \frac{d^3}{\chi^{2}}\right)$, incurring a slight increase relative to the runtime of \FastMix and shows the trade-off in $\chi$. In particular, in many regimes, it is significantly better than the complexity of classical OLS solvers, given by $O\left(nd^2\right)$. Comparing this to the complexity of solving OLS using sketching with SRHT, which is $\widetilde{O}\left(nd + \frac{d^2 \cdot \mathrm{rank}(X)}{\chi^2}\right)$ \citep[Section.~III.A]{pilanci2015randomized}, this complexity is worse whenever $\mathrm{rank}(X) \ll d$, and asymptotically matches whenever $n=\Omega\left(\frac{d^2}{\chi^2}\right)$.

\vspace{-1.0em}
\paragraph{Remark.} \ Another fast DP-OLS algorithm can be obtained by combining \FastMix\ with the sketch-and-solve method of \citet{sheffet2017differentially}. However, sketch-and-solve sketches $(X,Y)$ rather than $X$ alone. Since $\lambda_{\min}^{(X\hspace{-0.1em},\hspace{-0.1em}Y)} \hspace{-0.2em}\leq\hspace{-0.2em} \lambda_{\min}^{X}$ and $\lambda_{\max}^{(X\hspace{-0.1em},\hspace{-0.1em}Y)} \hspace{-0.2em}\geq\hspace{-0.2em} \lambda_{\max}^{X}$,\footnote{This holds since $\lambda_{\max}^{(X\hspace{-0.1em},\hspace{-0.1em}Y)} \hspace{-0.2em}=\hspace{-0.2em} \underset{v}{\max} \ v^{\top}\hspace{-0.1em}(X\hspace{-0.1em},\hspace{-0.1em}Y)^{\top}\hspace{-0.1em}(X\hspace{-0.1em},\hspace{-0.1em}Y)v \hspace{-0.2em}\geq\hspace{-0.2em} \underset{u}{\max} \ u^{\top}\hspace{-0.1em}X^{\top}\hspace{-0.1em}Xu \hspace{-0.2em}=\hspace{-0.2em} \lambda^{X}_{\max}$ and similarly for $\lambda_{\min}^{(X\hspace{-0.1em},\hspace{-0.1em}Y)} \hspace{-0.2em}\leq\hspace{-0.2em} \lambda_{\min}^{X}$.} the resulting utility degradation is typically worse than in the IHM-based construction and, for practical $\eps$, yields substantially lower accuracy. Thus, the practical gain comes not from \FastMix\ alone, but from using it within IHM \citep{lev2026near}.

\paragraph{State-of-the-Art Runtime.}
To the best of our knowledge, our method is the first method for DP-OLS based on fast sketching whose asymptotic runtime matches that of classical non-private fast-sketching solvers in some parameter regimes. Existing approaches fall into several categories. Methods that explicitly form the Gram matrix $X^{\top}X$, such as standard normal-equations solvers and AdaSSP \citep{wang_adassp}, require $O(nd^{2})$ time. Gradient-based methods avoid forming $X^{\top}X$, but their runtimes depend on the condition number. Letting $\kappa_X\coloneqq\lambda_{\max}^{X}/\lambda_{\min}^{X}$, non-accelerated full-gradient descent takes $O\left(\kappa_X\log(\nicefrac{1}{\beta})\right)$ iterations for error $\beta$, whereas accelerated full-gradient descent takes $O\left(\sqrt{\kappa_X}\log(\nicefrac{1}{\beta})\right)$ iterations. Each full gradient costs $O(nd)$ time, giving runtimes of $\widetilde{O}\left(nd\kappa_X\right)$ and $\widetilde{O}\left(nd\sqrt{\kappa_X}\right)$, respectively. These methods can therefore be prohibitively slow on poorly conditioned problems; related private approaches include methods based on objective/output perturbation \citep{chaudhuri2011differentially} and DP-SGD \citep{abadi2016deep}. Methods based on dense Gaussian sketching, including both non-private methods and private counterparts such as those of \citet{sheffet2017differentially,lev2026near}, also generally require $O(nd^{2})$ time, since without assumptions on $X$ applying a sketch with $k=\Theta(d)$ rows to $X$ (or to $(X,Y)$) costs $O(ndk)=O(nd^{2})$. By contrast, classical fast-sketching solvers require $\widetilde{O}\left(nd+\frac{d^2\operatorname{rank}(X)}{\chi^2}\right)$ time, where $\chi\in(0,1]$. Thus, the $nd^{2}$ term becomes $nd$ while it further avoids dependence on the condition number of $X^{\top}X$. Our method likewise (i) avoids explicitly forming $X^{\top}X$, (ii) has no runtime dependence on $1/\lambda_{\min}^{X}$, and (iii) matches the runtime of classical non-private fast-sketching solvers in the stated parameter regimes.

\begin{figure*}[t]
  \centering

    \begin{subfigure}[t]{0.27\textwidth}
    \centering
    \includegraphics[width=\linewidth]{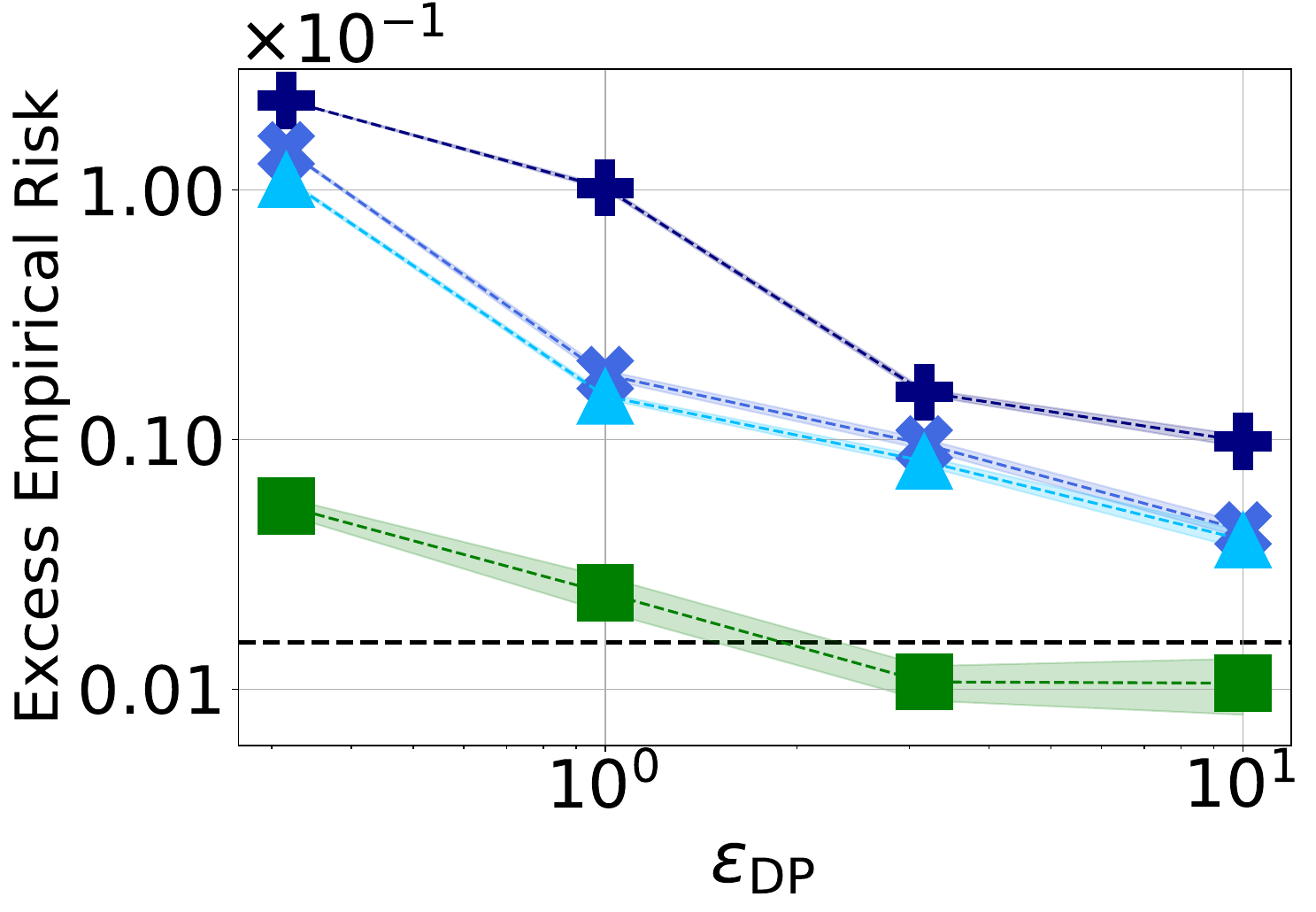}
    \vspace{-1.8em}
    \caption*{\hspace{1.0em} {\footnotesize Black Friday}}
  \end{subfigure}\hspace{0.01\textwidth}%
  \begin{subfigure}[t]{0.27\textwidth}
    \centering
    \includegraphics[width=\linewidth]{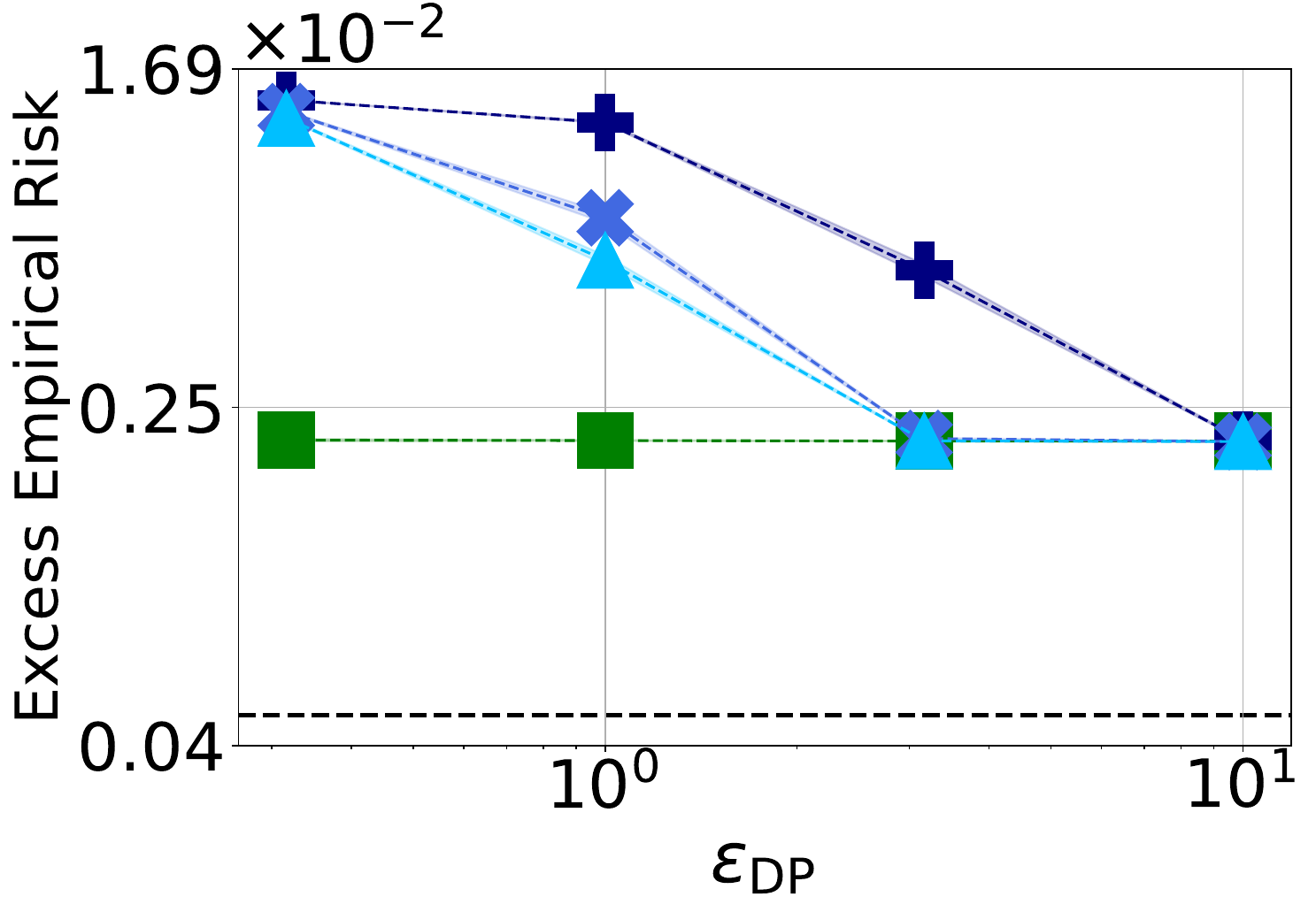}
    \vspace{-1.8em}
    \caption*{\hspace{1.5em} {\footnotesize Beijing}}
  \end{subfigure}\hspace{0.01\textwidth}%
    \begin{subfigure}[t]{0.27\textwidth}
    \centering
    \includegraphics[width=\linewidth]{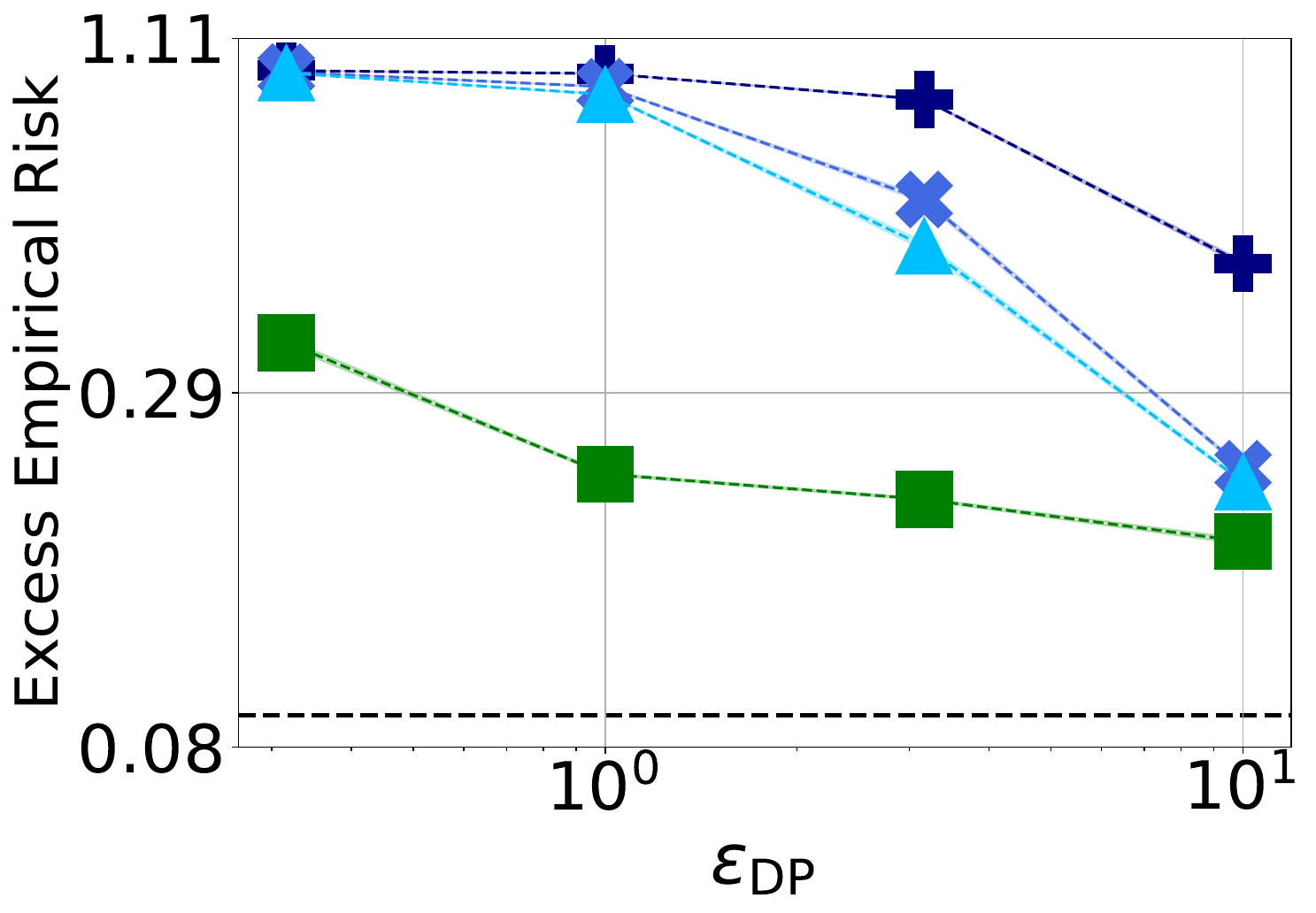}
    \vspace{-1.8em}
    \caption*{\hspace{1.5em} {\footnotesize YearsMSD}}
  \end{subfigure}\hspace{0.01\textwidth}%

  \begin{subfigure}[t]{0.27\textwidth}
    \centering
    \includegraphics[width=\linewidth]{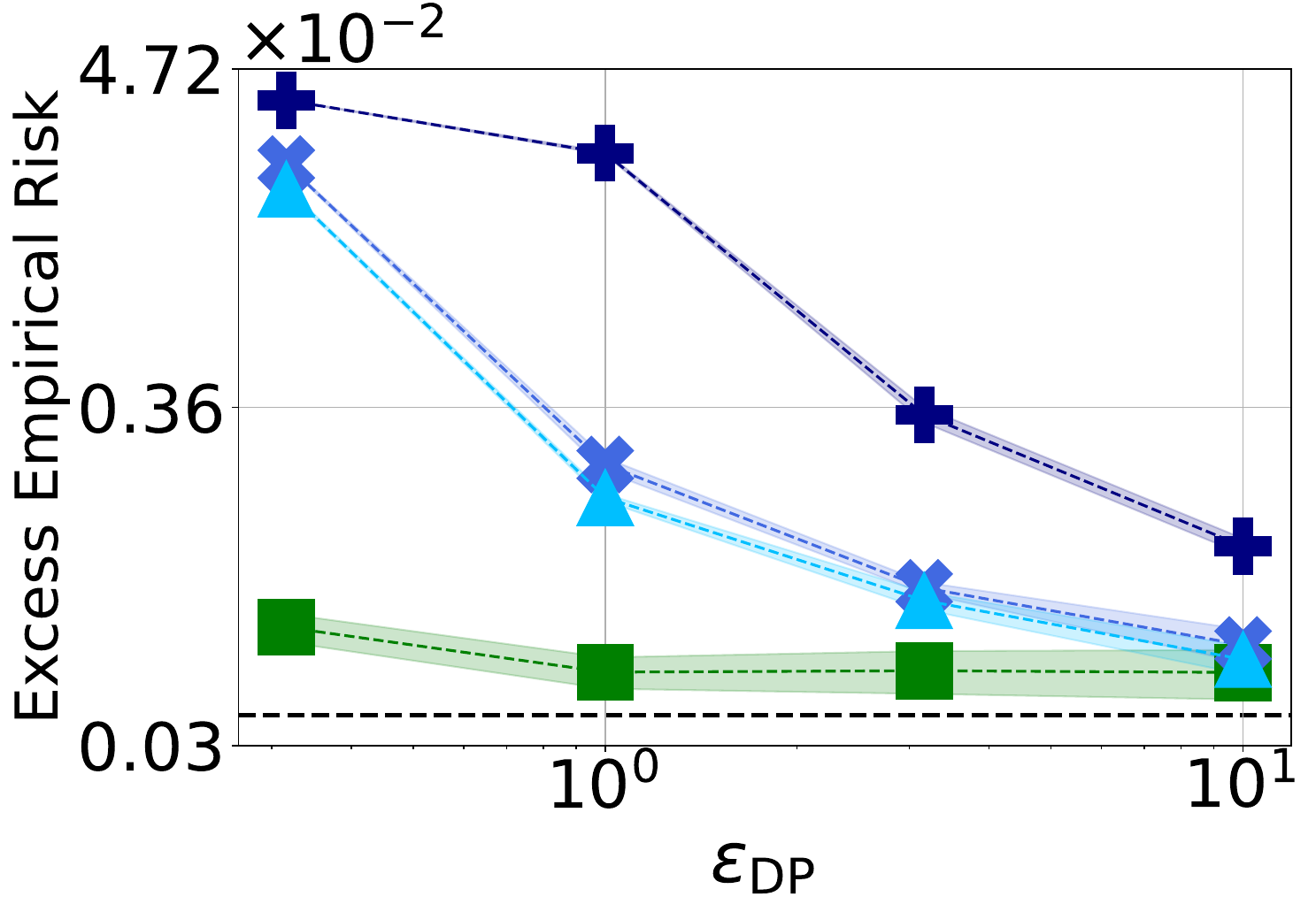}
    \vspace{-1.8em}
    \caption*{\hspace{1.0em} {\footnotesize Synthetic: Correlated}}
  \end{subfigure}
  \begin{subfigure}[t]{0.27\textwidth}
    \centering
    \includegraphics[width=\linewidth]{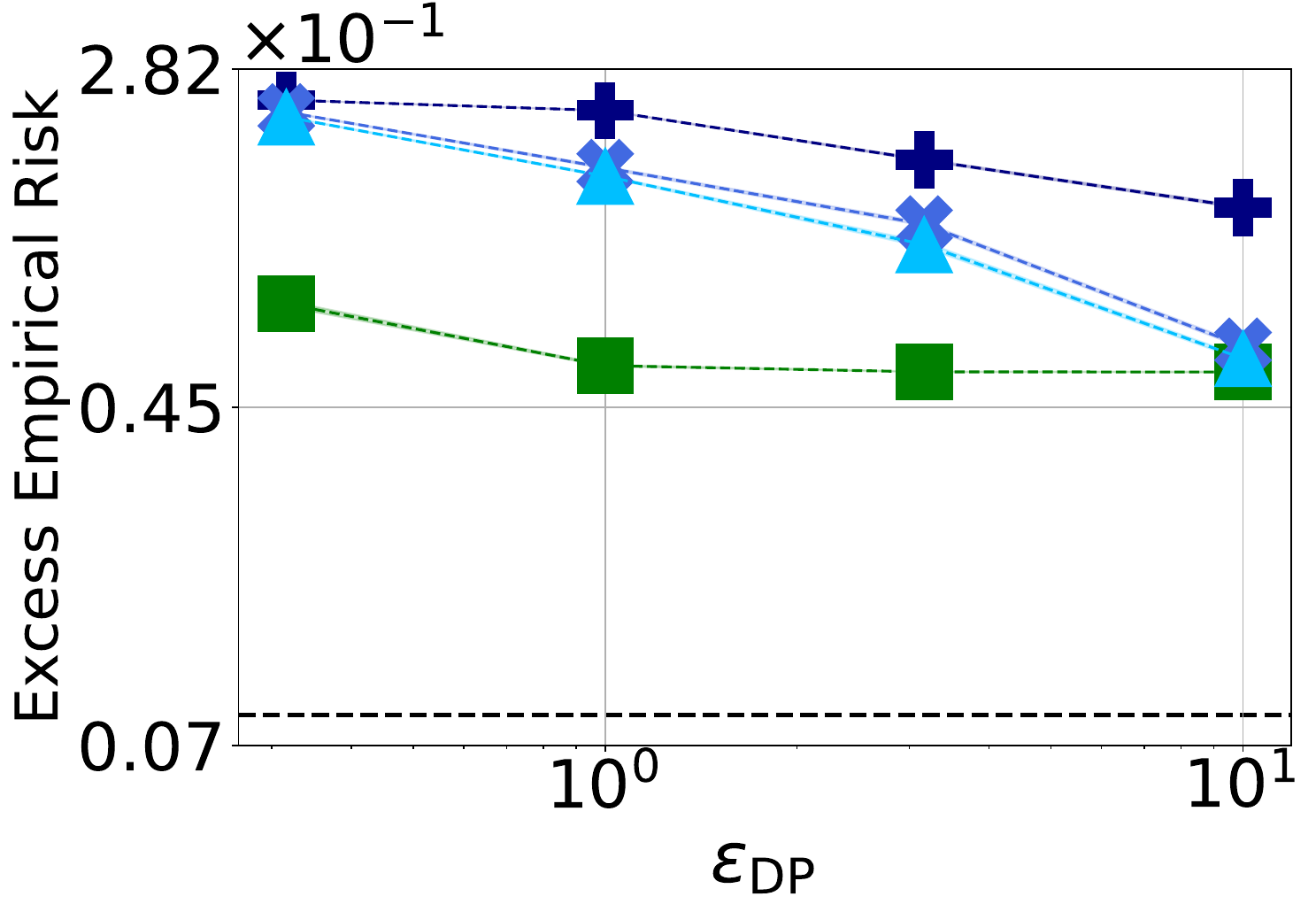}
    \vspace{-1.8em}
    \caption*{\hspace{1.0em} {\footnotesize Rossman}}
  \end{subfigure}\hspace{0.01\textwidth}%
  \begin{subfigure}[t]{0.27\textwidth}
    \centering
    \includegraphics[width=\linewidth]{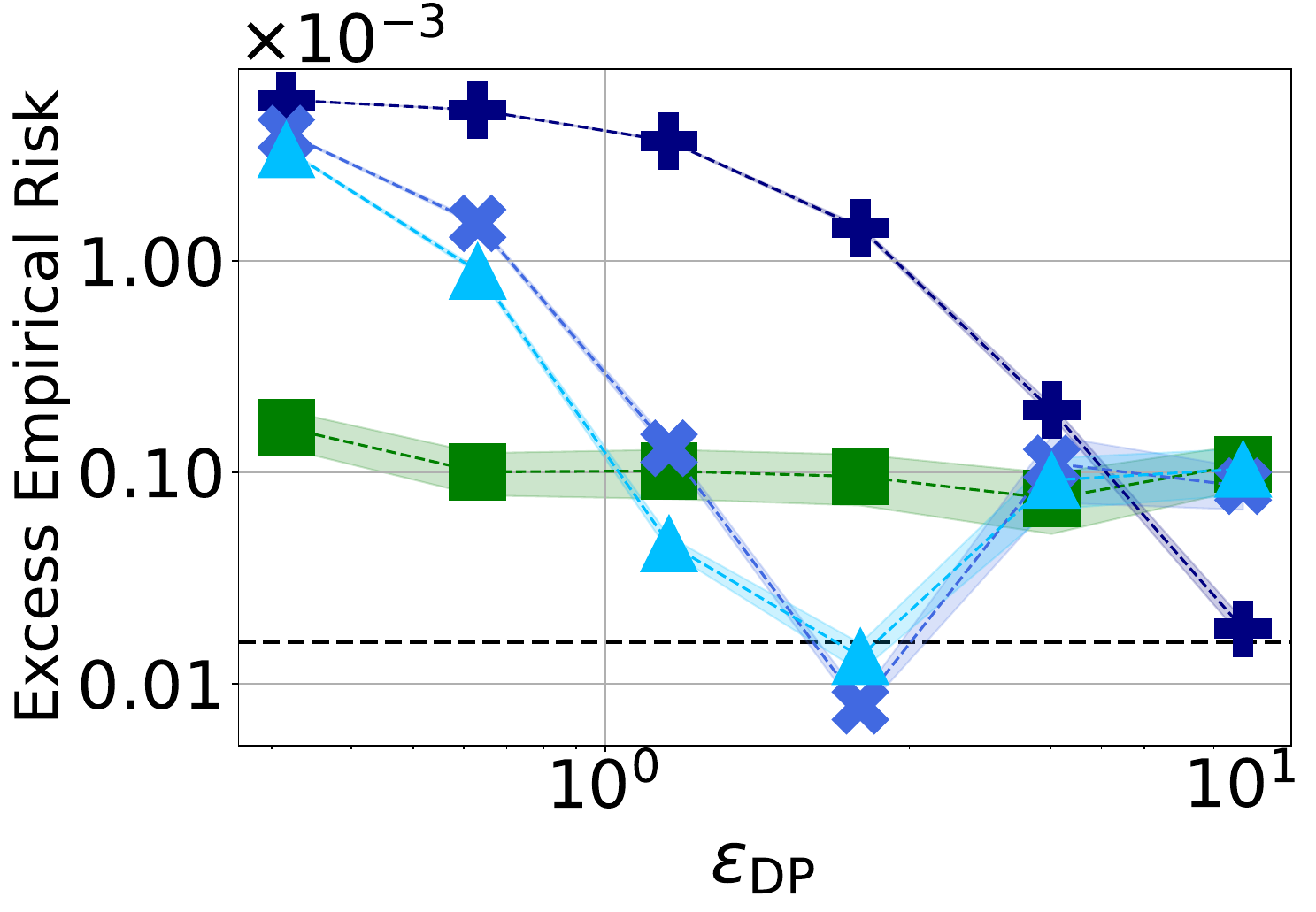}
    \vspace{-1.8em}
    \caption*{\hspace{1.0em} {\footnotesize Synthetic: $\lambda_{\min} = \nicefrac{n}{d}$}}
  \end{subfigure}\hspace{0.01\textwidth}%
    {\footnotesize
        \setlength{\tabcolsep}{2pt}%
        \renewcommand{\arraystretch}{0.9}%
        \begin{tabular*}{\linewidth}{@{\extracolsep{\fill}} c c c@{}}
        \textcolor{FastIHMColorA}{\rule{12pt}{1.6pt}$\pmb{+}$}
        $\mathrm{FastIHM}\hspace{-0.2em}:\hspace{-0.2em}
        k_2 = r_1 \cdot \max\{d,\hspace{-0.1em}\frac{4T}{\varrho}\hspace{-0.1em}\}$
        &
        \textcolor{FastIHMColorB}{\rule{12pt}{1.6pt}$\pmb{\times}$}
        $\mathrm{FastIHM}\hspace{-0.2em}:\hspace{-0.2em}
        k_2 = r_2 \cdot \max\{d,\hspace{-0.1em}\frac{4T}{\varrho}\hspace{-0.1em}\}$
        &
        \textcolor{FastIHMColorC}{\rule{12pt}{1.6pt}$\bm{\blacktriangle}$}
        $\mathrm{FastIHM}\hspace{-0.2em}:\hspace{-0.2em}
        k_2 = r_3 \cdot \max\{d,\hspace{-0.1em}\frac{4T}{\varrho}\hspace{-0.1em}\}$
        \\[-0.125em]
        \multicolumn{3}{@{}c@{}}{%
        \begin{tabular*}{0.36\linewidth}{@{\extracolsep{\fill}} c c@{}}
        \textcolor{Green}{\rule{12pt}{1.6pt}$\blacksquare$}$\ \mathrm{IHM}$
        &
        \textcolor{Black}{\hdashrule[0.5ex]{12pt}{0.9pt}{1.5pt 1pt}}%
        $\ \mathrm{FastIHS}: k = 6\cdot \max\{d,\frac{4T}{\varrho}\}$
        \end{tabular*}
        }
        \end{tabular*}
    }
    
    \caption{Performance comparison of \FastIHM and IHM. The upper plots are for $\frac{n}{d} \approx 1450$ and the lower for $\frac{n}{d}\approx 16000$. For each dataset, we use three values of $k_2$, represented by $r_1, r_2$, and $r_3$. The plots illustrate the dependence on $\lambda^{X}_{\max}$ and $\lambda^{X}_{\min}$, where smaller $\lambda^{X}_{\max}$ and larger $\lambda^{X}_{\min}$ improve performance. Moreover, in certain cases, \FastIHM matches the performance of IHM.}
  {
    \centering
    \footnotesize
    \setlength{\tabcolsep}{4.2pt}
    \renewcommand{\arraystretch}{1.03}
    \begin{adjustbox}{max width=\textwidth}
    \begin{tabular}{l
                    c
                    S[table-format=1.3]
                    S[table-format=1.3]
                    S[table-format=1.3]
                    S[table-format=1.3]}
    \toprule
    Dataset
    & {$(r_1, r_2, r_3)$}
    & {Fast IHS}
    & {Fast IHM: $r_1$}
    & {Fast IHM: $r_2$}
    & {Fast IHM: $r_3$} \\
    \midrule
    Black Friday
    & $(4,100,200)$ & 4.22 & 2.44 & 1.95 & 1.61 \\
    Beijing
    & $(4,100,200)$ & 6.89 & 3.12 & 2.38 & 1.94 \\
    YearsMSD
    & $(4,100,200)$ & 7.10 & 2.51 & 1.70 & 1.53 \\
    Synthetic Correlated
    & $(4,400,800)$ & 4.31 & 2.33 & 2.14 & 1.91 \\
    Rossman
    & $(4,400,800)$ & 5.13 & 3.01 & 2.63 & 2.40 \\
    Synthetic: $\lambda_{\min} = \frac{n}{d}$
    & $(4,400,800)$ & 5.0 & 2.52 & 2.26 & 2.06 \\
    \bottomrule
    \end{tabular}
    \end{adjustbox}
   }
   \caption{The ratio between the runtime of IHM and each of the different methods (FastIHS and FastIHM), for the three ratios $r_1, r_2$ and $r_3$. In most cases, FastIHM is \(2\) to \(3\) times faster than IHM.}
    \label{fig:main_text_ols}
\end{figure*}

\subsection{Simulations}
\label{s:sims}
Our experiments are designed to highlight the runtime advantage of \FastMix. For accuracy comparisons, we use the IHM from \citet{lev2026near}. This method was shown to attain state-of-the-art performance over a large collection of datasets, and further fixes the underlying implementation, making the runtime differences only due to the replacement of the Gaussian sketch with \FastMix. Our goal is to show that \FastMix attains state-of-the-art runtime for a given approximation accuracy, rather than to improve on the best known accuracy guarantees, which were already established in \citep{lev2026near}. We use IHS with SRHT (also termed \emph{FastIHS}) as our non-private baseline. We evaluate the methods on six large-scale datasets chosen to illustrate the applicability of the utility guarantee in \eqref{eq:iterative_fast_hessian_sketch_bound_empirical}. Dataset statistics are reported in \appendixref{app:statistics}, with full experimental details in \appendixref{app:exp_details}. In \appendixref{app:additional_exps}, we also include experiments with extreme $\nicefrac{n}{d}$, a regime where sketching is especially favorable. Each experiment shows the accuracy-privacy trade-off induced by varying $k_2$, and we further report the speedup factor attained relative to IHM. As shown in \figureref{fig:main_text_ols}, \FastIHM is substantially faster than IHM, with an accuracy loss that decreases as $k_2$ increases. The first row, corresponding to Black Friday, Years, and Beijing, considers datasets with $\nicefrac{n}{d} \approx 1450$, $\lambda_{\min}^{X} \approx 0$, and decreasing $\lambda_{\max}^{X}$. When $\lambda_{\max}^{X}$ is smaller, as is the case for the Beijing and the YearsMSD datasets, \FastIHM attains performance closer to IHM for the higher regime of $\eps$, while remaining faster. For the Black Friday dataset, where $\lambda^{X}_{\max}\approx 4\cdot 10^{4}$, there is a performance loss between \FastIHM and IHM, even for larger $\eps$'s. The second row considers datasets with $\nicefrac{n}{d} \approx 2^{14}$ and varying values of $\lambda_{\min}^{X}$ and $\lambda_{\max}^{X}$. Synthetic Correlated and Rossman both have poorly conditioned design matrices; However, the final plot corresponds to a setting where $\lambda_{\min}^{X}$ nearly attains its maximal theoretical value, $\nicefrac{n}{d}$. In this case, \FastIHM matches IHM's accuracy, as predicted by \eqref{eq:iterative_fast_hessian_sketch_bound_empirical}. Surprisingly, in this setting, the performance of \FastIHM\ is not monotone in $\eps$: for some values of $\eps$, it even outperforms IHM and matches the non-private baseline. We attribute this behavior to the implicit regularization induced by adding noise to the sketched Hessian. In IHM, the added noise level is always zero for all $\eps$. In contrast, for \FastIHM, the noise is zero only for the two largest values of $\eps$ and for the choices $r_2$ and $r_3$. For smaller values of $\eps$, the noise acts as a form of Hessian regularization. This is consistent with prior work showing that IHS can benefit from regularization or debiasing of the sketched Hessian \citep{zhang2023optimal,derezinski2020debiasing}. In our setting, this effect arises implicitly through the privacy noise added to the sketched Hessian. Developing a theory for optimal Hessian debiasing in the private setting is an interesting direction for future work.
\section{Discussion and Future Work}
\label{s:discussion}
We introduced a fast DP mechanism based on sketching, paralleling fast sketching methods. The mechanism attains state-of-the-art DP guarantees and, to the best of our knowledge, is the first fast sketching-based DP mechanism whose privacy degradation relative to a Gaussian sketch is only a constant factor in some regimes. We applied it to DP-OLS, obtaining an algorithm with state-of-the-art runtime and provable accuracy guarantees. Relative to prior DP sketching-based approaches, our method is significantly faster and, in some regimes, incurs negligible accuracy degradation. Large-scale regression simulations confirm the gains, demonstrating speedups.

Our results suggest several directions for future work. The current instantiation of \FastMix\ concatenates an SRHT with a Gaussian sketch. This choice has three drawbacks: it does not exploit sparsity in the input matrix; it does not minimize correlations between sketch columns, leaving $\widehat{\Delta}$ in Line~\ref{line:aux_computes_FastMix} relatively large (see \appendixref{app:srht_coherence}); and its cost depends on $k_1k_2d$, rather than $kd^2$ for a single sketch with $k$ rows, slightly increasing the complexity relative to \citet[Section.~III.A]{pilanci2015randomized}, in particular whenever $X$ has a low-rank structure. It would therefore be valuable to study alternative sketches within \FastMix, such as sparse sketches or sketches with nearly minimal coherence, and to examine whether the concatenated sketch can be replaced by a single sketch. Such variants may improve both accuracy and runtime. In particular, since modern sketching methods adapt the sketch to the matrix being sketched (see, e.g., \citep{derezinski2021newton}), it is interesting to ask whether analogous constructions are possible under DP. Beyond DP-OLS, \FastMix may serve as a primitive for second-order DP convex optimization algorithms in the spirit of the Newton sketch \citep{pilanci2017newton}, enabling computationally efficient DP second-order methods. 

\section{Acknowledgemens}
OL and AW were supported in part by the Simons Foundation Collaboration on Algorithmic Fairness under Award SFI-MPS-TAF-0008529-14, and VS and AW were supported in part by Assicurazioni Generali S.p.A. through MIT Award 036189-00006.
The work of MS and KL was supported in part by ERC grant 101125913, Simons Foundation Collaboration 733792, and Apple. Views and opinions expressed are however those of the author(s) only and do not necessarily reflect those of the European Union or the European Research Council Executive Agency. Neither the European Union nor the granting authority can be held responsible for them.

\clearpage        
\bibliographystyle{plainnat}
\bibliography{Config/Bib_Shortened}
\appendix

\newpage 
\section{Additional notation}
\label{app:notation}
Random variables are denoted using sans-serif fonts (e.g., \(\mathsf{X}, \mathsf{y}\)), while their realizations are represented by regular italics (e.g., \(X, y\)). The natural logarithm (with base $e$) is denoted by $\log(\cdot)$. The $L_2$ norm of a vector $v \coloneqq (v_1,\ldots, v_d)$ is given by $\left(\sum^{d}_{i=1}v^2_i\right)^{\frac{1}{2}}$ and is denoted by $\norm{v}$, its $L_{\infty}$ norm is given by $\underset{i=1,\ldots, d}{\max} \ \abs{v_i}$ and is denoted by $\norm{v}_{\infty}$, and its Mahalanobis norm with respect to a PSD matrix $B$, namely, the quantity $\left(v^{\top}Bv\right)^{\frac{1}{2}}$, is denoted by $\norm{v}_B$.
In the case where \(A\) is a square matrix, its trace and determinant are denoted by \(\mathrm{Tr}(A)\) and \(\det (A)\) respectively, and its minimal and the maximal eigenvalues are $\lambda_{\min}(A)$ and $\lambda_{\max}(A)$ respectively. The minimal and maximal eigenvalues of its Gram matrix are denoted by $\lambda^{A}_{\min}:= \lambda_{\min}(A^{\top}A)$ and $\lambda^{A}_{\max}:= \lambda_{\max}(A^{\top}A)$. The operator norm of a matrix $A$ is denoted by $\norm{A}_{\mathrm{op}}$ and is defined by $\underset{v:\norm{v}=1}{\sup} \ \norm{Av} = \sqrt{\lambda^{A}_{\max}}$ and the Frobenius norm by $\norm{A}_{\mathrm{fro}} \hspace{-0.25em}:=\hspace{-0.25em} \left(\hspace{-0.15em}\sum^{n}_{i=1}\hspace{-0.1em}\sum^{d}_{j=1}\hspace{-0.1em}a^2_{ij}\hspace{-0.15em}\right)^{\nicefrac{1}{2}}$. We use the notation $A \succeq 0$ and $A\succ0$ for denoting situations in which $A$ is Positive Semi-Definite (PSD) and in which $A$ is Positive Definite (PD). The Loewner order is defined in the usual way, where $B \preceq A$ denotes $A - B \succeq 0$. We usually denote our dataset $\left\{x_i\right\}^{n}_{i=1}$ where each $x_i \in \reals^{d}$ in the matrix form $X = (x_1,x_2,\ldots, x_n)^{\top}$. The set of integer numbers from $1$ to $n$ is denoted by $[n]$. The all zeros column vector of size $d$ is denoted by $\vec{0}_d \coloneqq (0,0, \ldots, 0)^{\top}$, and a matrix of size $a\times b$ full of zeros is denoted by $0_{a\times b}$. We denote by $\pN(0, \brI_{k_1\times k_2})$ a $k_1\times k_2$ matrix consisting of \iid \ Gaussian elements with zero mean and unit variance. The $k\times k$ identity matrix is denoted by $\brI_{k}$. We write \(y(x)=O(x)\) if \(\exists c>0, x_0 > 0\) such that \(|y(x)| \le cx\) for all \(x \ge x_0\). We write \(y(x)=o(x)\) if \(\underset{x\to\infty}{\lim} \frac{y(x)}{x} = 0\). We write \(y(x)=\Omega(x)\) if \(\exists c>0, x_0 > 0\) such that \(y(x) \ge cx\) for all \(x \ge x_0\). We write $y(x) = \Theta(x)$ if $y(x) = O(x)$ and $y(x) = \Omega(x)$. Given two probability distributions $\mathbb{P}$ and $\mathbb{Q}$ and a scalar parameter $\lambda \in [0, 1]$, the notation $\lambda \mathbb{P} + (1-\lambda)\mathbb{Q}$ denotes the distribution of the random variable $\mathsf{x} = \mathsf{b}\mathsf{p} + (1-\mathsf{b})\mathsf{q}$ for $\mathsf{p} \sim \mathbb{P}, \mathsf{q} \sim \mathbb{Q}, \mathsf{b} \sim \mathrm{Ber}(\lambda)$ and $\mathsf{p} \indep \mathsf{q}, \mathsf{q} \indep \mathsf{b}, \mathsf{p} \indep \mathsf{b}$. 
\section{Properties of Subsampled Randomized Hadamard Transform}
\label{app:helpers}
We first repeat the formal definition of a Hadamard matrix. Then, we provide an auxiliary lemma regarding SRHT, which will be proved useful in our analysis. 

\bigskip 

\begin{defn}[Hadamard Matrix]
    An Hadamard matrix of size $2n$ is denoted by $H_{2n}$ and is defined via the recursive formula
    \begin{align}
        \widetilde{H}_{2n} = \begin{pmatrix} \widetilde{H}_{n} & \widetilde{H}_{n} \\ \widetilde{H}_{n} & -\widetilde{H}_{n}\end{pmatrix}, \ \ \ \ \widetilde{H}_1 = (1), \ \ \ \ H_{n} = \frac{1}{\sqrt{n}}\widetilde{H}_{n}.
    \end{align}
\end{defn}
\begin{lem}
    \label{lem:properties_SRHT}
    For any \(n \hspace{-0.2em}\geq\hspace{-0.2em} 1\) and \(k \hspace{-0.2em}\in\hspace{-0.2em} [n]\), any \(S_{\mathrm{f}}\sim\mathrm{SRHT}(k,\hspace{-0.1em} n)\) satisfies the following properties.
    \begin{enumerate}[label=(\arabic*)]
        \item For all \(j_{1}, j_{2} \in [k]\) such that \(j_{1} \neq j_{2}\), the \(j_{1}\)-th row \((S_{\mathrm{f}})_{j_{1},:}\) and the \(j_{2}\)-th row \((S_{\mathrm{f}})_{j_{2},:}\) are orthogonal.

        \item For all \(i \in [n]\), the \(i\)-th column \((S_{\mathrm{f}})_{:,i}\) satisfies \(\|(S_{\mathrm{f}})_{:,i}\| = 1\).
\end{enumerate}
\end{lem}
\begin{proof}
    We start with establishing $(1)$. To that end, recall that any $S_{\mathrm{f}} \in \mathrm{SRHT}(k, n)$ is given by 
    \begin{equation}
        S_{\mathrm{f}} = \sqrt{\frac{n}{k}} P_{k\times n} H_n B_n
    \end{equation}
    where $B_n$ is a diagonal matrix with entries $\left\{\pm1\right\}$, $H_n$ is a normalized Hadamard matrix, and $P_{k\times n}$ is a row subsampling matrix with a single distinct entry equal to $1$ in each row. Then, denoting the rows of the matrix $H_n$ by $\left\{h^{\top}_i\right\}^{n}_{i=1}$ where each $h_i \in \reals^{n\times 1}$ we get that the rows of $S_{\mathrm{f}}$ are given by $\left\{\sqrt{\frac{n}{k}}h^{\top}_{j_i} B_n\right\}^{k}_{i=1}$ where each $j_i \in [n]$ is a distinct index from $[n]$. Thus, we get that the inner product between two different arbitrary rows with indices $j_1$ and $j_2$ such that $j_1\ne j_2$ is given by 
    \begin{equation}
        \frac{n}{k} h^{\top}_{j_1} B_n B^{\top}_n h_{j_2} = \frac{n}{k} h^{\top}_{j_1} h_{j_2} = 0
    \end{equation}
    and where the first equality is since $B_n B^{\top}_n = \brI_n$ by the definition of $B_n$, and the last equality is since the rows of the Hadamard matrix are orthogonal. This proves point (1). Moreover, for $j_1 = j_2$ we note that this inner product is given by $\frac{n}{k}$. 

    To prove (2), since $H_n$ is a normalized Hadamard matrix, and since the entries of $\mathsf{B}_n$ and $\mathsf{P}_{k\times n}$ are either $0$'s or either $\pm 1$, the entries of $S_{\mathrm{f}}$ belong to $\left\{\pm \frac{1}{\sqrt{k}}\right\}$. Since each column of $S_{\mathrm{f}}$ contains exactly $k$ elements, the norm of each column is $1$.  
\end{proof}
\section{Proof of \theoremref{thm:Privacy_First}}
\label{app:full_proof_privacy}



Our proof will use the next standard formula for the \Renyi divergence between two $d$-dimensional multivariate Gaussian distributions \citep{gil2013renyi}
\begin{equation}
    \label{eq:Gaussian_Div}
        \KLDA{\pN\left(\vec{0}_d, \Sigma_1\right)}{\pN\left(\vec{0}_d, \Sigma_2\right)} = - \frac{1}{2(\alpha - 1)} \log \left(\frac{\deter{\Sigma_1 + \alpha\left(\Sigma_2 - \Sigma_1\right)}}{\left(\deter{\Sigma_1}\right)^{1-\alpha}\left(\deter{\Sigma_2}\right)^{\alpha}}\right)\ \ \ 
\end{equation}
which holds for all $\alpha$ such that $\alpha \Sigma^{-1}_1 + (1-\alpha)\Sigma^{-1}_2 = \Sigma^{-1}_2 + \alpha\left(\Sigma^{-1}_1 - \Sigma^{-1}_2\right)\succ 0$. 

\bigskip 

We will first provide a proof of \lemmaref{lem:privacy_lemma}. To that end, we start with providing a tight characterization of the privacy guarantee of $\pM_{S_{\mathrm{f}}}$, which we then simplify towards the form presented in \lemmaref{lem:privacy_lemma}. Then, we provide some additional auxiliary lemmas, which are necessary for establishing the overall privacy guarantees. For proving the final claim, we will show that the assumptions of \lemmaref{lem:privacy_lemma} hold in the setting of \FastMix, which, under the right choice of parameters, makes its output $(\eps(\omega, \gamma, k_1, \delta), \delta)$-DP.

\subsection{Proof of \lemmaref{lem:privacy_lemma}}
\label{app:proof_of_aux_privacy_lemma}

Without loss of generality, we consider $X - X' = e_1 x^{\top}_1$; in simpler terms, $X'$ is obtained by zeroing out the first row in $X$.
We introduce some notation for convenience before delving into the proof.
The \(i^{th}\) column of \(S_{\mathrm{f}}\) is \(u_{i}\coloneq S_{\mathrm{f}}e_{i}\).
The matrices \(Z, Z'\) are the outputs of \(X, X'\) transformed using \(S_{\mathrm{f}}\) respectively as
\begin{equation}
    Z\coloneq S_{\mathrm{f}}X \qquad \text{and} \qquad Z'\coloneq S_{\mathrm{f}}X' = Z - S_{\mathrm{f}}e_1x^{\top}_1 = Z - u_1x^{\top}_1.
\end{equation}
Finally, we will use the notation $M\coloneq (Z^{\top}Z + \sigma^{2}\brI_{d})^{-1}$ and $\delta_i\coloneq Z^{\top}S_{\mathrm{f}}e_i - x_i$ throughout. 

We begin by observing the difference between \(Z'^{\top}Z'\) and \(Z^{\top}Z\) as these play a key role in the R\'{e}nyi divergence calculation that follows.
By definition,
\begin{align}
    Z'^{\top}Z' = \left(Z - u_1 x^{\top}_1\right)^{\top}\left(Z - u_1x^{\top}_1\right) = Z^{\top}Z + \norm{u_1}^2x_1x^{\top}_1 - Z^{\top}u_1x^{\top}_1 - \left(Z^{\top}u_1x^{\top}_1\right)^{\top}.
\end{align}
Define \(\widetilde{x}_{1} = \|u_{1}\| x_{1}\) and \(\widetilde{u}_{1} = \frac{u_{1}}{\|u_{1}\|}\).
Then,
\begin{equation*}
    Z'^{\top}Z' = Z^{\top}Z + \widetilde{x}_{1}\widetilde{x}_{1}^{\top} - Z^{\top}\widetilde{u}_{1}\widetilde{x}_{1}^{\top} - (Z^{\top}\widetilde{u}_{1}\widetilde{x}_{1}^{\top})^{\top}.
\end{equation*}
Since \(\|u_{1}\| \leq 1\), we have \(\|\widetilde{x}_{1}\| = \|u_{1}\| \|x_{1}\| \leq 1\) by the preconditions on \(S_{\mathrm{f}}\) and \(X\), and \(\|\widetilde{u}_{1}\| = 1\).\footnote{If $\norm{u_1}=0$, then $Z' = Z$ and the divergence is $0$. Thus, we assume that $\norm{u_1} > 0$. }
The remainder of the proof only uses bounds on \(\|\widetilde{x}_{1}\|\) and \(\|\widetilde{u}_{1}\|\), and therefore without loss of generality, we assume that \(\|u_{1}\| = 1, \|x_{1}\| \leq 1\)
and write
\begin{align}
    Z'^{\top}Z'
    = Z^{\top}Z + x_1x^{\top}_1 - Z^{\top}u_1x^{\top}_1 - \left(Z^{\top}u_1x^{\top}_1\right)^{\top}.
\end{align}

Next, we note that the difference between \(Z'^{\top}Z'^{\top}\) and \(Z^{\top}Z\) can be expressed as a rank-2 update.
More precisely, define \(U \in \reals^{d \times 2}\coloneq [-Z^{\top}u_{1} + x_{1}, x_{1}]\) and \(V \in \reals^{d \times 2}\coloneq [x_{1}, -Z^{\top}u_{1}]\).
Then,
\begin{align*}
    UV^{\top} &= -Z^{\top}u_{1}x_{1}^{\top} + x_{1}x_{1}^{\top} - x_{1}u_{1}^{\top}Z = Z'^{\top}Z'^{\top} - Z^{\top}Z~.
\end{align*}


Since $\mathsf{g}$ and $\xi$ are independent standard normal vectors this implies that 
\begin{align}
    (\mathsf{g}^{\top}S_{\mathrm{f}}X + \sigma \xi^{\top})^{\top} &\sim \pN\left(0, Z^{\top}Z + \sigma^2 \brI_d\right), \\
    (\mathsf{g}^{\top}S_{\mathrm{f}}X' + \sigma \xi^{\top})^{\top} &\sim \pN\left(0, Z'^{\top}Z' + \sigma^2 \brI_d\right) \\
    &\overset{(d)}= \pN\left(0, Z^{\top}Z + \sigma^2 \brI_d +  x_1x^{\top}_1 - Z^{\top}u_1x^{\top}_1 - \left(Z^{\top}u_1x^{\top}_1\right)^{\top}\right).
\end{align}
Using \eqref{eq:Gaussian_Div}, we have  
\begin{multline}
    \label{eq:divergence_simplified_Gaussian}
    \KLDA{\mathsf{g}^{\top}S_{\mathrm{f}}X+ \sigma \xi^{\top}}{\mathsf{g}^{\top}S_{\mathrm{f}}X'+ \sigma \xi^{\top}} \\
    =- \frac{1}{2(\alpha - 1)} \log \left(\frac{\deter{Z^{\top}Z+\sigma^2 \brI_d + \alpha(Z'^{\top}Z' - Z^{\top}Z)}}{\left(\deter{Z^{\top}Z +\sigma^2 \brI_d}\right)^{1-\alpha}\left(\deter{Z'^{\top}Z' + \sigma^2 \brI_d}\right)^{\alpha}}\right).
\end{multline}

Now, using the matrix determinant lemma and $\mathrm{det}(AB) = \mathrm{det}(A)\mathrm{det}(B)$ we get for any \(\beta \geq 0\)
\begin{equation}
    \label{eq:determinant_lemma_2}
    \deter{Z^{\top}Z + \sigma^2\brI_d + \beta UV^{\top}} = \deter{Z^{\top}Z + \sigma^2\brI_d}\deter{\brI_2 + \beta V^{\top}MU}
\end{equation}

where the inner matrix $V^{\top}MU$ can be rewritten explicitly as
\begin{equation}
    V^{\top}MU = 
        \begin{bmatrix}
                            x^{\top}_1M(-Z^{\top}u_1+x_1) & x^{\top}_1Mx_1 \\
                            (u_1)^{\top}ZM(Z^{\top}u_1-x_1) & -(u_1)^{\top}ZMx_1
        \end{bmatrix}.
\end{equation}

Define
\begin{equation}
    s_1:=x^{\top}_1Mx_1,\qquad s_2:=(u_1)^{\top}ZMZ^{\top}u_1, \qquad c\coloneq -x^{\top}_1MZ^{\top}u_1 = -(u_1)^{\top}ZMx_1.
\end{equation}

Then,
\begin{equation}
    \deter{\brI_2 + \beta V^{\top}MU} = \deter{\begin{pmatrix}  1 + \beta s_1 + \beta c & \beta s_1 \\
                        \beta s_2 + \beta c & 1 + \beta c
            \end{pmatrix}} = 1 + \beta s_1 + 2\beta c + \beta^{2} c^2 - \beta^{2} s_1s_2~.
\end{equation}
Applying this to \(\deter{Z'^{\top}Z' + \sigma^{2}\brI_{d}}\) with \(\beta = 1\), we get
\begin{align}
    \deter{Z'^{\top}Z' + \sigma^2\brI_d} &= \deter{Z^{\top}Z + \sigma^2\brI_d}\left(1 + s_1 + 2c + c^2 - s_1s_2\right)~,
\end{align}
and applying this with \(\beta = \alpha\) for the numerator in \eqref{eq:divergence_simplified_Gaussian}, we obtain
\begin{equation}
    \label{eq:determinant_lemma_3}
    \deter{Z^{\top}Z + \sigma^2\brI_d + \alpha (Z'^{\top}Z' - Z^{\top}Z)} = \deter{Z^{\top}Z + \sigma^2\brI_d}\left(1 + \alpha s_1 + 2\alpha c + \alpha^{2} c^2 - \alpha^{2} s_1s_2 \right)~.
\end{equation}
Substituting these inside \eqref{eq:divergence_simplified_Gaussian} yields
\begin{align}
    D_{\alpha}\left(\mathsf{g}^{\top}S_{\mathrm{f}}X + \sigma \xi^{\top}\right. &\| \left.\mathsf{g}^{\top}S_{\mathrm{f}}X' + \sigma \xi^{\top}\right) \\
    &= 
    -\frac{1}{2(\alpha - 1)} \log \left(\frac{\deter{Z^{\top}Z + \sigma^2\brI_d + \alpha(Z'^{\top}Z' - Z^{\top}Z)}}{\left(\deter{Z^{\top}Z + \sigma^2\brI_d}\right)^{1-\alpha}\left(\deter{Z'^{\top}Z' + \sigma^2\brI_d}\right)^{\alpha}}\right)\\
    \label{eq:divergence_final_Gaussian}
    &=  \frac{1}{2(\alpha - 1)} \log \left(\frac{\left(1 - (s_1s_2 - c^2) + \left(2c + s_1\right)\right)^{\alpha}}{1 - \alpha^2\left(s_1s_2 - c^2\right) + \alpha \left(2c + s_1\right)}\right).
\end{align}

Our goal now is to obtain a closed-form upper bound on \eqref{eq:divergence_final_Gaussian}, which depends only on global properties of $S_{\mathrm{f}}X$ such as the row norm bound $\textsc{C}^2_{S_{\mathrm{f}}X}$ and the minimal eigenvalue $\lambda^{S_{\mathrm{f}}X}_{\min}$. To do this, we obtain bounds on \(s_{1}s_{2} - c^{2}\) and \(2c + s_{1}\) that solely depend on these quantities, and then obtain an upper bound for \eqref{eq:divergence_final_Gaussian} by casting it as an optimization problem.

Firstly, \(s_{1}s_{2} - c^{2}\) is non-negative due to the Cauchy-Schwarz inequality. Then, we further note that the next upper bound
\begin{align}
    s_1s_2 - c^2 &= \left(x^{\top}_1Mx_1\right)\left(u^{\top}_1ZMZ^{\top}u_1\right) - \left(x^{\top}_1MZ^{\top}u_1\right)^2 \\
    &= \left(x^{\top}_1Mx_1\right)\left((x_1 + \delta_1)^{\top}M(x_1 + \delta_1)\right) - \left(x^{\top}_1M(x_1 + \delta_1)\right)^2\\
    &= \left(x^{\top}_1Mx_1\right)\left(\delta_1^{\top}M\delta_1\right) - \left(x^{\top}_1M\delta_1\right)^2\\
    &\leq \left(x^{\top}_1Mx_1\right)\left(\delta_1^{\top}M\delta_1\right)\\
    &\leq \norm{x_1}^2\norm{\delta_1}^2\left(\lambda_{\max}(M)\right)^2\\
    &\leq \frac{\norm{Z^{\top}S_{\mathrm{f}}e_1 - x_1}^2}{\left(\lambda^{Z}_{\min} + \sigma^2\right)^2}\\
    &\leq \underset{i\in[n]}{\max} \ \frac{\norm{Z^{\top}S_{\mathrm{f}}e_i - x_i}^2}{\left(\lambda^{Z}_{\min} + \sigma^2\right)^2}
\end{align}
holds by the definition of $\delta_1$ and the assumption $\norm{x_1}^2 \leq 1$.

Next, we work with \(2c + s_{1}\).
Note that \(\delta_{1} = Z^{\top}u_{1} - x_{1}\) and 
\begin{equation*}
    2c + s_{1} = -2\delta_{1}^{\top}Mx_{1} - 2\|x_{1}\|^{2}_{M} + \|x_{1}\|^{2}_{M} = -2\delta_{1}^{\top}Mx_{1} - \|x_{1}\|^{2}_{M}~.
\end{equation*}
Therefore,
\begin{equation}
    -\norm{x_1}^2_M - 2\norm{x_1}_M \norm{\delta_1}_M \leq 2c+s_1 \leq -\norm{x_1}^2_M + 2\norm{x_1}_M \norm{\delta_1}_M,
\end{equation}
and which we can further bound as 
\begin{align}
    2c + s_1 &\leq 2\norm{x_1}_M\norm{\delta_1}_M \leq \frac{2\norm{Z^{\top}S_{\mathrm{f}}e_1 - x_1}}{\lambda^{Z}_{\min} + \sigma^2} \leq 2\cdot \frac{\underset{i\in[n]}{\max}\ \norm{Z^{\top}S_{\mathrm{f}}e_i - x_i}}{\lambda^{Z}_{\min} + \sigma^2}, \\ 
    2c + s_1 &\geq -\frac{1}{\lambda^{Z}_{\min} + \sigma^2} - 2\frac{\norm{Z^{\top}S_{\mathrm{f}}e_1 - x_1}}{\lambda^{Z}_{\min} + \sigma^2} \geq -\frac{1 + 2\cdot \underset{i\in[n]}{\max} \ \norm{Z^{\top}S_{\mathrm{f}}e_i - x_i}}{\lambda^{Z}_{\min} + \sigma^2}.
\end{align}


From \eqref{eq:divergence_final_Gaussian}, we obtain natural constraints for the quantities involved:
\begin{equation}
    \label{eq:constraints_function_log}
    1 - (s_1s_2 - c^{2}) + s_1 + 2c > 0, \quad 1 - \alpha^2\left(s_1s_2 - c^2\right) + \alpha \left(2c + s_1\right) > 0. 
\end{equation}

With this, we obtain an upper bound for \eqref{eq:divergence_final_Gaussian} by posing this as an optimization problem below; this can be viewed as the tightest upper bound possible with the given constraints:
\begin{equation}
\label{eq:opt_problem_c2ms1s2_s2}
    \sup_{(s_1,s_2,c)\in \mathcal{D}_1}\;
    \Bigg\{\log\left(\frac{\big(1+ s_1 + 2c + c^2 - s_1 s_2\big)^{\alpha}}{1 + \alpha\left( s_1 + 2c\right) + \alpha^2 \left(c^2 - s_1 s_2\right)}\right)\Bigg\},
\end{equation}
where
\begin{align*}
    \mathcal{D}_1 &:=\Bigg\{(s_1,s_2,c)\in\reals^3:
        -\underset{i\in[n]}{\max} \ \frac{\norm{Z^{\top}S_{\mathrm{f}}e_i - x_i}^2}{\left(\lambda^{Z}_{\min} + \sigma^2\right)^2}\le c^2 - s_1s_2 \le 0, \ \alpha > 1, \\ 
        &\qquad\qquad -\frac{1 + 2\cdot \underset{i\in[n]}{\max} \ \norm{Z^{\top}S_{\mathrm{f}}e_i - x_i}}{\lambda^{Z}_{\min} + \sigma^2} \leq s_1+2c \leq \frac{2\cdot \underset{i\in[n]}{\max} \ \norm{Z^{\top}S_{\mathrm{f}}e_i - x_i}}{\lambda^{Z}_{\min} + \sigma^2}, \\
        &\qquad\qquad 1+ s_1 + 2c + c^2 - s_1 s_2>0,\ 1+\alpha\left( s_1 + 2c\right) + \alpha^2 \left(c^2 - s_1 s_2\right)>0\Bigg\}.
\end{align*}

From Claim.~\ref{claim:renyi-div-upper-bound} applied with 
\begin{align}
    a&\gets s_1, b\gets s_2, c\gets c, \ell \gets \frac{\underset{i\in[n]}{\max} \ \norm{Z^{\top}S_{\mathrm{f}}e_i - x_i}^2}{\left(\lambda^{Z}_{\min} + \sigma^2\right)^2}, m\gets \frac{1 + 2\cdot \underset{i\in[n]}{\max} \ \norm{Z^{\top}S_{\mathrm{f}}e_i - x_i}}{\lambda^{Z}_{\min} + \sigma^2}
\end{align}
and $M\gets \frac{2\cdot \underset{i\in[n]}{\max} \ \norm{Z^{\top}S_{\mathrm{f}}e_i - x_i}}{\lambda^{Z}_{\min} + \sigma^2}$, we have the following bound for the objective in \eqref{eq:opt_problem_c2ms1s2_s2} as
\begin{multline*}
    \KLDA{\mathsf{g}^{\top}S_{\mathrm{f}}X + \sigma\xi^{\top}}{\mathsf{g}^{\top}S_{\mathrm{f}}X' + \sigma\xi^{\top}} \\ \leq \frac{1}{2(\alpha - 1)}\log\left(\frac{\left(1 - \underset{i\in[n]}{\max} \ \tfrac{1 + 2\norm{Z^{\top}S_{\mathrm{f}}e_i - x_i}}{\lambda^{Z}_{\min} + \sigma^2} - \underset{i\in[n]}{\max} \ \tfrac{\norm{Z^{\top}S_{\mathrm{f}}e_i - x_i}^2}{\left(\lambda^{Z}_{\min} + \sigma^2\right)^2}\right)^{\alpha}}{1 - \alpha\cdot \underset{i\in[n]}{\max} \ \tfrac{1 + 2\norm{Z^{\top}S_{\mathrm{f}}e_i - x_i}}{\lambda^{Z}_{\min} + \sigma^2} -  \alpha^2 \cdot \underset{i\in[n]}{\max} \ \tfrac{\norm{Z^{\top}S_{\mathrm{f}}e_i - x_i}^2}{\left(\lambda^{Z}_{\min} + \sigma^2\right)^2}}\right).
\end{multline*}

Then, since 
\begin{align}
    1 + 2\cdot \underset{i\in[n]}{\max} \ \norm{Z^{\top}S_{\mathrm{f}}e_i - x_i} = \textsc{C}^2_{S_{\mathrm{f}}X}, \ \ \ \underset{i\in[n]}{\max} \ \norm{Z^{\top}S_{\mathrm{f}}e_i - x_i}^2 = \frac{1}{4}\left(1 - \textsc{C}^2_{S_{\mathrm{f}}X}\right)^2, 
\end{align}
and since $\frac{\lambda^{Z}_{\min} + \sigma^2}{\textsc{C}^2_{S_{\mathrm{f}}X}} = \gamma$ we obtain the upper bound 
\begin{align}
    &\KLDA{\mathsf{g}^{\top}S_{\mathrm{f}}X + \sigma\xi^{\top}}{\mathsf{g}^{\top}S_{\mathrm{f}}X' + \sigma\xi^{\top}}\\ 
    &\qquad\leq \frac{1}{2(\alpha - 1)}\log\left(\frac{\left(1 - \underset{i\in[n]}{\max} \ \tfrac{1 + 2\norm{Z^{\top}S_{\mathrm{f}}e_i - x_i}}{\lambda^{Z}_{\min} + \sigma^2} - \underset{i\in[n]}{\max} \ \tfrac{\norm{Z^{\top}S_{\mathrm{f}}e_i - x_i}^2}{\left(\lambda^{Z}_{\min} + \sigma^2\right)^2}\right)^{\alpha}}{1 - \alpha \cdot \underset{i\in[n]}{\max} \ \tfrac{1 + 2\norm{Z^{\top}S_{\mathrm{f}}e_i - x_i}}{\lambda^{Z}_{\min} + \sigma^2} - \alpha^2\cdot \underset{i\in[n]}{\max} \ \tfrac{\norm{Z^{\top}S_{\mathrm{f}}e_i - x_i}^2}{\left(\lambda^{Z}_{\min} + \sigma^2\right)^2}}\right)\\
    &\qquad = \frac{1}{2(\alpha - 1)}\log\left(\frac{\left(1 - \tfrac{\textsc{C}^2_{S_{\mathrm{f}}X}}{\lambda^{Z}_{\min} + \sigma^2} - \left(\tfrac{\textsc{C}^2_{S_{\mathrm{f}}X}}{2\left(\lambda^{Z}_{\min} + \sigma^2\right)}\right)^2\left(\tfrac{1 - \textsc{C}^2_{S_{\mathrm{f}}X}}{\textsc{C}^2_{S_{\mathrm{f}}X}}\right)^2\right)^{\alpha}}{1 - \tfrac{\alpha \textsc{C}^2_{S_{\mathrm{f}}X}}{\lambda^{Z}_{\min} + \sigma^2} - \left(\tfrac{\alpha \textsc{C}^2_{S_{\mathrm{f}}X}}{2\left(\lambda^{Z}_{\min} + \sigma^2\right)}\right)^2\left(\tfrac{1-\textsc{C}^2_{S_{\mathrm{f}}X}}{\textsc{C}^2_{S_{\mathrm{f}}X}}\right)^2}\right)\\
    &\qquad = \frac{1}{2(\alpha - 1)}\cdot \left(\alpha\cdot \log\left(1 - \frac{1}{\gamma} - \frac{1}{4\gamma^2}\left(\frac{1 - \textsc{C}^2_{S_{\mathrm{f}}X}}{\textsc{C}^2_{S_{\mathrm{f}}X}}\right)^2\right)\right.\\
    & \qquad\qquad\qquad\qquad\qquad \left.- \log\left(1 - \frac{\alpha}{\gamma} - \frac{\alpha^2}{4\gamma^2} \left(\frac{1 - \textsc{C}^2_{S_{\mathrm{f}}X}}{\textsc{C}^2_{S_{\mathrm{f}}X}}\right)^2\right)\right)\\
    &\qquad\coloneq \phi\left(\alpha; \gamma, \textsc{C}^2_{S_{\mathrm{f}}X}\right). 
\end{align}
Finally, we place constraints on \(\alpha, \gamma\) for which the R\'{e}nyi divergence above is well-defined.
In particular, we require
\begin{equation*}
    1 - \frac{1}{\gamma} - \frac{1}{4\gamma^{2}}\left(\frac{1 - \textsc{C}^2_{S_{\mathrm{f}}X}}{\textsc{C}^2_{S_{\mathrm{f}}X}}\right)^{2} > 0 \quad\text{and}\quad 1 - \frac{\alpha}{\gamma} - \frac{\alpha^2}{4\gamma^2} \left(\frac{1 - \textsc{C}^2_{S_{\mathrm{f}}X}}{\textsc{C}^2_{S_{\mathrm{f}}X}}\right)^2 > 0~.
\end{equation*}
It can be verified that when
\begin{align}
    \gamma > \overline{\gamma}, \ \ \ 1 < \alpha < \nicefrac{\gamma}{\overline{\gamma}} 
\end{align}
where $\overline{\gamma} \coloneqq \frac{1}{2}\left(1 + \sqrt{1 + (\textsc{C}_{S_{\mathrm{f}}X}^{-2} - 1)^{2}}\right)$, this holds.

It remains to analyze the case where one row of $X'$ is zeroed out relative to $X$. In this case, $S_{\mathrm{f}}X' = S_{\mathrm{f}}X + S_{\mathrm{f}}e_i x_i^{\top} = S_{\mathrm{f}}X + u_i x_i^{\top}$. The analysis is unchanged, except for a sign change in $u_i$, which yields a sign change in $c$. By symmetry, it leads to the same final result. Since the bound is independent of the row removed from $X$, it applies whenever any point $i \in [n]$ is removed. This completes the proof of the closed-form and tightest guarantee. 

\paragraph{Simplifying the Bound.}\ 
Note that the previous proof established the upper bound 
\begin{equation}
    \underset{X'\simeq X}{\max} \ \KLDA{\pM_{S_{\mathrm{f}}}(X)}{\pM_{S_{\mathrm{f}}}(X')} \leq \phi\left(\alpha; \gamma, \textsc{C}^2_{S_{\mathrm{f}}X}\right)
\end{equation}
for all $\alpha \in (1, \nicefrac{\gamma}{\overline{\gamma}})$ provided that $\gamma > \overline{\gamma}$. First, since $\textsc{C}^2_{S_{\mathrm{f}}X} \geq 1$ and $\left(\frac{1 - \textsc{C}^2_{S_{\mathrm{f}}X}}{\textsc{C}^2_{S_{\mathrm{f}}X}}\right)^2 \leq 1$ for all $\textsc{C}^2_{S_{\mathrm{f}}X} \geq 1$, the previous monotonicity arguments further yields the bound 
\begin{align}
    \underset{(X,X'):\ X\simeq X'}{\max} \ &\KLDA{\pM_{S_{\mathrm{f}}}(X)}{\pM_{S_{\mathrm{f}}}(X')} \\
    &\leq \frac{1}{2(\alpha - 1)}\left(\alpha\cdot \log\left(1 - \frac{1}{\gamma} - \frac{1}{4\gamma^2}\right) - \log\left(1 - \frac{\alpha}{\gamma} - \frac{\alpha^2}{4\gamma^2}\right)\right) \\
    &\coloneq \overline{\phi}(\alpha;\gamma)
\end{align}
for all $\alpha \in (1, \nicefrac{4\gamma}{5})$ provided that $\gamma > \nicefrac{5}{4}$ and where these constraints were obtained since $\overline{\gamma} \leq \nicefrac{5}{4}$. Now, note that by Calim.~\ref{claim:upper_bound_phi_bar}, over this range for $\alpha$ and $\gamma$ it further holds that 
\begin{align}
    \overline{\phi}(\alpha; \gamma) \leq \frac{1}{2(\alpha - 1)}\left(\alpha \cdot \log\left(1 - \frac{5}{4\gamma}\right) - \log\left(1 - \frac{5\alpha}{4\gamma}\right)\right).
\end{align}
This upper bound matches the function $\varphi(\alpha; \nicefrac{4\gamma}{5})$ defined in \propositionref{prop:Privacy_First}. Thus, using \citet[Corollary.~1]{lev2025gaussianmix}, we know that this is upper bounded by 
\begin{align}
    \frac{1}{2(\alpha - 1)}\left(\alpha \cdot \log\left(1 - \frac{5}{4\gamma}\right) - \log\left(1 - \frac{5\alpha}{4\gamma}\right)\right) \leq \frac{25\alpha}{32\gamma^2}
\end{align}
for all $\alpha \in (1, \nicefrac{8\gamma}{25})$ and provided that $\gamma \geq \nicefrac{25}{8}$. This establishes the bound presented in \lemmaref{lem:privacy_lemma}. \hfill \qedsymbol 

\vspace{1.0em}
\begin{remark}
\normalfont
The same proof recovers the privacy analysis from \citet{lev2025gaussianmix} whenever $S_{\mathrm{f}}^{\top}S_{\mathrm{f}} = \brI_n$. In this case, we get $\textsc{C}^2_{S_{\mathrm{f}}X} = 1$, and the overall divergence becomes 
\begin{equation}
    \frac{1}{2(\alpha - 1)}\left(\alpha \cdot \log\left(1 - \frac{1}{\lambda^{X}_{\min} + \sigma^2}\right) - \log\left(1 - \frac{\alpha}{\lambda^{X}_{\min} + \sigma^2}\right)\right)
\end{equation}
and for $\lambda^{X}_{\min} + \sigma^2 > 1$ and $1 < \alpha < \lambda^{X}_{\min} + \sigma^2$, which recovers exactly the analysis from \citet[Appendix.~B]{lev2025gaussianmix}. We note that in this case, $S_{\mathrm{f}}$ is an orthogonal matrix, so $\mathsf{g}^{\top}S_{\mathrm{f}} \sim \pN(\vec{0}_n, \brI_n)$, which is distributionally equivalent to a standard normal vector, so the overall mechanism is procedurally the same as the Gaussian mixing mechanism from \citet{lev2025gaussianmix}. 
\end{remark}

\subsubsection{Supplementary Algebraic Claims}

Here, we state and prove algebraic claims that we have used in the proof of \theoremref{thm:Privacy_First}.

\begin{claim}
    \label{claim:renyi-div-upper-bound}
    Let \(\alpha > 1\). Suppose \(m, \ell > 0\) and \(M\) is a general constant satisfying $\abs{M} \leq m$. Consider the set \(\mathcal{D}_{1} \subseteq \reals^{3}\)
    \begin{align*}
        \mathcal{D}_{1} &= \left\{(a, b, c) \in \reals^{3}: \begin{array}{l}
            -\ell \leq c^{2} - ab \leq 0 \\
            -m \leq a + 2c \leq M \\
            1 + a + 2c + c^{2} - ab > 0 \\
            1 + \alpha(a + 2c) + \alpha^{2}(c^{2} - ab) > 0
        \end{array}\right\}~.
    \end{align*}
    Then,
    \begin{equation*}
        \sup_{(a, b, c) \in \mathcal{D}_{1}} \log\left(\frac{(1 + a + 2c + c^{2} - ab)^{\alpha}}{1 + \alpha(a + 2c) + \alpha^{2}(c^{2} - ab)}\right) \leq \log\left(\frac{(1 - m - \ell)^{\alpha}}{1 - \alpha m - \alpha^{2}\ell}\right).
    \end{equation*}
\end{claim}

\begin{proof}
For convenience, we use the following notation.
Let \(x = c^{2} - ab\), \(y = 2c + a\), \(A = 1 + x + y\), and \(B = 1 + \alpha y + \alpha^{2}x\).
The objective can be expressed as \(\alpha \cdot \log(A) - \log (B)\).
For any \(y\), 
\begin{equation*}
    \frac{\partial}{\partial x} \{\alpha \cdot \log(A) - \log(B)\} = \frac{\alpha}{A} - \frac{\alpha^{2}}{B} = \frac{\alpha}{AB}(B - \alpha A) = \frac{\alpha(\alpha - 1)}{AB}(-1 + \alpha x)~.
\end{equation*}
Since \(\alpha > 1\) and \(x = c^{2} - ab \leq 0\) according to the constraints, this partial derivative is non-positive.
As a result, we can eliminate \(x\) from the maximization to obtain the upper bound characterized by the problem
\begin{equation*}
    \sup_{y \in \mathcal{D}_{2}} \log\left(\frac{(1 + y - \ell)^{\alpha}}{1 + \alpha y - \alpha^{2}\ell}\right)
\end{equation*}
where \(\mathcal{D}_{2} := \{y \in \reals : -m \leq y \leq M, 1 + y - \ell > 0, 1 + \alpha y - \alpha^{2}\ell > 0\}\).

We take the derivative of this objective in \(y\) to obtain
\begin{equation}
    \label{eq:first_derivative_div}
    \frac{\partial}{\partial y} \log\left(\frac{(1 + y - \ell)^{\alpha}}{1 + \alpha y - \alpha^{2}\ell}\right) = \frac{\alpha(\alpha - 1) (y - (\alpha + 1) \ell)}{(1 + y - \ell) (1 + \alpha y - \alpha^{2}\ell)}.
\end{equation}
We note that this derivative attains a single \(0\) at \(y^{\star} = (\alpha + 1)\ell\), and moreover is negative for all $y<y^{\star}$ and positive for all $y>y^{\star}$. Therefore, since the objective is continuous and differentiable, $y^{\star}$ is the unique minimizer whenever it lies in the feasible domain; otherwise, the objective is monotone on the feasible interval. In either case, the supremum over the bounded feasible interval is attained on its boundary. 

We note that the interval boundaries in our case are given by $-m$ and $M$.
We demonstrate below that the maximum is attained at \(-m\).
To see this, let
 \begin{equation}
     \Delta(y)\coloneq \log\left(\frac{\left(1 - y - \ell\right)^{\alpha}}{1 - \alpha y - \alpha^2 \ell}\right) - \log\left(\frac{\left(1 + y - \ell\right)^{\alpha}}{1 + \alpha y - \alpha^2 \ell}\right),\ \ \ \ \ y\geq 0.
 \end{equation}
 Then, using \eqref{eq:first_derivative_div}, we note that 
 \begin{align}
     \frac{\partial\Delta(y)}{\partial y} &= \frac{\alpha(\alpha - 1)\left(y + (\alpha + 1)\ell\right)}{(1-y-\ell)(1 - \alpha y - \alpha^2 \ell)} -  \frac{\alpha(\alpha - 1)\left(y - (\alpha + 1)\ell\right)}{(1+y-\ell)(1 + \alpha y - \alpha^2\ell)}\\
     &= \frac{2\alpha(\alpha-1)(\alpha+1)\left(y^2+\ell(1 - \ell)(1 - \alpha^2 \ell)\right)}{(1-y-\ell)(1+y-\ell)(1-\alpha y-\alpha^2 \ell)(1+\alpha y-\alpha^2 \ell)}.
 \end{align}
Now, since the denominator is positive for all $y$ in the interval (since the positiveness of the different factors is required for the positiveness of the factors inside the $\log$), and since $\alpha > 1$, we note that the sign of this derivative is determined by the sign of $y^2 + \ell(1 - \ell)(1 - \alpha^2 \ell)$. However, since the definiteness of the $\log$ at $y = 0$ requires $1 - \alpha^2 \ell \geq 0$ and $1 - \ell \geq 0$ and since $\ell\geq 0$ we get that $\ell(1 - \ell)(1 - \alpha^2 \ell) \geq 0$, so $y^2 + \ell(1 - \ell)(1 - \alpha^2 \ell) \geq 0$, and this derivative is non-negative. Since $\Delta(0) = 0$, we get that $\Delta \geq 0$, so the negative branch is larger. Since $\abs{M} \leq m$, this together with the previous analysis of the derivative further says that the maximum of the overall objective is attained on $-m$. This completes the proof.
\end{proof}

\begin{claim}
    \label{claim:upper_bound_phi_bar}
    Let $x > \nicefrac{5}{4}$. Then, for all $\alpha \in (1, \nicefrac{4x}{5})$,
    \begin{align}
        \alpha \cdot \log\left(1 - \frac{1}{x} - \frac{1}{4x^2}\right) - \log\left(1 - \frac{\alpha}{x} - \frac{\alpha^2}{4x^2}\right) \leq \alpha \cdot \log\left(1 - \frac{5}{4x}\right) - \log\left(1 - \frac{5\alpha}{4x}\right).
    \end{align}
\end{claim}
\begin{proof}
Define \(g(z) := \log\left(1 - z - \frac{z^{2}}{4}\right) - \log\left(1 - \frac{5z}{4}\right)\) for \(z \in [0, \nicefrac{4}{5})\).
We note that the desired inequality is equivalent to
    \begin{align}
        \alpha \cdot g\left(\nicefrac{1}{x}\right) \leq g\left(\nicefrac{\alpha}{x}\right).
    \end{align}
    Since $g(0)  = 0$, the claim follows if $g(z)$ is convex over the range $z\in [0,\nicefrac{4}{5})$, since convexity then implies that for all $\alpha > 1$ 
    \begin{align}
        g(z) = g\left(\frac{1}{\alpha}(\alpha z) + \left(1 - \frac{1}{\alpha}\right)\cdot 0\right)  \leq \frac{1}{\alpha} \cdot g(\alpha z) + \left(1 - \frac{1}{\alpha}\right)g(0) 
    \end{align}
    so $\alpha\cdot g(z) \leq g(\alpha z)$, and setting \(z = \nicefrac{1}{x}\) would prove the statement of the lemma.
    
    The remainder of the proof focuses on the convexity of \(g\).
    The second derivative of \(g\) is
    \begin{align}
        g''(z) &= -\frac{2\left(4 - 4z - z^2\right) + (4+2z)^2}{\left(4 - 4z - z^2\right)^2} + \frac{25}{\left(4-5z\right)^2} = \frac{16 + 32z - 112z^2 + 80z^3 - 25z^4}{(4 - 5z)^2(4 - 4z - z^2)^2}.
    \end{align}
    The denominator is positive for \(z \in [0, \nicefrac{4}{5})\).
    The numerator has two real roots over \(\reals\) at \(z_{1} \approx 0.814\) and \(z_{2} \approx -0.247\).
    As a result, since on $z=0$ its positive, the quartic polynomial in the numerator is also non-negative in the range \(z \in [0, \nicefrac{4}{5})\), which establishes that \(g''(z) \geq 0\) in \(z \in [0, \nicefrac{4}{5})\).
    This asserts that \(g\) is convex, thus proving the claim.
\end{proof}
\subsection{Bound For $\textsc{C}^2_{S_{\mathrm{f}}X} \to \infty$}
\label{app:guarantee_CFX_infinity}
We note that the previous bound requires that the quantity $\underset{i\in [n]}{\max} \ \norm{Z^{\top}S_{\mathrm{f}}e_i - x_i}$ will be finite. However, as we show now, a finite divergence is also possible in the case where we do not rely on such a bound. This requires a particular analysis since taking $\textsc{C}^2_{S_{\mathrm{f}}X}\to \infty$ in the previous analysis yields $\gamma \to \infty$. However, as we show now, a more fine-grained analysis in this case leads to a finite divergence that depends on $\lambda^{Z}_{\min} + \sigma^2$. 

In particular, note that the next trivial bounds can be obtained on $s_1, s_2$, and $c$:
\begin{equation}
    s_1 \hspace{-0.15em}\leq\hspace{-0.15em} \frac{\norm{x_1}^2}{\lambda^{Z}_{\min} \hspace{-0.15em}+\hspace{-0.15em} \sigma^2} \hspace{-0.15em}\leq\hspace{-0.15em} \frac{1}{\lambda^{Z}_{\min} \hspace{-0.15em}+\hspace{-0.15em} \sigma^2},\ \ s_2 \hspace{-0.15em}\leq\hspace{-0.15em} \frac{\lambda_{\max}(Z^{\top}Z)}{\lambda_{\max}(Z^{\top}Z) \hspace{-0.15em}+\hspace{-0.15em} \sigma^2}\norm{u_1}^2 \hspace{-0.15em}\leq\hspace{-0.15em} \norm{u_1}^2 \hspace{-0.15em}\leq\hspace{-0.15em} 1, \ \ c \hspace{-0.15em}\leq\hspace{-0.15em} \sqrt{s_1s_2} \hspace{-0.15em}\leq\hspace{-0.15em} \left(\lambda^{Z}_{\min} \hspace{-0.15em}+\hspace{-0.15em} \sigma^2\right)^{-\nicefrac{1}{2}}. 
\end{equation}
Then, using the notation 
\begin{equation}
    \log\left(\frac{\big(1+ s_1 + 2c + c^2 - s_1 s_2\big)^{\alpha}}{1 + \alpha\left( s_1 + 2c\right) + \alpha^2 \left(c^2 - s_1 s_2\right)}\right)\coloneq \log\left(\frac{A^{\alpha}}{B}\right) = \alpha \cdot \log(A) - \log(B),
\end{equation}
we note that the divergence is a monotonically increasing function of $s_2$, since
\begin{equation}
    \frac{\partial}{\partial s_2} \left\{\alpha \cdot \log(A) - \log(B)\right\} = -\frac{\alpha s_1}{A} + \frac{\alpha^2 s_1}{B} = \frac{\alpha s_1}{AB}\left(\alpha A - B\right) = \frac{\alpha(\alpha - 1) s_1}{AB}\left(1 + \alpha(s_1s_2 - c^2) \right) \geq 0.
\end{equation}
Thus, we can substitute $s_2 = 1$ and to get the upper bound 
\begin{equation}
    \log\left(\frac{\big(1+ s_1 + 2c + c^2 - s_1 s_2\big)^{\alpha}}{1 + \alpha\left( s_1 + 2c\right) + \alpha^2 \left(c^2 - s_1 s_2\right)}\right) \leq 2\alpha \cdot \log(1+c) -\log\left(\left(1 + \alpha c\right)^2 + \alpha(1-\alpha)s_1\right). 
\end{equation}
Now, since $1 - \alpha \leq 0$, this function is monotonically increasing in $s_1$. Thus, we substitute $s_1 = \left(\lambda^{Z}_{\min} + \sigma^2\right)^{-1}$ and obtain the upper bound 
\begin{equation}
    \log\left(\frac{\big(1+ s_1 + 2c + c^2 - s_1 s_2\big)^{\alpha}}{1 + \alpha\left( s_1 + 2c\right) + \alpha^2 \left(c^2 - s_1 s_2\right)}\right)\leq 2\alpha \cdot \log(1+c) - \log\left(\left(1 + \alpha c\right)^2 + \tfrac{\alpha - \alpha^2}{\lambda^{Z}_{\min} + \sigma^2}\right). 
\end{equation}
Now, this bound can be optimized over $c$ under the constraint in which $\abs{c} \leq \left(\lambda^{Z}_{\min} + \sigma^2\right)^{-\nicefrac{1}{2}}$. However, by similar arguments as used before, we can prove that the maximum is attained on $c = -\left(\lambda^{Z}_{\min} + \sigma^2\right)^{-\nicefrac{1}{2}}$, which yields the closed-form bound  
\begin{align}
    &\KLDA{\mathsf{g}^{\top}S_{\mathrm{f}}X \hspace{-0.15em}+\hspace{-0.15em} \sigma \xi^{\top}}{\mathsf{g}^{\top}S_{\mathrm{f}}X' \hspace{-0.15em}+\hspace{-0.15em} \sigma \xi^{\top}}\\
    &\qquad\qquad\qquad\leq \frac{1}{2(\alpha \hspace{-0.15em}-\hspace{-0.15em} 1)}\hspace{-0.15em}\cdot\hspace{-0.15em} \left(\hspace{-0.15em}2\alpha \hspace{-0.15em}\cdot\hspace{-0.15em} \log\left(1 \hspace{-0.15em}-\hspace{-0.15em} \tfrac{1}{\sqrt{\lambda^{Z}_{\min} + \sigma^2}}\right) \hspace{-0.15em}-\hspace{-0.15em} \log\hspace{-0.15em}\left(\hspace{-0.15em}\left(1 \hspace{-0.15em}-\hspace{-0.15em} \tfrac{\alpha}{\sqrt{\lambda^{Z}_{\min} + \sigma^2}}\hspace{-0.15em}\right)^2 \hspace{-0.15em}+\hspace{-0.15em} \tfrac{\alpha - \alpha^2}{\lambda^{Z}_{\min} + \sigma^2}\hspace{-0.15em}\right)\hspace{-0.15em}\right)\\
    &\qquad\qquad\qquad= \frac{1}{2(\alpha \hspace{-0.15em}-\hspace{-0.15em} 1)}\hspace{-0.15em}\cdot\hspace{-0.15em} \left(\hspace{-0.15em}2\alpha \hspace{-0.15em}\cdot\hspace{-0.15em} \log\left(1 \hspace{-0.15em}-\hspace{-0.15em} \tfrac{1}{\sqrt{\lambda^{Z}_{\min} + \sigma^2}}\right) \hspace{-0.15em}-\hspace{-0.15em} \log\left(1 - \tfrac{2\alpha}{\sqrt{\lambda^{Z}_{\min} + \sigma^2}} + \tfrac{\alpha}{\lambda^{Z}_{\min} + \sigma^2}\hspace{-0.15em}\right)\hspace{-0.15em}\right)
\end{align}
and which is defined for all $1 < \alpha < \tfrac{\lambda^{Z}_{\min} + \sigma^2}{2\sqrt{\lambda^{Z}_{\min} + \sigma^2} - 1}$ provided that $\lambda^{Z}_{\min} + \sigma^2 > 1$. 

\bigskip 
\subsection{Auxiliary Lemmas}
We now provide three additional auxiliary lemmas, which will prove useful for making our overall privacy claim of \theoremref{thm:Privacy_First}. Throughout, we use the notion of $\ell_1$ sensitivity, which, given a target function $f(X)$, is defined by
\begin{align}
    \underset{(X,X'): X'\simeq X}{\max} \ \abs{f(X) - f(X')}.    
\end{align}

\begin{lem}
    \label{lem:operator_norms}
    Let $u,v\in\reals^{d}$. Then,
    \begin{align}
        \norm{uv^{\top}}_{\mathrm{op}} = \norm{u}\norm{v},
        \qquad
        \norm{uv^{\top} + vu^{\top}}_{\mathrm{op}}
        = \norm{u}\norm{v} + \abs{u^{\top}v}.
    \end{align}
\end{lem}

\begin{proof}
    If either $u=\vec{0}_d$ or $v=\vec{0}_d$, both identities are immediate. Hence, assume
    $u,v\neq \vec{0}_d$, and define
    \[
        \overline{u}\coloneq \frac{u}{\norm{u}},
        \qquad
        \overline{v}\coloneq \frac{v}{\norm{v}}.
    \]

    We first calculate the operator norm of $uv^{\top}$. We have
    \begin{align}
        uv^{\top}
        =
        \norm{u}\norm{v}\,
        \overline{u}\ \overline{v}^{\top}.
    \end{align}
    Since $\overline{u}\ \overline{v}^{\top}$ is a rank-one matrix with its only
    nonzero singular value equal to one, it follows that
    \begin{align}
        \norm{uv^{\top}}_{\mathrm{op}}
        =
        \norm{u}\norm{v}.
    \end{align}

    Next, consider
    \begin{align}
        A
        \coloneq
        uv^{\top}+vu^{\top}
        =
        \norm{u}\norm{v}
        \left(
        \overline{u}\ \overline{v}^{\top}
        +
        \overline{v}\ \overline{u}^{\top}
        \right).
    \end{align}
    Then, 
    \begin{align}
        \norm{A}_{\mathrm{op}} = \max_{\norm{x}=1}\abs{x^{\top}Ax} = 2\norm{u}\norm{v}\max_{\norm{x}=1}\abs{(\overline{u}^{\top}x)(\overline{v}^{\top}x)}.
    \end{align}
    Thus, it remains to show that
    \begin{align}
        \max_{\norm{x}=1}
        \abs{
        (\overline{u}^{\top}x)(\overline{v}^{\top}x)
        }
        =
        \frac{1+\abs{\overline{u}^{\top}\overline{v}}}{2}.
    \end{align}

    Let $\rho\coloneq \abs{\overline{u}^{\top}\overline{v}}$. For any
    $\norm{x}=1$, by $2|ab|\leq a^{2}+b^{2}$,
    \begin{align}
        2\abs{
        (\overline{u}^{\top}x)(\overline{v}^{\top}x)
        }
        &\leq
        (\overline{u}^{\top}x)^{2}
        +
        (\overline{v}^{\top}x)^{2} \\
        &=
        x^{\top}
        \left(
        \overline{u}\ \overline{u}^{\top}
        +
        \overline{v}\ \overline{v}^{\top}
        \right)x \\
        &\leq
        \lambda_{\max}
        \left(
        \overline{u}\ \overline{u}^{\top}
        +
        \overline{v}\ \overline{v}^{\top}
        \right).
    \end{align}
    The nonzero eigenvalues of
    $\overline{u}\ \overline{u}^{\top}
    +
    \overline{v}\ \overline{v}^{\top}$
    are $1+\rho$ and $1-\rho$, which follows since the eigenvectors of this matrix are $\frac{\overline{u} + \overline{v}}{\norm{\overline{u} + \overline{v}}}$ and $\frac{\overline{u} - \overline{v}}{\norm{\overline{u} - \overline{v}}}$. Hence,
    \begin{align}
        \max_{\norm{x}=1}
        \abs{
        (\overline{u}^{\top}x)(\overline{v}^{\top}x)
        }
        =
        \frac{1+\rho}{2}.
    \end{align}
    Substituting this into the previous equation gives
    \begin{align}
        \norm{uv^{\top}+vu^{\top}}_{\mathrm{op}}
        &=
        2\norm{u}\norm{v}\cdot
        \frac{1+\abs{\overline{u}^{\top}\overline{v}}}{2} \\
        &=
        \norm{u}\norm{v}
        +
        \abs{u^{\top}v}.
    \end{align}
    This completes the proof.
\end{proof}

\begin{lem}
\label{lem:sensitivity_M}
Let \(X\) be a dataset such that \(\|x_{i}\| \leq 1\) for all \(i \in [n]\).
    Suppose \(S_{\mathrm{f}} \in \reals^{k \times n}\) is a matrix such that \(\|S_{\mathrm{f}}e_{i}\| \leq 1\) for all \(i \in [n]\).
Then, 
\begin{align}
    \max_{(X,X'):\,X'\simeq X}\abs{\underset{i\in[n]}{\max}\norm{X^\top S_{\mathrm{f}}^{\top}S_{\mathrm{f}}e_i - x_i} - \underset{i\in[n]}{\max}\norm{X'^\top S_{\mathrm{f}}^{\top}S_{\mathrm{f}}e_i - x_i}} \leq \widehat{\Delta}
\end{align}
where $\widehat{\Delta} \coloneq \underset{i,j}{\max} \ \abs{e^{\top}_i\left(S_{\mathrm{f}}^{\top}S_{\mathrm{f}} - \brI_n\right)e_j}$.
\end{lem}

\begin{proof}
Choose an arbitrary $j\in[n]$ and define $P:=S_{\mathrm{f}}^\top S_{\mathrm{f}} - \brI_n$. For each $i\in[n]$, set
\begin{equation}
    v_i\coloneq X^\top P e_i,\qquad
    v_i'\coloneq \left(X - e_j x_j^\top\right)^\top P e_i = v_i - x_j\left(e_j^{\top} P e_i\right).
\end{equation}
We note that $v'_i$ is the quantity $f_i(X) := X^{\top}(S_{\mathrm{f}}^{\top}S_{\mathrm{f}} - \brI_n)e_i$, evaluated on the neighboring dataset $X'$ obtained by zeroing out the \(j^{th}\) row in $X$.
By the triangle and Cauchy-Schwarz inequalities, we get 
\begin{equation}
    \norm{v_i'}-\norm{v_i} \leq \abs{\norm{v_i'}-\norm{v_i}} \leq \norm{v_i' - v_i} = \norm{x_j}\abs{e_j^\top P e_i}~.
\end{equation}
Substituting the definition of \(P\) and the condition on \(\|x_{i}\|\),
\begin{equation}
    \abs{\norm{v_i'} - \norm{v_i}}\le \underset{i, j}{\max} \ \abs{e^{\top}_i\left(S_{\mathrm{f}}^{\top}S_{\mathrm{f}} - \brI_n\right)e_j}
\end{equation}
for every $i\in [n]$. Now, let $a_i\coloneq \norm{v_i'}$ and $b_i:=\norm{v_i}$. 
Then,
\begin{equation}
    \abs{\underset{i\in [n]}{\max} \ a_i - \underset{i\in [n]}{\max} \ b_i} \le \underset{i\in [n]}{\max} \ \abs{a_i - b_i} \le \underset{i, j}{\max} \ \abs{e^{\top}_i\left(S_{\mathrm{f}}^{\top}S_{\mathrm{f}} - \brI_n\right)e_j}.
\end{equation}
Since \(j\) is arbitrary, and since the bound is independent of the particular choice of \(X\), it remains valid after taking the maximum over all \(j\in[n]\) and all neighboring pairs \((X,X')\). Hence,
\begin{align}
    \max_{(X,X'):\,X'\simeq X}\left|\max_{i\in[n]} f_i(X)-\max_{i\in[n]} f_i(X')\right|
    \leq
    \max_{i,j}\left|e_i^{\top}\!\left(S_{\mathrm{f}}^{\top}S_{\mathrm{f}}-\brI_n\right)e_j\right|.
\end{align}
This is precisely the desired \(\ell_1\)-sensitivity.
\end{proof}

\bigskip 

\begin{lem}
\label{lem:sensitivity_F_eigenvalues}
Let $S_{\mathrm{f}}\in\mathbb{R}^{k\times n}$ satisfying $\|S_{\mathrm{f}}e_i\| \leq 1$ for all $i\in[n]$. Let $X\in\mathbb{R}^{n\times d}$ satisfying $\|x_i\|\le 1$ for all $i\in[n]$ and define $Z\coloneq S_{\mathrm{f}}X$. Set $m\coloneq \underset{i\in [n]}{\max} \ \norm{Z^{\top}S_{\mathrm{f}}e_i - x_i}$ and $\Lambda\coloneq 2 - \underset{i\in[n]}{\min} \norm{S_{\mathrm{f}}e_i}^2$. Then, 
\begin{align}
    \underset{X': X'\simeq X}{\max} \ \abs{\lambda^{S_{\mathrm{f}}X'}_{\min} - \lambda^{S_{\mathrm{f}}X}_{\min}} \leq \Lambda + 2m.
\end{align}
\end{lem}

\begin{proof}
Denoting $X^{\top}S_{\mathrm{f}}^{\top}S_{\mathrm{f}}e_i = x_i + \delta_i$ and recalling that $\norm{S_{\mathrm{f}}e_i}^2 \leq 1$ for all $i\in [n]$ we note that 
\begin{align}
    \left(X - e_ix^{\top}_i\right)^{\top}S_{\mathrm{f}}^{\top}S_{\mathrm{f}}\left(X - e_ix^{\top}_i\right) &= X^{\top}S_{\mathrm{f}}^{\top}S_{\mathrm{f}}X + \norm{S_{\mathrm{f}}e_i}^2x_ix^{\top}_i - X^{\top}S_{\mathrm{f}}^{\top}S_{\mathrm{f}}e_ix^{\top}_i - x_i e^{\top}_iS_{\mathrm{f}}^{\top}S_{\mathrm{f}}X\\
    &= X^{\top}S_{\mathrm{f}}^{\top}S_{\mathrm{f}}X + \norm{S_{\mathrm{f}}e_i}^2x_ix^{\top}_i - (x_i + \delta_i)x^{\top}_i - x_i (x_i + \delta_i)^{\top}\\
    &= X^{\top}S_{\mathrm{f}}^{\top}S_{\mathrm{f}}X + \left(\norm{S_{\mathrm{f}}e_i}^2 - 2\right)x_ix^{\top}_i - \delta_i x^{\top}_i - x_i\delta^{\top}_i. 
\end{align}
Now, by Weyl's inequality \citep[Theorem.~1]{stewart1998perturbation} for matrices $X^{\top}X$ and $Y^{\top}Y$ it holds
\begin{equation}
    \abs{\lambda^{X}_{\min} - \lambda^{Y}_{\min}} \leq \norm{X^{\top}X - Y^{\top}Y}_{\mathrm{op}}. 
\end{equation}
Thus, 
\begin{align}
    \abs{\lambda^{Z - S_{\mathrm{f}}e_ix^{\top}_i}_{\min} - \lambda^{Z}_{\min}} &\leq \norm{\left(\norm{S_{\mathrm{f}}e_i}^2 - 2\right)x_ix^{\top}_i - \delta_i x^{\top}_i - x_i\delta^{\top}_i}_{\mathrm{op}}\\
    &\leq \abs{\norm{S_{\mathrm{f}}e_i}^2 - 2}\norm{x_ix^{\top}_i}_{\mathrm{op}} + \norm{\delta_i x^{\top}_i + x_i \delta^{\top}_i}_{\mathrm{op}}\\
    &= \abs{\norm{S_{\mathrm{f}}e_i}^2 - 2}\norm{x_i}^2 + \norm{\delta_i}\norm{x_i} + \abs{\delta^{\top}_i x_i}\\  
    &\leq \abs{\norm{S_{\mathrm{f}}e_i}^2 - 2}\norm{x_i}^2 + 2\norm{\delta_i}\norm{x_i}\\
    &\leq \underset{i\in[n]}{\max} \ \abs{\norm{S_{\mathrm{f}}e_i}^2 - 2} + 2m\\
    &\leq \left(2 - \underset{i\in[n]}{\min} \ \norm{S_{\mathrm{f}}e_i}^2\right) + 2m\\
    &= \Lambda + 2m
\end{align}
where we have used the triangle inequality, \lemmaref{lem:operator_norms}, the assumptions $\norm{x_i} \hspace{-0.2em}\leq\hspace{-0.2em} 1$ and $\norm{S_{\mathrm{f}}e_i} \hspace{-0.2em}\leq\hspace{-0.2em} 1$ for all $i\hspace{-0.2em}\in\hspace{-0.2em} [n]$, and the definitions of $m$ and $\Lambda$. 

\end{proof}

\bigskip 
The next lemma provides a variant of the adaptive composition theorem in which the conditional privacy guarantee for a later release may fail on a rare event involving an earlier private release. We use this lemma to establish privacy when a privately computed, high-probability upper bound on the local sensitivity is used to calibrate a subsequent release.
\begin{lem}
    \label{lem:privacy_rare_event_adaptive_composition}
    Let $\pM_1: \pX^{n} \to \pY$ be $(\eps_1, \delta_1)$-\DP. For every dataset $X$, let $\pG_{X} \subseteq \pY$ be a measurable subset satisfying 
    \begin{align}
        \label{eq:bad_event_pm_1}
        \mathbb{P}\left(\pM_1(X)\in \pG_{X}\right) \geq 1 - \overline{\delta}
    \end{align}
    for $\overline{\delta}$ that is independent of $X$. Let $\pM_2: \pX^{n}\times \pY \to \pZ$ be a randomized mechanism satisfying that for every $X\simeq X', y\in \pG_{X}$ and measurable subset $\pA\subseteq \pZ$ it holds that 
    \begin{align}
        \label{eq:second_mechanism_dp}
        \mathbb{P}\left(\pM_2(X, y) \in \pA\right) \leq e^{\eps_2}\cdot \mathbb{P}\left(\pM_2(X', y) \in \pA\right) + \delta_2.
    \end{align}
    Then, $\left(\pM_1(X), \pM_2(X, \pM_1(X))\right)$ is $\left(\eps_1+\eps_2, \delta_1 + \delta_2 + \overline{\delta}\right)$-\DP.
\end{lem}

\begin{proof}
    Let $\pS \subseteq \pY\times \pZ$ and $\pM(X) \coloneqq \left(\pM_1(X), \pM_2\left(X, \pM_1(X)\right)\right)$. Following the $(\eps,\delta)$-DP definition (\definitionref{defn:eps_delta_DP}), our goal is to prove that
    \begin{align}
        \mathbb{P}\left(\pM(X) \in \pS\right) \leq e^{\eps}\cdot \mathbb{P}\left(\pM(X') \in \pS\right) + \delta
    \end{align}
    for $\eps = \eps_1 + \eps_2$ and $\delta = \delta_1 + \delta_2 + \overline{\delta}$ and for all $X\simeq X'$ and $\pS \subseteq \pY\times \pZ$. We first note that by \eqref{eq:second_mechanism_dp}, 
    \begin{align}
        \label{eq:m2_dp_tight_bound}
        \mathbb{P}\left(\pM_2(X, y) \in \pA\right) &\leq \min\left\{1, e^{\eps_2}\cdot \mathbb{P}\left(\pM_2(X', y) \in \pA\right) + \delta_2\right\}\\
        &\leq \min\left\{1, e^{\eps_2}\cdot \mathbb{P}\left(\pM_2(X', y) \in \pA\right)\right\} + \delta_2
    \end{align}
    for all $\pA \subseteq \pZ$. Moreover, for any $\pC\subseteq\pY$ we define
    \begin{align}  
        \label{eq:mu_event_definition}
        \mu_1(\pC) := \underset{\pB\subseteq\pC}{\sup} \ \left\{\mathbb{P}\left(\pM_1(X)\in \pB\right) - e^{\eps_1}\cdot \mathbb{P}\left(\pM_1(X')\in \pB\right)\right\} 
    \end{align}
    and we note that this definition together with the $(\eps,\delta)$-DP definition implies $\mu_1(\pY) \leq \delta_1$ (see similar definition and its interpretation in \citep{balle2018privacy}).

    We denote $\pS_y := \left\{w\in \pZ: (y,w) \in \pS\right\}$ and $q_X(y) := \mathbb{P}\left(\pM_2(X,y) \in \pS_y\right)$. Then, we note that 
    \begin{align}
        \mathbb{P}\left(\pM(X) \in \pS\right) &\underset{(a)}{=} \int_{y\notin \pG_X}q_X(y)\mathbb{P}\left(\pM_1(X) \in dy\right)  + \int_{y\in \pG_X}q_X(y)\mathbb{P}\left(\pM_1(X) \in dy\right)\\
        &\leq \mathbb{P}\left(\pM_1(X) \notin \pG_X\right)  + \int_{y\in \pG_X}q_X(y)\mathbb{P}\left(\pM_1(X) \in dy\right)\\
        &\underset{(b)}{\leq} \overline{\delta} + \int_{y\in \pG_X}q_X(y)\mathbb{P}\left(\pM_1(X) \in dy\right)\\
        &\underset{(c)}{=} \overline{\delta} + \int_{y\in \pG_X}\mathbb{P}\left(\pM_2(X,y) \in \pS_y\right)\mathbb{P}\left(\pM_1(X) \in dy\right)\\
        &\underset{(d)}{\leq} \overline{\delta} + \int_{y\in \pG_X}\left(\min\left\{1,e^{\eps_2}\cdot \mathbb{P}\left(\pM_2\left(X', y\right)\in \pS_y\right) \right\} + \delta_2\right) \mathbb{P}\left(\pM_1(X) \in dy\right)\\
        &\underset{(e)}{\leq} \overline{\delta} + \delta_2 + \int_{y\in \pG_X}\min\left\{1,e^{\eps_2}\cdot \mathbb{P}\left(\pM_2\left(X', y\right)\in \pS_y\right) \right\} \mathbb{P}\left(\pM_1(X) \in dy\right)\\
        &\underset{(f)}{\leq} \overline{\delta} + \delta_2 + \int_{y\in \pG_X}\min\left\{1,e^{\eps_2}\cdot \mathbb{P}\left(\pM_2\left(X', y\right)\in \pS_y\right) \right\} \left(e^{\eps_1}\cdot \mathbb{P}\left(\pM_1\left(X'\right) \in dy\right) + \mu_1(dy)\right)\\
        &\leq \overline{\delta} + \delta_2 + \int_{y\in \pG_X}\left(e^{\eps_1 + \eps_2}\cdot \mathbb{P}\left(\pM_2\left(X', y\right)\in \pS_y\right)\mathbb{P}\left(\pM_1\left(X'\right) \in dy\right) + \mu_1(dy)\right) \\
        &\leq \overline{\delta} + \delta_2 + \mu_1\left(\pG_X\right) + e^{\eps_1 + \eps_2}\cdot \int_{y\in \pG_X}\mathbb{P}\left(\pM_2\left(X', y\right)\in \pS_y\right)\mathbb{P}\left(\pM_1\left(X'\right) \in dy\right)\\
        &\underset{(g)}{\leq} \overline{\delta} + \delta_1 + \delta_2 + e^{\eps_1 + \eps_2}\cdot \int_{y\in \pG_X}\mathbb{P}\left(\pM_2\left(X', y\right)\in \pS_y\right)\mathbb{P}\left(\pM_1\left(X'\right) \in dy\right)\\
        &\underset{(h)}{\leq} \overline{\delta} + \delta_1 + \delta_2 + e^{\eps_1 + \eps_2}\cdot \mathbb{P}\left(\pM(X') \in \pS\right).
    \end{align}
    where (a) is by using decomposing the event $\pM(X)\in \pS$ to $\pS = \left\{\pS_y \cap \left(y\in\pG_X\right)\right\} \cup \left\{\pS_y \cap \left(y\notin\pG_X\right)\right\}$, (b) is by using \eqref{eq:bad_event_pm_1}, (c) is by the definition of $q_X(y)$ and $\pS_y$, (d) is by using \eqref{eq:m2_dp_tight_bound}, (e) is by using the upper bound $\mathbb{P}\left(\pM_1(X) \in \pG_X\right) \leq 1$, (f) is by \eqref{eq:mu_event_definition} where $dy$ is an infinitesimal neighborhood around the point $y$, (g) is since $\mu_1(\pY) \leq \delta_1$ and (h) is by the law of total probability. 
\end{proof}
\subsection{Proof of \theoremref{thm:Privacy_First}}
\label{app:prof_theorem_1} 

\subsubsection{Differential Privacy}
\label{app:DP_guarantee_fastmix}
We recall certain definitions here for the convenience of the reader.
\begin{center}
\renewcommand{\arraystretch}{1.5}
\begin{tabular}{c|l}
    \(\widehat{\Lambda}\) & \(2 - \underset{i \in [n]}{\min}\  \norm{S_{\mathrm{f}}e_{i}}^{2}\) \\
    \(\widehat{\Delta}\) & \(\underset{i, j \in [n]}{\max}\ \abs{e_{i}^{\top}(S_{\mathrm{f}}^{\top}S_{\mathrm{f}} - \brI_{n})e_{j}}\) \\
    \(\widehat{m}\) & \(\underset{i \in [n]}{\max}\  \norm{X^{\top}S_{\mathrm{f}}^{\top}S_{\mathrm{f}}e_{i} - x_{i}}\) \\
    \(\widetilde{m}\) & \(\max\left\{0, \widehat{m} + \omega \widehat{\Delta}(\tau - \mathsf{z}_{1})\right\}\) where \(\mathsf{z}_{1} \sim \mathrm{Lap}(0, 1)\) \\
     $\widehat{\lambda}$ & $\lambda^{S_{\mathrm{f}}X}_{\min}$\\
    \(\widetilde{\lambda}\) & \(\max\left\{0, \widehat{\lambda} - \omega(\widehat{\Lambda} + 2\widetilde{m})(\tau - \mathsf{z}_{2})\right\}\) where  \(\mathsf{z}_{2} \sim \mathrm{Lap}(0, 1)\).
\end{tabular}
\end{center}

Let $\mathcal{E} \hspace{-0.2em}:=\hspace{-0.2em} \left\{\left(\widetilde{\lambda} \hspace{-0.15em}\leq\hspace{-0.15em} \widehat{\lambda}\right)\hspace{-0.15em}\bigcap\hspace{-0.15em} \left(\widetilde{m} \hspace{-0.15em}\geq\hspace{-0.15em} \widehat{m}\right)\right\}$. First, note that for the $\tau$ stated in the theorem it holds that 
\begin{align}
    \mathbb{P}\left(\widehat{\lambda} \leq \widetilde{\lambda}\right) = \mathbb{P}\left(\mathsf{z}_2 \geq \tau\right) = \frac{1}{2}\exp\left\{-\tau\right\} = \frac{\delta}{3}
\end{align}
and 
\begin{align}
    \mathbb{P}\left(\widetilde{m} \leq \widehat{m}\right) = \mathbb{P}\left(\mathsf{z}_1 \geq \tau\right) = \frac{1}{2}\exp\left\{-\tau\right\} = \frac{\delta}{3}.
\end{align}
Then, by using a union-bound, it holds that
\begin{equation}
    \label{eq:prob_bound_m_tilde}
    \mathbb{P}\left(\mathcal{E}^{\mathrm{c}}\right) \leq \mathbb{P}\left(\widehat{\lambda} \leq \widetilde{\lambda}\right) + \mathbb{P}\left(\widetilde{m}\leq \widehat{m}\right) = \mathbb{P}\left(\mathsf{z}_1 \geq \tau \right) + \mathbb{P}\left(\mathsf{z}_2 \geq \tau \right) \leq \frac{2\delta}{3}.
\end{equation}

To establish the privacy guarantee of the algorithm, we consider its three private releases: $\widetilde{m}$, $\widetilde{\lambda}$, and the output of the final sketching step. We analyze each release and their combination using \lemmaref{lem:privacy_rare_event_adaptive_composition}, which accounts for both their adaptive dependence and the associated rare failure event.

\paragraph{Privacy Guarantees for the Releases $\widetilde{m}$ and $\widetilde{\lambda}.$} \ 

By \lemmaref{lem:sensitivity_M}, the $\ell_1$ sensitivity of $\widehat{m}$ is upper bounded by $\widehat{\Delta}$; Thus, $\widetilde{m}$ is $\frac{1}{\omega}$-DP release of $\widehat{m}$ by using classical properties of the Laplace mechanism \citep[Section.~3.3]{Dwork_AlgFoundationd_DP}.

By \lemmaref{lem:sensitivity_F_eigenvalues}, the $\ell_1$ sensitivity of $\widehat{\lambda}$ is $\widehat{\Lambda} \hspace{-0.2em}+\hspace{-0.2em} 2\widehat{m}$. Thus, using the same properties of the Laplace mechanism we get that $\widetilde{\lambda}$ is $\frac{1}{\omega}$-DP release of $\widehat{\lambda}$ for every fixed $\widetilde{m}$ satisfying $\widetilde{m} \geq \widehat{m}$. 


\paragraph{Privacy Guarantees of the Sketching Step.}\ 

The $S_{\mathrm{f}}$ considered in the theorem satisfies the unit column-norm condition required by \lemmaref{lem:privacy_lemma}. Then, for every fixed $\widetilde{\eta}$ satisfying $\widehat{\lambda} + \widetilde{\eta}^2 \geq \left(1 + 2\widehat{m}\right)\gamma$ the privacy guarantee of  \lemmaref{lem:privacy_lemma} holds with respect to every row in the output $\mathsf{S}_{\mathrm{G}}S_{\mathrm{f}}X + \widetilde{\eta}\xi$. Since the output comprised of $k_1$ independnt rows of the Gaussian sketch applied over $S_{\mathrm{f}}X$, using composition for the $k_1$ different rows we get that the output satisfies $\left(\alpha, k_1\hspace{-0.2em}\cdot\hspace{-0.2em} \overline{\phi}\left(\alpha; \gamma\right)\right)$-RDP for $1 \hspace{-0.2em}<\hspace{-0.2em} \alpha \hspace{-0.2em}<\hspace{-0.2em} \nicefrac{4\gamma}{5}$, where we have used the upper bound $\phi\left(\alpha; \gamma, \textsc{C}^2_{S_{\mathrm{f}}X}\right) \hspace{-0.2em}\leq\hspace{-0.2em} \overline{\phi}(\alpha; \gamma)$. Using the conversion from RDP to $(\eps, \delta)$-DP of \propositionref{prop:Renyi_classical_translate} yields that the sketching step is $(\eps_{\mathrm{s}}, \delta_{\mathrm{s}})$-DP with $\delta_{\mathrm{s}} = \nicefrac{\delta}{3}$ and 
\begin{align}
    \eps_{\mathrm{s}} = \underset{1 < \alpha < \nicefrac{4\gamma}{5}}{\min} \ \left\{k_1 \overline{\phi}(\alpha; \gamma) + \frac{\log(\nicefrac{3}{\delta}) + (\alpha - 1)\log(1 - \nicefrac{1}{\alpha}) - \log(\alpha)}{\alpha - 1}\right\},
\end{align}
provided that $\widehat{\lambda} + \widetilde{\eta}^2 \geq \left(1 + 2\widehat{m}\right)\gamma$. 


\paragraph{Composition.}\ 

To finalize the privacy argument, we use \lemmaref{lem:privacy_rare_event_adaptive_composition} with respect to the overall release $\left(\widetilde{m}, \widetilde{\lambda}, \mathsf{S}_{\mathrm{G}}S_{\mathrm{f}}X + \widetilde{\eta}\xi\right)$. First, note that by the previous argument we get that $\left(\widetilde{m}, \widetilde{\lambda}\right)$ is $(\nicefrac{2}{\omega}, \nicefrac{\delta}{3})$-DP by using \lemmaref{lem:privacy_rare_event_adaptive_composition} with $\pM_1(X) \gets \widetilde{m}, \pG_X \gets \left\{s\in\reals: \ s\geq \widehat{m}\right\}$ and $\pM_2 \gets \widetilde{\lambda}$, where $\mathbb{P}\left(\widetilde{m} \in \pG_X\right)$ was calculated in \eqref{eq:prob_bound_m_tilde}.  

Now, let $\pM_1(X) \gets \left(\widetilde{m}, \widetilde{\lambda}\right), \pG_{X} \gets \left\{(s,t) \in \reals^{2}: \ s \geq \widehat{m}, t\leq \widehat{\lambda}\right\}$ and $\pM_2(X, \widetilde{\eta}) \gets \mathsf{S}_{\mathrm{G}}S_{\mathrm{f}}X + \widetilde{\eta}\xi$. We note that the $\widetilde{\eta}$ defined in Line.~\ref{line:compute_eta_tilde} satisfies $\widetilde{\eta}^2 + \widetilde{\lambda} \geq \gamma \left(1 + 2\widetilde{m}\right)$, which under the event $\left\{\pM_1(X) \in \pG_{X}\right\}$ ensures that $\widetilde{\eta}^2 + \widehat{\lambda} \geq \gamma \left(1 + 2\widehat{m}\right)$. Then, since by the previous arguments we know that $\pM_1(X)$ is $(\nicefrac{2}{\omega}, \nicefrac{\delta}{3})$-DP, by \lemmaref{lem:privacy_rare_event_adaptive_composition} we get that the overall output satisfies $(\widehat{\eps}(\omega, \gamma, k_1, \delta), \delta)$-DP, where 
\begin{align}
    \widehat{\eps}(\omega, \gamma, k_1, \delta) &\coloneq \frac{2}{\omega} + \underset{1 < \alpha < \nicefrac{4\gamma}{5}}{\min} \ \left\{k_1 \overline{\phi}(\alpha; \gamma) + \frac{\log(\nicefrac{3}{\delta}) + (\alpha - 1)\log(1 - \nicefrac{1}{\alpha}) - \log(\alpha)}{\alpha - 1}\right\}, \\ 
    \overline{\phi}(\alpha; \gamma) &\coloneq \frac{1}{2(\alpha \hspace{-0.15em}-\hspace{-0.15em} 1)}\hspace{-0.15em}\cdot\hspace{-0.15em} \left(\hspace{-0.15em}\alpha\hspace{-0.15em}\cdot\hspace{-0.15em} \log\hspace{-0.15em}\left(1 \hspace{-0.15em}-\hspace{-0.15em} \frac{1}{\gamma} \hspace{-0.15em}-\hspace{-0.15em} \frac{1}{4\gamma^2}\hspace{-0.15em}\right) \hspace{-0.15em}-\hspace{-0.15em} \log\left(\hspace{-0.15em}1 \hspace{-0.15em}-\hspace{-0.15em} \frac{\alpha}{\gamma} \hspace{-0.15em}-\hspace{-0.15em} \frac{\alpha^2}{4\gamma^2}\hspace{-0.15em}\right)\hspace{-0.15em}\right)
\end{align}
and which holds for all $\gamma > \nicefrac{5}{4}$. Substituting the upper bound on $\overline{\phi}$ derived in \appendixref{app:proof_of_aux_privacy_lemma} yields 
\begin{align}
    \eps(\omega, \gamma, k_1, \delta) &\coloneq \frac{2}{\omega} + \underset{1 < \alpha < \nicefrac{8\gamma}{25}}{\min} \ \left\{\frac{25\alpha k_1}{32\gamma^2} + \frac{\log(\nicefrac{3}{\delta}) + (\alpha - 1)\log(1 - \nicefrac{1}{\alpha}) - \log(\alpha)}{\alpha - 1}\right\}\\
    &\leq \frac{2}{\omega} + \underset{1 < \alpha < \nicefrac{8\gamma}{25}}{\min} \ \left\{\frac{25\alpha k_1}{32\gamma^2} + \frac{\log(\nicefrac{3}{\delta})}{\alpha - 1}\right\}.
\end{align}
Substutitung $\alpha = \min\left\{\frac{8\gamma}{25}, 1 + \gamma\sqrt{\frac{32\log(\nicefrac{3}{\delta})}{25k_1}}\right\}$ yields the upper bound of \theoremref{thm:Privacy_First}.

\subsubsection{Runtime Complexity}  
    We provide the proof by explicitly writing the complexity of the different steps inside \FastMix and their runtime complexity. 

    \paragraph{Computing $Z$.} 
    Involves applying an SRHT of size $(k_2, n)$ to $X\in \reals^{n\times d}$. Following classical arguments (see, for example, \citet{ailon2009fast}), the runtime complexity of this step is $O(nd\log(k_2))$. However, in our implementation, we first form the full matrix $H_n\mathsf{B}_nX$ and then subsample its rows according to the non-zero locations in the matrix $\mathsf{P}_{k\times n}$, since we use the values computed for $H_n\mathsf{B}_nX$ in our subsequent steps. Thus, the complexity of this step is $O(nd\log(n))$. 

    \paragraph{Computing $\widetilde{\lambda}$.} 
    Since $Z\in\reals^{k_2\times d}$, one can compute $Z^{\top}Z$ in $O(k_2d^2)$ time and then compute its minimal eigenvalue in $O(d^3)$ (via standard eigensolvers). Thus, the total complexity is $O(k_2d^2 + d^3)$. In the regime $k_2 \ge d$ $\big($which is necessary for ensuring that $\lambda_{\min}^{X} \approx \lambda_{\min}^{Z}$$\big)$, this simplifies to $O(k_2d^2)$. We note that in this regime, this can be achieved via a single SVD of $Z$, which costs $O(k_2 d^2)$.

    \paragraph{Computing $\widehat{m}$ and $\widehat{\Delta}$.} 
    For computing $\widehat{m}$, we first note that 
    \begin{align}
        Z^{\top}S_{\mathrm{f}}e_i \hspace{-0.2em}-\hspace{-0.2em} x_i \hspace{-0.2em}=\hspace{-0.2em}X^{\top}\hspace{-0.3em}\left(\frac{n}{k} \mathsf{B}^{\top}_nH^{\top}_n\mathsf{P}^{\top}_{k\times n}\mathsf{P}_{k\times n}H_n\mathsf{B}_n \hspace{-0.2em}-\hspace{-0.2em} \brI_n\hspace{-0.2em}\right)\hspace{-0.2em}e_i \hspace{-0.2em}=\hspace{-0.2em} (H_n\mathsf{B}_nX)^{\top}\hspace{-0.3em}\left(\frac{n}{k} \mathsf{P}^{\top}_{k\times n}\mathsf{P}_{k\times n} \hspace{-0.2em}-\hspace{-0.2em} \brI_n\hspace{-0.2em}\right)\hspace{-0.2em}H_n\mathsf{B}_ne_i.
    \end{align}
    Then, since \(\mathsf P_{k\times n}\) selects \(k\) rows, the matrix $\frac{n}{k}\mathsf P_{k\times n}^{\top}\mathsf P_{k\times n}-\brI_n$ is diagonal, with exactly \(k\) nonzero diagonal corrections relative to \(-\brI_n\). Hence, multiplying by this matrix amounts only to rescaling selected rows (equivalently, columns after transposition), so forming $(H_n\mathsf B_nX)^{\top}\!\left(\frac{n}{k}\mathsf P_{k\times n}^{\top}\mathsf P_{k\times n}-\brI_n\right)$ costs \(O(nd)\). Then, since we have $H_n\mathsf{B}_nX$ from the previous computations, the overall computations requires applying the Hadamard transform $H_nB_n$ on the matrix $(H_n\mathsf{B}_nX)^{\top}\left(\frac{n}{k} \mathsf{P}^{\top}_{k\times n}\mathsf{P}_{k\times n} - \brI_n\right)$, which costs $O(nd\log(n))$, and then computing the norms of the columns of this matrix which takes another $O(nd)$. Thus, the overall complexity of this step is $O(nd\log(n))$. 
    
    Computing $\widehat{\Delta}$ requires computing the coherence $\widehat{\Delta} = \underset{i\neq j}{\max}\ \abs{e^{\top}_i S_{\mathrm{f}}^{\top}S_{\mathrm{f}}e_j}.$ First, we note that if $S_{\mathrm{f}} = \sqrt{\frac{n}{k}} \mathsf{P}_{k\times n}H_n\mathsf{B}_n$ we have $S_{\mathrm{f}}^{\top}S_{\mathrm{f}} = \frac{n}{k} \mathsf{B}_n H_n^{\top}\mathsf{P}^{\top}_{k\times n}\mathsf{P}_{k\times n}H_n\mathsf{B}_n.$ Since $\mathsf{B}_n$ is diagonal with Rademacher entries, multiplying on the left and right by $\mathsf{B}_n$ only changes the signs of the entries, and therefore does not affect their absolute values. Hence,
    \begin{equation}
        \widehat{\Delta} = \frac{n}{k} \max_{i\neq j} \abs{e^{\top}_i H_n^{\top}\mathsf{P}^{\top}_{k\times n}\mathsf{P}_{k\times n}H_n e_j}.
    \end{equation}
    
    Now let $z\in\{0,1\}^n$ denote the indicator vector of the sampled rows, namely
    \begin{equation}
        z_r = 1 \quad\text{if row }r\text{ is selected by }\mathsf{P}_{k\times n}, \qquad z_r = 0 \text{ otherwise.}
    \end{equation}
    Writing $h_i$ for the $i$-th column of $H_n$, we obtain
    \begin{equation}
        e^{\top}_i H_n^{\top}\mathsf{P}^{\top}_{k\times n}\mathsf{P}_{k\times n}H_n e_j = \sum_{r=1}^n z_r (h_i)_r (h_j)_r.
    \end{equation}
    Since $\mathsf{P}_{k\times n}H_n$ is obtained by selecting $k$ rows of the Hadamard matrix, and since the entry-wise product of two Hadamard columns is, up to the factor $1/\sqrt{n}$, another Hadamard column\footnote{this holds since, by the definition of the elements of the Hadamard matrix it holds that $(H_n)_{i,j}(H_n)_{k,j} = (-1)^{\langle i,j\rangle_2 + \langle k,j\rangle_2} = (-1)^{\langle i\oplus k,j\rangle_2}$, corresponding to the element $(H_n)_{i\oplus k, j}$ and where $\langle a, b\rangle_2$ denotes the mod-$2$ inner product between the binary expansions of $a$ and $b$}, for every $i\neq j$ there exists a non-constant Hadamard column $h_\ell$ such that $h_i \odot h_j = \frac{1}{\sqrt{n}} h_\ell$ where $a \odot b$ denotes entry-wise product between the vectors $a$ and $b$. Therefore,
    \begin{equation}
        e^{\top}_i H_n^{\top}\mathsf{P}^{\top}_{k\times n}\mathsf{P}_{k\times n}H_n e_j = \frac{1}{\sqrt{n}} \sum_{r=1}^n z_r (h_\ell)_r = \frac{1}{\sqrt{n}} (H_n z)_\ell.
    \end{equation}
    It follows that
    \begin{equation}
        \widehat{\Delta} = \frac{\sqrt{n}}{k} \max_{\ell \neq 0} \abs{(H_n z)_\ell},
    \end{equation}
    where the maximum is taken over the values of the Hadamard transformed vector $H_nz$.
    
    Hence, computing $\widehat{\Delta}$ reduces to:
    (i) forming the indicator vector $z$ of the sampled rows,
    (ii) applying one Hadamard transform to $z$, and
    (iii) taking the maximum absolute value over the resulting outputs.
    The overall complexity is therefore $O(n\log n)$ time and $O(n)$ memory.
    
    \paragraph{Computing $\mathsf{S}_{\mathrm{G}}Z$.} 
    Since the Gaussian sketch $\mathsf{S}_{\mathrm{G}}$ is a dense matrix of size $k_1 \times k_2$, and since it is applied to a matrix of size $k_2 \times d$, this complexity is $O(k_1k_2d)$. 

    \paragraph{Overall Computation.} 
    Combining these terms leads to the next overall complexity 
    \begin{equation}
        O\left(k_1k_2d + nd\log(n) + k_2 d^2 + d^3\right).
    \end{equation}
\section{Simplifying $\phi$ and Closed-Form $(\eps, \delta)$-DP Guarantees}
\label{app:tcdp_bound}
Recall that we have established the upper bound 
\begin{align}
    \underset{X'\simeq X}{\max} \KLDA{\pM_{S_{\mathrm{f}}(X)}}{\pM_{S_{\mathrm{f}}}(X')} \hspace{-0.2em}\leq\hspace{-0.2em}  \frac{1}{2(\alpha\hspace{-0.2em}-\hspace{-0.2em}1)}\hspace{-0.15em}\cdot\hspace{-0.15em} \left(\hspace{-0.2em}\alpha\hspace{-0.15em}\cdot\hspace{-0.15em}\log\left(1 \hspace{-0.2em}-\hspace{-0.2em} \frac{1}{\gamma_{\mathrm{f}}} \hspace{-0.2em}-\hspace{-0.2em} \frac{1}{4\gamma_{\mathrm{f}}^2}\right) \hspace{-0.2em}-\hspace{-0.2em} \log\left(1 \hspace{-0.2em}-\hspace{-0.2em} \frac{\alpha}{\gamma_{\mathrm{f}}} \hspace{-0.2em}-\hspace{-0.2em} \frac{\alpha^2}{4\gamma_{\mathrm{f}}^2}\hspace{-0.2em}\right)\hspace{-0.2em}\right) \hspace{-0.2em}\coloneq\hspace{-0.2em} \overline{\phi}(\alpha; \gamma_{\mathrm{f}})
\end{align}
for all $1 < \alpha < \nicefrac{4\gamma_{\mathrm{f}}}{5}$ and $\gamma_{\mathrm{f}} > \nicefrac{5}{4}$ and where $\gamma_{\mathrm{f}} \coloneq \frac{\lambda^{Z}_{\min} + \sigma^2}{\textsc{C}^2_{S_{\mathrm{f}}X}}$. 

Our goal now is to simplify $\overline{\phi}(\alpha; \gamma_{\mathrm{f}})$. As we will show, it admits a simplified expression which allows us to assess its degradation relative to the guarantees of the Gaussian sketch (given in \propositionref{prop:Privacy_First}). To that end, note that by \citet[Lemma.~1]{lev2025gaussianmix} 
\begin{align}
    \label{eq:upper_bound_gaussmix}
    \underset{X'\simeq X}{\max} \KLDA{\pM_{\brI_n}(X)}{\pM_{\brI_n}(X')} \leq \frac{1}{2(\alpha - 1)}\cdot \left(\alpha \cdot \log\left(1 - \frac{1}{\gamma}\right) - \log\left(1 - \frac{\alpha}{\gamma}\right)\right)
\end{align}
for $\gamma > 1$ and $1 < \alpha < \gamma$ and where $\gamma$ lower bounds $\lambda^{X}_{\min} + \sigma^2$. Additionally, by \citet[Corollary.~1]{lev2025gaussianmix}, for all $1 < \alpha < \nicefrac{2\gamma}{5}$ it holds that 
\begin{equation}
    \label{eq:tcdp_bound_gaussmix}
    \frac{1}{2(\alpha - 1)}\cdot \left(\alpha \cdot \log\left(1 - \frac{1}{\gamma}\right) - \log\left(1 - \frac{\alpha}{\gamma}\right)\right) \leq \frac{\alpha}{2\gamma^2}.
\end{equation}

\subsection{Upper Bound on $\overline{\phi}$}
We note that in the range $1 < \alpha < \nicefrac{4\gamma_{\mathrm{f}}}{5}$ and $\gamma_{\mathrm{f}} > \nicefrac{5}{4}$ it holds that   
\begin{equation}
    \label{eq:tcdp_bound_fastmix}
    \overline{\phi}(\alpha; \gamma_{\mathrm{f}}) \leq \frac{1}{2(\alpha - 1)}\cdot \left(\alpha \cdot \log\left(1 - \frac{5}{4\gamma_{\mathrm{f}}}\right) - \log\left(1 - \frac{5\alpha}{4\gamma_{\mathrm{f}}}\right)\right).
\end{equation}
Then, using \eqref{eq:tcdp_bound_gaussmix} we get that 
\begin{align}
    \overline{\phi}(\alpha; \gamma_{\mathrm{f}}) \leq \frac{25\alpha}{32\gamma^2_{\mathrm{f}}}
\end{align}
for all $1 < \alpha < \nicefrac{8\gamma_{\mathrm{f}}}{25}$.
Since $\gamma_{\mathrm{f}}$ lower bounds $\frac{\lambda^{Z}_{\min} + \sigma^2}{\textsc{C}^2_{S_{\mathrm{f}}X}}$, comparing \eqref{eq:upper_bound_gaussmix} and \eqref{eq:tcdp_bound_fastmix} suggests that our new guarantees are degraded by the factor $\frac{5}{4}\cdot \textsc{C}^2_{S_{\mathrm{f}}X}$ relative to \propositionref{prop:Privacy_First}, and moreover suffers from being dependent on the minimal eigenvalue $\lambda^{Z}_{\min}$ rather than on $\lambda^{X}_{\min}$. Since for many choices of $S_{\mathrm{f}}$ $\big($for example, whenever $S_{\mathrm{f}}$ is an SRHT with $k\chi^2 = \Omega\left(\mathrm{rank}(X)\cdot \log^4(n)\right)$$\big)$ it holds that $\lambda^{Z}_{\min} \geq (1-\chi)\lambda^{X}_{\min}$ for $\chi\in (0, 1]$, this adds at most multiplicative degradation of $1-\chi$.  

\subsection{Upper Bound on $\overline{\phi}$ Whenever $\textsc{C}^2_{S_{\mathrm{f}}X} \to \infty$}
In the absence of any meaningful bound on $\textsc{C}^2_{S_{\mathrm{f}}X}$, (namely, assuming that $\textsc{C}^2_{S_{\mathrm{f}}X} \to \infty$), by the analysis from \appendixref{app:guarantee_CFX_infinity} the overall bound on the divergence becomes $\widehat{\phi}(\alpha; \gamma_{\mathrm{f}, \infty})$, where 
\begin{align}
    \widehat{\phi}(\alpha; \gamma_{\mathrm{f}, \infty}) &= \frac{1}{2(\alpha - 1)}\cdot \left(2\alpha \cdot \log\left(1 - \frac{1}{\sqrt{\gamma_{\mathrm{f}, \infty}}}\right) - \log\left(1 - \frac{2\alpha}{\sqrt{\gamma_{\mathrm{f}, \infty}}} + \frac{\alpha}{\gamma_{\mathrm{f}, \infty}}\right)\right)\\
    &= \frac{1}{2(\alpha - 1)}\cdot \left(\alpha \cdot \log\left(1 - \left(\frac{2}{\sqrt{\gamma_{\mathrm{f}, \infty}}} - \frac{1}{\gamma_{\mathrm{f}, \infty}}\right)\right) - \log\left(1 - \alpha \left(\frac{2}{\sqrt{\gamma_{\mathrm{f}, \infty}}} - \frac{1}{\gamma_{\mathrm{f}, \infty}}\right)\right)\right)
\end{align}
for all $1 < \alpha < \frac{\sqrt{\gamma_{\mathrm{f}, \infty}}}{2}$ and provided that $\gamma_{\mathrm{f}, \infty} > 1$ where $\gamma_{\mathrm{f}, \infty}$ lower bounds $\lambda^{Z}_{\min} + \sigma^2$. Then, using \eqref{eq:tcdp_bound_gaussmix} we get that this is upper bounded by 
\begin{equation}
    \widehat{\phi}(\alpha; \gamma_{\mathrm{f}, \infty}) \leq \frac{\alpha}{2}\cdot \left(\frac{2}{\sqrt{\gamma_{\mathrm{f}, \infty}}} - \frac{1}{\gamma_{\mathrm{f}, \infty}}\right)^2 = \frac{2\alpha}{\gamma_{\mathrm{f}, \infty}}\cdot \left(1 - \frac{1}{2\sqrt{\gamma_{\mathrm{f}, \infty}}}\right)^2 \leq \frac{2\alpha}{\gamma_{\mathrm{f}, \infty}}
\end{equation}
for all $1 < \alpha < \frac{\sqrt{\gamma_{\mathrm{f}, \infty}}}{5}$, suggesting that this bound decays like $\left(\lambda^{Z}_{\min} + \sigma^2\right)^{-1}$ instead of $\left(\lambda^{Z}_{\min} + \sigma^2\right)^{-2}$ as in \citet{lev2025gaussianmix}. Additionally, since it depends on $\lambda^{Z}_{\min}$, it is similarly degraded by the multiplicative constant $1-\chi$. 

\vspace{1.0em}

\subsection{Closed-Form $(\eps, \delta)$-DP Guarantees}
Incorporating these inside the conversion from \RDP to $(\eps, \delta)$-DP \citep[Lemma.~6]{bun2018composable} and multiplying by $k_1$, which amounts to applying a full Gaussian sketch (rather than a single Gaussian vector as currently analyzed), yields the closed-form upper bounds on the privacy guarantees, which are given by 
\begin{align}
    \eps \leq \begin{cases}
        \frac{2k_1}{\lambda^{Z}_{\min} + \sigma^2} + \sqrt{\frac{8k_1\log(\nicefrac{1}{\delta})}{\lambda^{Z}_{\min} + \sigma^2}}\ &\mathrm{if} \ \log(\nicefrac{1}{\delta}) < \left(\frac{\sqrt{\lambda^{Z}_{\min} + \sigma^2}}{5} - 1\right)^2\cdot\frac{2k_1}{\lambda^{Z}_{\min} + \sigma^2}\\
        \frac{2k_1}{5\sqrt{\lambda^{Z}_{\min} + \sigma^2}} + \frac{5\log(\nicefrac{1}{\delta})}{\sqrt{\lambda^{Z}_{\min} + \sigma^2} - 5}\ &\mathrm{otherwise}
    \end{cases}
\end{align}
if $\textsc{C}^2_{S_{\mathrm{f}}} \to \infty$ and by 
\begin{align}
    \label{eq:tcdp_conversion_finite_C}
    \eps \leq \begin{cases}
        \frac{25k_1}{32\left(\lambda^{Z}_{\min} + \sigma^2\right)^2} + \frac{\sqrt{50k_1\log(\nicefrac{1}{\delta})}}{4\cdot \left(\lambda^{Z}_{\min} + \sigma^2\right)}\ &\mathrm{if} \ \log(\nicefrac{1}{\delta}) < \left(\frac{8\left(\lambda^{Z}_{\min} + \sigma^2\right)}{25} - 1\right)^2\cdot \frac{25k_1}{32\left(\lambda^{Z}_{\min} + \sigma^2\right)^2}\\
        \frac{k_1}{4\cdot \left(\lambda^{Z}_{\min} + \sigma^2\right)} + \frac{\log(\nicefrac{1}{\delta})}{\frac{8}{25}\cdot \left(\lambda^{Z}_{\min} + \sigma^2\right) - 1}\ &\mathrm{otherwise}
    \end{cases}    
\end{align}
if $\textsc{C}^2_{S_{\mathrm{f}}} < \infty$. 
\section{Proof of \theoremref{thm:main_fihm}}
\label{app:proof_fihm}
Our proof follows similarly to the analysis of the iterative Hessian mixing, presented in \citet{lev2026near}. We first provide auxiliary lemmas that will assist us in establishing the accuracy guarantees. 


\subsection{Auxiliary Definitions and Lemmas}
We first provide the RDP guarantees of the Gaussian mechanism, as established by \citet{mironov2017renyi}. 
\begin{lem}[{\citep{mironov2017renyi}}]
    \label{lem:RDP_Mironov}
    Let \(f : \pazocal{X} \to \reals^{m}\) be a map.
    If there exists \(b > 0\) such for any \(X, X' \in \pazocal{X}\) satisfying \(X \simeq X'\) it holds that \(\|f(X) - f(X')\| \leq b\),
    then the Gaussian mechanism \(X \mapsto f(X) + \sigma \xi\) for \(\xi \sim \pN(\vec{0}_{m}, \brI_{m})\) satisfies \(\left(\alpha, \frac{\alpha b^{2}}{2\sigma^{2}}\right)\) for all \(\alpha > 1\).
\end{lem}

Next, we prove auxiliary lemmas for the overall accuracy guarantee. The first establishes guarantees for a sketch formed by concatenating Gaussian and fast sketches, as in \FastMix. We later use it to generalize the accuracy results established by \citet{lev2026near}. The second uses \citep[Theorem.~1]{cohen2016optimal_amm} for bounding $\widehat{\Delta}$. The third shows that the SRHT sketch preserves $\lambda^{X}_{\min}$ with high probability.
\begin{lem}
\label{lem:probability_pilanci_concatenated}
    Let $\mathsf{S}_{\mathrm{tot}}$ defined via $\mathsf{S}_{\mathrm{tot}} = \mathsf{S}_{\mathrm{G}}\widehat{\mathsf{S}}_{\mathrm{H}}$ for $\mathsf{S}_{\mathrm{G}}\sim \pN(0, \brI_{k_1\times (k_2 + d)})$ and $\widehat{\mathsf{S}}_{\mathrm{H}} \coloneq \begin{pmatrix}
        \mathsf{S}_{\mathrm{H}} & 0_{k_2\times d} \\ 0_{d\times n} & \brI_d
    \end{pmatrix}$ where $\mathsf{S}_{\mathrm{H}} \sim \mathrm{SRHT}(k_2, n)$. Then, there exists universal constants $(c_0, c_1, c_2, c_3)$ such that for any $A \in \reals^{(n+d)\times d}$ of the form $A = \Big[B^{\top}, b\brI_d\Big]^{\top}$ for some $B \in \reals^{n\times d}$ and $b\in \reals$, and for any $\chi \in (0, 1]$ and $k_1, k_2\in \nats$ satisfying $k_1\chi^2 \geq c_0d,\  k_2\chi^2 \geq c_0\cdot \mathrm{rank}\left(B\right)\cdot \log^4(n)$ it holds 
    \begin{equation}
        \mathbb{P}\left(\mathcal{C}\left(A, \chi, k_1\right) \right) \geq 1 - c_1 \cdot \left(\exp\left\{-c_2k_1\chi^2\right\} + \exp\left\{-\frac{c_3k_2\chi^2}{\log^4(n)}\right\}\right)
    \end{equation}
    where 
    \begin{equation}
        \mathcal{C}(A, \chi, k_1) \coloneq \left\{\sup_{u \in \reals^{d}, v \in \reals^{d}} \frac{\left|u^{\top}A^{\top}\left(\frac{1}{k_1}\mathsf{S}_{\mathrm{tot}}^{\top}\mathsf{S}_{\mathrm{tot}} - \brI_{n+d}\right) A v\right|}{\|A u\| \cdot \|A v\|} \leq \frac{\chi}{2} \ \bigcap\ \inf_{v \in \reals^{d}} \frac{\norm{\mathsf{S}_{\mathrm{tot}}A v}^2}{k_1 \cdot \|Av\|^{2}} \geq 1-\chi\right\}.
    \end{equation}
\end{lem}

\begin{proof}
Let 
\begin{equation}
    \label{eq:good_events_S_R}
    E_{\mathrm{H}}(\eta) \coloneq \left\{\sup_{u \in \reals^{n}, v \in \reals^{d}} \frac{\left|u^{\top}\left(\mathsf{S}^{\top}_{\mathrm{H}}\mathsf{S}_{\mathrm{H}} - \brI_{n}\right) B v\right|}{\norm{u} \cdot \norm{B v}} \leq \frac{\eta}{2}, \ \ \ \ \ \underset{v\in \reals^{d}}{\inf} \ \frac{\norm{\mathsf{S}_{\mathrm{H}}Bv}^2}{\norm{Bv}^2} \geq 1 - \eta\right\}
\end{equation}
for $\eta \in (0, 1]$. By \citet[Lemma.~1]{pilanci_hessiansketch} applied as in \citet[Corollary.~2]{pilanci2015randomized}, $\mathbb{P}\left(E_{\mathrm{H}}(\eta)\right) \geq 1 - c_{1,\mathrm{H}}\cdot \exp\left\{-c_{2,\mathrm{H}}k_2\eta^2\right\}$ provided that $k_2\eta^2 \geq c_{0, \mathrm{H}}\cdot \mathrm{rank}\left(B\right)\cdot \log^4(n)$ for some universal constants $c_{0, \mathrm{H}}, c_{1,\mathrm{H}}, c_{2,\mathrm{H}}$.\footnote{By our convention, sketches in the set $\mathrm{SRHT}(k,n)$ contain additional normalization by factor $\sqrt{k}$ which was not presented in \citet{pilanci_hessiansketch}. Thus, we avoid the factor $\nicefrac{1}{k_2}$ while considering the events in \eqref{eq:good_events_S_R}.} In particular, this holds for the $k_2$ from the lemma statement and with $\eta \gets \chi$ whenever $c_0 \geq c_{0, \mathrm{H}}$.  

Then, we note that conditioned on $E_{\mathrm{H}}(\eta)$, for every $v\in \reals^{d}$ it holds that 
\begin{equation}
    \norm{\mathsf{S}_{\mathrm{H}}B v}^2 + b^2\norm{v}^2 \geq (1-\chi)\norm{B v}^2 + b^2\norm{v}^2 \geq (1-\chi)\left(\norm{B v}^2 + b^2\norm{v}^2\right),
\end{equation}
and which then implies
\begin{equation}
    \label{eq:event_aug_SRHT_sketch_A2}
    \inf_{v \in \reals^{d}} \frac{\norm{\mathsf{S}_{\mathrm{H}}B v}^2 + b^2 \norm{v}^2}{\norm{Bv}^{2} + b^2 \norm{v}^2} \geq 1-\chi.
\end{equation}
Additionally, we note that 
\begin{equation}
    \label{eq:event_aug_SRHT_sketch_A1}
    \sup_{u_1 \in \reals^{n}, u_2 \in \reals^{d}, v \in \reals^{d}} \hspace{-0.15em}\frac{\left|\left(u^{\top}_1, u^{\top}_2\right)\hspace{-0.15em}\left(\hspace{-0.15em}\begin{pmatrix}
        \mathsf{S}_{\mathrm{H}}^{\top}\mathsf{S}_{\mathrm{H}} & 0_{n\times d} \\ 0_{d\times d} & \brI_d
    \end{pmatrix} \hspace{-0.15em}-\hspace{-0.15em} \brI_{n+d}\hspace{-0.15em}\right) \begin{pmatrix}
        B \\ b\brI_d
    \end{pmatrix} v\right|}{\sqrt{\norm{u_1}^2 + \norm{u_2}^2} \cdot \sqrt{\norm{B v}^2 + b^2 \norm{v}^2}} \hspace{-0.15em}\leq\hspace{-0.15em} \sup_{u_1 \in \reals^{n},v \in \reals^{d}}\hspace{-0.15em} \frac{\left|u^{\top}_1\hspace{-0.15em}\left(\mathsf{S}_{\mathrm{H}}^{\top}\mathsf{S}_{\mathrm{H}} \hspace{-0.15em}-\hspace{-0.15em} \brI_{n}\right)\hspace{-0.15em}B v\right|}{\norm{u_1} \cdot \norm{B v}} \hspace{-0.15em}\leq\hspace{-0.15em} \frac{\chi}{2}.
\end{equation}

Now, we make the notation $\mathsf{C} \coloneq \widehat{\mathsf{S}}_{\mathrm{H}}A \in \reals^{k_2\times d}$ and 
\begin{equation}
    \label{eq:A_1_condition_S_G}
    E_{\mathrm{G}}(\eta; M) \coloneq \left\{\sup_{u \in \reals^{k_2+d}, v \in \reals^{d}} \frac{\left|u^{\top}\left(\frac{1}{k_1}\mathsf{S}^{\top}_{\mathrm{G}}\mathsf{S}_{\mathrm{G}} - \brI_{k_2+d}\right) M v\right|}{\norm{u} \cdot \norm{M v}} \leq \frac{\eta}{2}, \ \ \underset{v\in \reals^{d}}{\inf} \ \frac{\norm{\mathsf{S}_{\mathrm{G}}Mv}^2}{k_1 \cdot \norm{Mv}^2} \geq 1 - \eta\right\}\hspace{-0.25em}.
\end{equation}
Using \citet[Lemma.~1]{pilanci_hessiansketch}, we know that for every realization $\mathsf{C} = C$ it holds that $\mathbb{P}\left(E_{\mathrm{G}}(\eta; C)\right) \geq 1 - c_{1, \mathrm{G}} \cdot \exp\left\{-c_{2, \mathrm{G}} k_1 \eta^2\right\}$ provided that $k_1 \eta^2 \geq c_{0, \mathrm{G}} \cdot \mathrm{rank}(C)$, and which since $\mathrm{rank}(C) \leq d$ holds by the constrains we have on $k_1$ in the lemma statement for $\eta \gets \chi$. Then, by using the law of total probability, we note that under the same constraint on $k_1$ it further holds that 
\begin{equation}
    \mathbb{P}\left(E_{\mathrm{G}}(\eta; \mathsf{C})\right) \geq 1 - c_{1, \mathrm{G}} \cdot \exp\left\{-c_{2, \mathrm{G}} k_1 \eta^2\right\}
\end{equation}

Thus, since 
\begin{equation}
    \inf_{v \in \reals^{d}} \frac{\norm{\mathsf{S}_{\mathrm{G}}\mathsf{C} v}^2}{k_1 \cdot \norm{Av}^{2}} = \inf_{v \in \reals^{d}} \left\{\frac{\norm{\mathsf{S}_{\mathrm{G}}\mathsf{C} v}^2}{k_1 \cdot \norm{\mathsf{C}v}^{2}}\cdot\frac{\norm{\mathsf{C}v}^{2}}{\norm{A v}^2}\right\},
\end{equation}
we get that 
\begin{equation}
    \inf_{v \in \reals^{d}} \frac{\norm{\mathsf{S}_{\mathrm{G}}\mathsf{C} v}^2}{k_1 \cdot \norm{Av}^{2}} \geq (1-\chi) \cdot \underset{v\in \reals^{d}}{\inf} \ \frac{\norm{\mathsf{C} v}^2}{\norm{Av}^{2}} \geq (1-\chi)^2
\end{equation}
where the first inequality is by \eqref{eq:A_1_condition_S_G} and the second inequality is by \eqref{eq:event_aug_SRHT_sketch_A2}, and both hold since they hold for every realization $\mathsf{C} = C$, thus holds with probability $1$ whenever we look at the event with the random matrix $\mathsf{C}$. 

Note that by picking $M \gets \widehat{\mathsf{S}}_{\mathrm{H}}A$ in \eqref{eq:A_1_condition_S_G} and for the $k_1$ in the lemma statement we have  
\begin{equation}
    \sup_{w \in \reals^{k_2 + d}, v \in \reals^{d}} \frac{\left|w^{\top}\left(\frac{1}{k_1}\mathsf{S}^{\top}_{\mathrm{G}}\mathsf{S}_{\mathrm{G}} - \brI_{k_2 + d}\right) \widehat{\mathsf{S}}_{\mathrm{H}}A v\right|}{\norm{w} \cdot \norm{\widehat{\mathsf{S}}_{\mathrm{H}}A v}} \leq \frac{\chi}{2}.
\end{equation}
Then, by writing 
\begin{equation}
    \label{eq:subspace_mebedding_S_G}
    \frac{1}{k_1}\mathsf{S}_{\mathrm{tot}}^{\top}\mathsf{S}_{\mathrm{tot}} - \brI_{n+d} = \frac{1}{k_1}\widehat{\mathsf{S}}_{\mathrm{H}}^{\top}\mathsf{S}^{\top}_{\mathrm{G}}\mathsf{S}_{\mathrm{G}}\widehat{\mathsf{S}}_{\mathrm{H}} - \brI_{n+d} = \widehat{\mathsf{S}}_{\mathrm{H}}^{\top}\left(\frac{1}{k_1}\mathsf{S}^{\top}_{\mathrm{G}}\mathsf{S}_{\mathrm{G}} - \brI_{k_2+d}\right)\widehat{\mathsf{S}}_{\mathrm{H}} + \left(\widehat{\mathsf{S}}_{\mathrm{H}}^{\top}\widehat{\mathsf{S}}_{\mathrm{H}} - \brI_{n+d}\right)
\end{equation}
and using the independence between $\mathsf{S}_{\mathrm{H}}$ and $\mathsf{S}_{\mathrm{G}}$ we note that for every $w,v \in \reals^{d}$ and $u = Aw$ it holds that
\begin{align}
    \frac{\abs{u^{\top}\left(\frac{1}{k_1}\mathsf{S}_{\mathrm{tot}}^{\top}\mathsf{S}_{\mathrm{tot}} - \brI_{n+d}\right)Av}}{\norm{u}\norm{Av}} &= \frac{\abs{\left(\widehat{\mathsf{S}}_{\mathrm{H}}u\right)^{\top}\left(\frac{1}{k_1}\mathsf{S}^{\top}_{\mathrm{G}}\mathsf{S}_{\mathrm{G}} - \brI_{k_2+d}\right)\widehat{\mathsf{S}}_{\mathrm{H}}Av + u^{\top}\left(\widehat{\mathsf{S}}_{\mathrm{H}}^{\top}\widehat{\mathsf{S}}_{\mathrm{H}} - \brI_{n+d}\right)Av}}{\norm{u}\norm{Av}}\\
    &\leq \frac{\abs{\left(\widehat{\mathsf{S}}_{\mathrm{H}}u\right)^{\top}\left(\frac{1}{k_1}\mathsf{S}^{\top}_{\mathrm{G}}\mathsf{S}_{\mathrm{G}} - \brI_{k_2+d}\right)\widehat{\mathsf{S}}_{\mathrm{H}}Av} + \abs{u^{\top}\left(\widehat{\mathsf{S}}_{\mathrm{H}}^{\top}\widehat{\mathsf{S}}_{\mathrm{H}} - \brI_{n+d}\right)Av}}{\norm{u}\norm{Av}}\\
    &\leq \frac{\frac{\chi}{2}\cdot \norm{\widehat{\mathsf{S}}_{\mathrm{H}}u}\cdot\norm{\widehat{\mathsf{S}}_{\mathrm{H}}Av} + \frac{\chi}{2}\cdot \norm{u}\cdot \norm{Av}}{\norm{u}\norm{Av}}\\
    &= \frac{\frac{\chi}{2}\cdot \norm{\widehat{\mathsf{S}}_{\mathrm{H}}Aw}\cdot\norm{\widehat{\mathsf{S}}_{\mathrm{H}}Av} + \frac{\chi}{2}\cdot \norm{Aw}\cdot \norm{Av}}{\norm{Aw}\norm{Av}} \\
    &\leq \frac{1}{2}\left(\chi + \chi\cdot \left(1 + \frac{\chi}{2}\right)\right)\\
    &\leq \frac{5}{4}\chi
\end{align}
where we have used \eqref{eq:subspace_mebedding_S_G} and \eqref{eq:event_aug_SRHT_sketch_A1} and the inequalities 
\begin{equation}
    \norm{\widehat{\mathsf{S}}_{\mathrm{H}}Av}^2 \leq \left(1 + \frac{\chi}{2}\right)\norm{Av}^2, \ \ \ \norm{\widehat{\mathsf{S}}_{\mathrm{H}}Aw}^2 \leq \left(1 + \frac{\chi}{2}\right)\norm{Aw}^2  
\end{equation}
which due to \eqref{eq:event_aug_SRHT_sketch_A1}. Since, as discussed, the conditional bound for $E_{\mathrm{G}}(\eta; \mathsf{C})$ holds uniformly for every $\mathsf{C} = C$ whenever $k_1\eta^2_1 \geq c_{0, \mathrm{G}}d$, it holds that 
\begin{equation}
    \mathbb{P}\left(E_{\mathrm{G}}(\eta_1; C) \mid \mathsf{C} = C\right)
    \ge 1 - c_{1, \mathrm{G}}\cdot \exp\left\{-c_{2, \mathrm{G}}k_1\eta^2_1\right\},
\end{equation}
and therefore,
\begin{align}
    \mathbb{P}\left(E_{\mathrm{H}}(\eta_2) \cap E_{\mathrm{G}}(\eta_1; \mathsf{C})\right)
    &= \E{\mathbbm{1}_{E_{\mathrm{H}}(\eta_2)}\cdot \mathbb{P}\left(E_{\mathrm{G}}(\eta_1; \mathsf{C}) \mid \mathsf{C}\right)} \\
    &\ge \left(1 - c_{1, \mathrm{G}}\cdot \exp\left\{-c_{2, \mathrm{G}}k_1\eta^2_1\right\}\right)\cdot \mathbb{P}\left(E_{\mathrm{H}}(\eta_2)\right) \\
    &\ge 1 - c_{1, \mathrm{G}}\cdot \exp\left\{-c_{2, \mathrm{G}}k_1\eta^2_1\right\}
       - c_{1, \mathrm{F}}\cdot \exp\left\{-\frac{c_{2, \mathrm{F}}k_2\eta^2_2}{\log^4(n)}\right\}
\end{align}
where $\mathbbm{1}_{E_{\mathrm{H}}(\eta)}$ denotes an indicator for the event $E_{\mathrm{H}}(\eta)$. The proof is finalized by adjusting constants and noting that this bound holds with $\eta_1\hspace{-0.15em}\gets\hspace{-0.15em} \chi, \eta_2\hspace{-0.15em}\gets\hspace{-0.15em}\chi$ for the $k_1, k_2$ used in the lemma statement. 
\end{proof}

The proof also implicitly justifies the choice of the parameters $k_1$ and $k_2$ in the lemma statement. Since the overall approximation budget $\chi$ is shared by the sketches $\mathsf{S}_{\mathrm{H}}$ and $\mathsf{S}_{\mathrm{G}}$, a natural choice is to split it evenly, setting the approximation parameter of each sketch to $\nicefrac{\chi}{2}$, as in our calibration of $k_1$ and $k_2$. Additionally, since $k_2$ relates to an SRHT sketch, it should be inflated by a factor $\log^4(n)$ to account for a similar factor in the constraints, and furthermore, since it is applied directly to $B$, it can be scaled with $\mathrm{rank}(B)$ rather than $d$. 
\begin{lem}
\label{lem:nelson_prob}
Let $\mathsf{S}_{\mathrm{H}}\sim \mathrm{SRHT}(k, n)$ and let $X \in \mathbb{R}^{n\times d}$ with $n\geq d$. Then, there exist three universal constants $c_0,c_1,c_2>0$ such that for any $\chi,\varrho \in (0,1]$ satisfying
\begin{equation}
    k\chi^2 \geq c_0\cdot \left(\mathrm{rank}(X) + \log\left(\frac{n c_1}{\chi\varrho}\right)\log\left(\frac{c_2 n\left(\mathrm{rank}(X)+1\right)}{\varrho}\right)\right),
\end{equation}
then w.p. at least $1-\varrho$,
\begin{equation}
    \max_{i\in[n]} \norm{\left(\mathsf{S}_{\mathrm{H}}X\right)^{\top}\mathsf{S}^{\top}_{\mathrm{H}}e_i - x_i} \leq \chi \norm{X}_{\mathrm{op}}.
\end{equation}
\end{lem}

\begin{proof}
Let $X = U\Sigma V^\top$ be a compact SVD of $X$, where
$U\in\mathbb{R}^{n\times r}$ has orthonormal columns with $r\coloneq \mathrm{rank}(X)$ and note that $\norm{X}_{\mathrm{op}} = \norm{V\Sigma}_{\mathrm{op}}$. Then, for every $i\in[n]$
\begin{equation}
    \norm{\left(\mathsf{S}_{\mathrm{H}}X\right)^{\top}\hspace{-0.15em}\mathsf{S}^{\top}_{\mathrm{H}}e_i \hspace{-0.15em}-\hspace{-0.15em} x_i} \hspace{-0.15em}=\hspace{-0.15em} \norm{X^{\top}\hspace{-0.15em}\left(\mathsf{S}_{\mathrm{H}}^{\top} \mathsf{S}_{\mathrm{H}}\hspace{-0.15em}-\hspace{-0.15em}\brI_n\hspace{-0.15em}\right)\hspace{-0.15em}e_i} \hspace{-0.15em}=\hspace{-0.15em} \norm{V\Sigma U^{\top}\hspace{-0.15em}\left(\mathsf{S}_{\mathrm{H}}^\top \mathsf{S}_{\mathrm{H}}\hspace{-0.15em}-\hspace{-0.15em}\brI_n\hspace{-0.15em}\right)\hspace{-0.15em}e_i} \hspace{-0.15em}\leq\hspace{-0.15em} \norm{X}_{\mathrm{op}}\hspace{-0.15em}\norm{U^{\top}\hspace{-0.15em}\left(\mathsf{S}_{\mathrm{H}}^\top \mathsf{S}_{\mathrm{H}}\hspace{-0.15em}-\hspace{-0.15em}\brI_n\hspace{-0.15em}\right)\hspace{-0.15em}e_i}.
\end{equation}
Thus, it is enough to bound $\norm{U^\top\left(\mathsf{S}_{\mathrm{H}}^\top \mathsf{S}_{\mathrm{H}}-\brI_n\right)e_i}$. Fix $i\in[n]$ and let $Q_i\in\mathbb{R}^{n\times t_i}$ have orthonormal columns spanning $\mathrm{span}(U,e_i)$ and where $t_i \le r+1$. Since both $U$ and $e_i$ belong to $\mathrm{span}(Q_i)$, there exist $R_i$ and $q_i$ such that
\begin{equation}
    U = Q_iR_i, \qquad e_i = Q_i q_i, \qquad \norm{R_i}_{\mathrm{op}}\le 1, \qquad \norm{q_i}_2\le 1.
\end{equation}
Therefore,
\begin{equation}
    \norm{U^\top\left(\mathsf{S}_{\mathrm{H}}^\top \mathsf{S}_{\mathrm{H}}-\brI_n\right)e_i} = \norm{R_i^\top Q_i^\top\left(\mathsf{S}_{\mathrm{H}}^\top \mathsf{S}_{\mathrm{H}}-\brI_n\right)Q_i q_i} \le \norm{Q_i^\top\left(\mathsf{S}_{\mathrm{H}}^\top \mathsf{S}_{\mathrm{H}}-\brI_n\right)Q_i}_{\mathrm{op}}.
\end{equation}
Hence
\begin{equation}
\norm{X^\top\left(\mathsf{S}_{\mathrm{H}}^\top \mathsf{S}_{\mathrm{H}}-\brI_n\right)e_i}
\le
\norm{X}_{\mathrm{op}}\,
\norm{Q_i^\top\left(\mathsf{S}_{\mathrm{H}}^\top \mathsf{S}_{\mathrm{H}}-\brI_n\right)Q_i}_{\mathrm{op}}.
\end{equation}

Then, for each fixed $i$, by \citet[Theorem.~9]{cohen2016optimal_amm} it holds that 
\begin{equation}
    \E{\norm{Q_i^\top\left(\mathsf{S}_{\mathrm{H}}^\top \mathsf{S}_{\mathrm{H}}-\brI_n\right)Q_i}_{\mathrm{op}}^{2}} \le \chi^2 \delta_i,
\end{equation}
provided
\begin{equation}
    k \geq \frac{C}{\chi^{2}} \left( (r+1)+\log\left(\frac{1}{\chi\delta_i}\right)\log\left(\frac{r+1}{\delta_i}\right)\right)
\end{equation}
for a universal constant $C>0$. By Markov's inequality, this further implies that 
\begin{equation}
    \mathbb{P}\left(\norm{Q_i^\top\left(\mathsf{S}_{\mathrm{H}}^\top \mathsf{S}_{\mathrm{H}}-\brI_n\right)Q_i}_{\mathrm{op}}>\chi\right) \leq \delta_i.
\end{equation}
Consequently,
\begin{equation}
    \mathbb{P}\left(\norm{X^\top\left(\mathsf{S}_{\mathrm{H}}^\top \mathsf{S}_{\mathrm{H}}-\brI_n\right)e_i} > \chi \norm{X}_{\mathrm{op}}\right) \leq \delta_i.
\end{equation}

Then, choosing $\delta_i \coloneq \nicefrac{\varrho}{n}$ and union bounding over $i\in[n]$ yields
\begin{equation}
    \mathbb{P}\left( \max_{i\in[n]} \norm{X^\top\left(\mathsf{S}_{\mathrm{H}}^\top \mathsf{S}_{\mathrm{H}}-\brI_n\right)e_i} > \chi \norm{X}_{\mathrm{op}} \right) \leq \sum_{i=1}^n \delta_i = \varrho.
\end{equation}
This completes the proof. 
\end{proof}

\begin{lem}
    \label{lem:SRHT_preserves_lambda_min} 
    Let $X \in \reals^{n\times d}$ and let $\mathsf{S}_{\mathrm{H}} \in \reals^{k\times n}$ be \normalfont{SRHT} sketch (\definitionref{defn:SRHT}). There exists universal constants $c_0, c_1, c_2$ such that if $k \geq \frac{c_0\cdot\mathrm{rank}(X)}{\chi^2}\cdot \log^4(n)$ then w.p. at least $1 - c_1 \cdot \exp\left\{-\frac{c_2 k\chi^2}{\log^4(n)}\right\}$  
    \begin{equation}
        \left(1 - \chi\right)\cdot \lambda^{X}_{\min} \leq \lambda^{\mathsf{S}_{\mathrm{H}}X}_{\min} \leq \left(1 + \chi\right)\cdot \lambda^{X}_{\min},\ \ \ \left(1 - \chi\right)\cdot \lambda^{X}_{\max} \leq \lambda^{\mathsf{S}_{\mathrm{H}}X}_{\max} \leq \left(1 + \chi\right)\cdot \lambda^{X}_{\max}.
    \end{equation}
\end{lem}

\begin{proof}
    The proof follows by \citet[Lemma.~1]{pilanci_hessiansketch}, which implies that for $k \geq \frac{c_0\cdot \mathrm{rank}(X)}{\chi^2}\cdot \log^4(n)$ w.p. at least $1 - c_1 \cdot \exp\left\{-\frac{c_2 k\chi^2}{\log^4(n)}\right\}$ 
    \begin{equation}
        \label{eq:sketching_guarantee_step1}
        \underset{v\in \reals^{d}: \ \norm{v} = 1}{\sup} \ \abs{X^{\top}v^{\top}\left(\mathsf{S}^{\top}_{\mathrm{H}}\mathsf{S}_{\mathrm{H}} - \brI_n\right)Xv} \leq \chi\norm{Xv}^2.
    \end{equation}
    We note that \eqref{eq:sketching_guarantee_step1} implies 
    \begin{equation}
        \left(1 - \chi\right)\norm{Xv}^2 \leq \norm{\mathsf{S}_{\mathrm{H}}Xv}^2 \leq \left(1 + \chi\right)\norm{Xv}^2, \ \ \ \forall v\in\reals^{d}.
    \end{equation}
    We now restrict attention to vectors $v$ that have unit norm. Then, since by definition $\norm{Xv}^2 \geq \lambda^{X}_{\min}$, 
    \begin{equation}
        \left(1 - \chi\right)\lambda^{X}_{\min} \leq \norm{\mathsf{S}_{\mathrm{H}}Xv}^2 \leq \left(1 + \chi\right)\norm{Xv}^2.
    \end{equation}
    Setting $v\gets v^{\star}$, where $v^{\star}$ corresponds to the minimal singular value of $X$, yields
    \begin{equation}
        \left(1 - \chi\right)\lambda^{X}_{\min} \leq \norm{\mathsf{S}_{\mathrm{H}}Xv^{\star}}^2 \leq \left(1 + \chi\right)\lambda^{X}_{\min}.
    \end{equation}
    This implies 
    \begin{equation}
        \left(1 - \chi\right)\lambda^{X}_{\min} \leq \lambda^{\mathsf{S}_{\mathrm{H}}X}_{\min} \leq \left(1 + \chi\right)\lambda^{X}_{\min}
    \end{equation}
    since $\norm{\mathsf{S}_{\mathrm{H}}Xv^{\star}}^2 \geq \lambda^{\mathsf{S}_{\mathrm{H}}X}_{\min}$ and since we could instead pick $v$ to be the singular vector corresponding to the minimum singular value of $\mathsf{S}_{\mathrm{H}}X$. The corresponding derivation for the maximum eigenvalue proceeds in the same manner. The proof then follows after adjusting $\chi$ and the relevant constants. 
\end{proof}
Finally, the next lemma provides high probability bounds for the norm of a Gaussian vector 
\begin{lem}
\label{lem:excess_prob_Gaussian}
    Let $\xi \sim \pN(\vec{0}_d, \brI_d)$ and let $\varrho \in (0, 1]$. Then, 
    \begin{equation}
        \mathbb{P}\left(\norm{\xi}^2 \leq \left(\sqrt{2d} + \sqrt{4\log(\nicefrac{1}{\varrho})}\right)^2\right) \geq 1 - \varrho^2. 
    \end{equation}
\end{lem}

The previous lemmas and definitions allow us to derive high-probability accuracy guarantees for \FastIHM. In particular, \lemmaref{lem:probability_pilanci_concatenated} parallels a similar lemma proved in \citet{lev2026near} and enables a recursive characterization. In addition, \lemmaref{lem:nelson_prob} and \lemmaref{lem:SRHT_preserves_lambda_min} control the degradation induced by applying the Gaussian sketch to the compressed datasets $\left\{\mathsf{S}_{\mathrm{H}, t}X\right\}_t$, as done in \FastMix. Using the notation of \lemmaref{lem:privacy_lemma}, they control $\lambda^{\mathsf{S}_{\mathrm{H}, t}X}_{\min}$ and $\textsc{C}^2_{\mathsf{S}_{\mathrm{H}, t}X}$ for all $t=0,1,\ldots,T-1$. 

\subsection{\textbf{Proof of \theoremref{thm:main_fihm}}}
\label{app:prf:thm:privacy_utility_fihm}
We first prove the DP guarantee for the algorithm, and then the utility guarantee. Throughout, given $(k_1, T, \textsc{C})$ and for a target failure probability $\varrho$ we make the next choice of noise parameters 
\begin{align}
    \label{eq:parameters_choice_fast_ihm}
    \sigma &= \frac{3\sqrt{T\log\left(\nicefrac{3}{\delta}\right)}\textsc{C}}{\sqrt{2}\eps}\left(1+\sqrt{1+\frac{\eps}{3 \log\left(\nicefrac{3}{\delta}\right)}}\right), \ \tau = \max\left\{\log\left(\frac{3T}{\delta}\right), \log\left(\frac{16T}{\varrho}\right)\right\}, \ \omega = \frac{6T}{\eps}\ \ \ \ \ \\
    \gamma_1 &\hspace{-0.2em}=\hspace{-0.2em} \frac{15\sqrt{2k_1T\log\left(\nicefrac{3}{\delta}\right)}}{8\varepsilon}\hspace{-0.2em}\left(\hspace{-0.2em}1\hspace{-0.2em}+\hspace{-0.2em}\sqrt{1\hspace{-0.2em}+\hspace{-0.2em}\frac{\varepsilon}{3 \log\left(\nicefrac{3}{\delta}\right)}}\hspace{-0.15em}\right), \gamma_2 \hspace{-0.2em}=\hspace{-0.2em} \frac{2k_1T \hspace{-0.2em}+\hspace{-0.2em} 25\log\left(\nicefrac{3}{\delta}\right)}{8\eps}
\end{align}
and 
\begin{align}
    \gamma \hspace{-0.2em}=\hspace{-0.2em} \begin{cases}
        \gamma_1 \ &\mathrm{if} \ \log\left(\nicefrac{3}{\delta}\right) \leq \left(\frac{8\gamma_1}{25} \hspace{-0.2em}-\hspace{-0.2em} 1\right)^2\hspace{-0.2em}\cdot\hspace{-0.2em} \frac{25k_1T}{32\gamma^2_1}\\
        \gamma_2 \ &\mathrm{otherwise}
    \end{cases}.
\end{align}

\subsubsection{Differential Privacy: }
\label{app:proof_thm_linear_DP}
\algorithmref{alg:private_fast_hessian_sketch} involves $2T$ calls to private mechanisms: $T$ calls to the Gaussian mechanism for constructing $\widetilde{G}_t$ and another $T$ calls for forming $(\widetilde{X}_t, \widetilde{\eta}_t)$ (splitted between the function \texttt{MultipleFastMixing} and the fast sketching step in Line.~\ref{line:fast-sketch-X}). Inside the function \texttt{MultipleFastMixing}, $\widetilde{m}_t$ and $\widetilde{\lambda}_t$ are computed via the Laplace mechanism. Moreover, given these private values, the RDP guarantees for the steps computing $\widetilde{X}_t$ are known due to \lemmaref{lem:privacy_lemma}. Throughout, quantities computed in the $t$'th call to \FastMix are indexed by $t$ $\Big($for example, $\widetilde{m}_t, \widetilde{\lambda}_t, \mathsf{S}_{\mathrm{H}, t}, \mathsf{S}_{\mathrm{G}, t}$$\Big)$.

Our proof uses \lemmaref{lem:RDP_Mironov} and \lemmaref{lem:privacy_lemma} to calculate the RDP guarantees for each step forming $\widetilde{G}_t$ and $\widetilde{X}_t$, and then applies the RDP composition theorem to obtain their overall guarantees separately. In particular, we separate the contribution of the Laplace mechanisms from the overall RDP guarantees of the output of \FastMix. We then convert these RDP guarantees to $(\eps,\delta)$-DP, and combine them with the $\eps$-DP guarantee of the steps involving the Laplace mechanism.

\textbf{Gaussian Mechanism.}

Let \((X', Y') \simeq (X, Y)\) i.e., \((X', Y')\) is a zero-out neighbor of \((X, Y)\). Then, since we set $\textsc{C}_X = 1$, by the Cauchy-Schwarz inequality we get that
\begin{equation}
    \label{eq:sensitivity_G_t}
    \norm{X'^{\top}\mathrm{clip}_{\textsc{C}}\left(Y' - X'\widehat{\theta}_{t}\right) - X^{\top}\mathrm{clip}_{\textsc{C}}\left(Y - X\widehat{\theta}_{t}\right)} \leq \textsc{C}
\end{equation}
which holds since the clipping operation happens coordinate-wise, thus $X'^{\top}\mathrm{clip}_{\textsc{C}}\left(Y' - X'\widehat{\theta}_{t}\right) - X^{\top}\mathrm{clip}_{\textsc{C}}\left(Y - X\widehat{\theta}_{t}\right)$ is of the form $\mathrm{clip}_{\textsc{C}}\left(y_i - x^{\top}_i \widehat{\theta}_t\right)\cdot x_i$ for the coordinate $i$ that have been removed and which by the assumption $\norm{x_i} \leq 1$ is upper bounded by $\textsc{C}$. Thus, the $\ell_2$ sensitivity of the quantity $X^{\top}\mathrm{clip}_{\textsc{C}}\left(Y - X\widehat{\theta}_t\right)$ is at most \(\textsc{C}\) irrespective of \(\widehat{\theta}_{t}\). 

Then, \lemmaref{lem:RDP_Mironov} and \eqref{eq:sensitivity_G_t} yield that the RDP guarantee of each step that forms $\widetilde{G}_t$ is $\left(\alpha,\frac{\alpha\textsc{C}^2}{2\sigma^2}\right)$.

\textbf{Generating $\left\{\widetilde{m}_t, \widetilde{\lambda}_t\right\}^{T-1}_{t=0}$.}\ \ 

By the choice of $\tau$, it holds that 
\begin{align}
    \mathbb{P}\left(\widetilde{m}_t \geq \widehat{m}_t\right) = \mathbb{P}\left(\widetilde{\lambda}_t \leq \lambda^{\mathsf{S}_{\mathrm{H}, t}X}_{\min}\right) = 1 - \frac{1}{2}\cdot \exp\left\{-\tau\right\} \geq 1 - \frac{\delta}{6T}, \ \ \forall t=1,2,\ldots, T.    
\end{align}
Since these events are independent, it holds that 
\begin{align}
    \mathbb{P}\left(\mathcal{B}\right) \geq \left(1 - \frac{\delta}{6T}\right)^{2T} \geq 1 - \frac{\delta}{3}
\end{align}
where $\mathcal{B} \coloneq \bigcap^{T-1}_{t=0} \left\{\left\{\widetilde{m}_t \geq \widehat{m}_t\right\}\cap \left\{\widetilde{\lambda}_t \leq \lambda^{\mathsf{S}_{\mathrm{H}, t}X}_{\min}\right\}\right\}$ and the last step is by the inequality $(1-x)^{n} \geq 1-nx$ which holds for $x\geq 0, n\geq 1$ and $nx \leq 1$.

Now, each joint release $\left(\widetilde{m}_t, \widetilde{\lambda}_t\right)$ is similar to the analysis performed in \appendixref{app:DP_guarantee_fastmix}. Thus, for the $(\tau, \omega)$ from \eqref{eq:parameters_choice_fast_ihm}, each step is $\left(\nicefrac{\eps}{3T},\nicefrac{\delta}{6T} \right)$-DP.

\textbf{Fast Sketching.}

We note that our algorithm ensures that under the event $\mathcal{B}$
\begin{equation}
    \label{eq:guarantee_eigenvalue_sketching}
    \lambda^{\mathsf{S}_{\mathrm{H}, t}X}_{\min} + \eta^2 \geq \gamma\cdot \left(1 + 2 \cdot \underset{i\in [n]}{\max} \ \norm{\left(\mathsf{S}_{\mathrm{H}, t}X\right)^{\top}\mathsf{S}_{\mathrm{H}, t} - x_i}\right), \ \ \mathrm{for\ all} \ \ t=0,1,\ldots, T-1    
\end{equation}
and for all the cases of the algorithm. Thus, by \lemmaref{lem:privacy_lemma}, the RDP guarantee of each of these steps is upper bounded by $\left(\alpha, k_1 \cdot\overline{\phi}(\alpha;\gamma)\right)$ where $\overline{\phi}(\alpha; \gamma)$ was defined in \sectionref{s:fast_mixing} and is given by  
\begin{equation}
    \overline{\phi}(\alpha; \gamma) \hspace{-0.15em}\coloneq\hspace{-0.15em} \frac{1}{2(\alpha - 1)}\hspace{-0.15em}\cdot\hspace{-0.15em} \left(\hspace{-0.15em}\alpha \cdot \log\left(\hspace{-0.15em}1 \hspace{-0.15em}-\hspace{-0.15em} \frac{1}{\gamma} \hspace{-0.15em}-\hspace{-0.15em} \frac{1}{4\gamma^2}\hspace{-0.15em}\right) \hspace{-0.15em}-\hspace{-0.15em} \log\left(\hspace{-0.15em}1 \hspace{-0.15em}-\hspace{-0.15em} \frac{\alpha}{\gamma} \hspace{-0.15em}-\hspace{-0.15em} \frac{\alpha^2}{4\gamma^2}\hspace{-0.15em}\right)\hspace{-0.15em}\right)\ \ \mathrm{for} \ \gamma > \nicefrac{5}{4} \ \mathrm{and} \ \alpha \in \left(1, \nicefrac{4\gamma}{5}\right).
\end{equation}

\textbf{Composition.}

We compose the $T$ different steps of the Gaussian mechanism and the $T$ sketching steps in RDP. Then, we convert this to $(\eps,\delta)$-DP via the conversion from \citet{canonne2020discrete}, and then compose with the steps that perform the calculations of $\widetilde{m}_t$ and $\widetilde{\lambda}_t$ in $(\eps,\delta)$-DP. For the $T$ steps of the Gaussian mechanism, we get 
\begin{align}
    \widehat{\eps}_{\text{\tiny Gauss}} &= \underset{\alpha > 1}{\min} \ \left\{\frac{\alpha T \textsc{C}^2}{2\sigma^2} + \frac{\log\left(\nicefrac{3}{\delta}\right) + (\alpha - 1)\log(1 - \nicefrac{1}{\alpha}) - \log(\alpha)}{\alpha - 1}\right\}\\
    &\leq \underset{\alpha > 1}{\min} \ \left\{\frac{\alpha T \textsc{C}^2}{2\sigma^2} + \frac{\log\left(\nicefrac{3}{\delta}\right)}{\alpha-1}\right\}\\
    &= \frac{\sqrt{2T\log\left(\nicefrac{3}{\delta}\right)}\textsc{C}}{\sigma} + \frac{T\textsc{C}^2}{2\sigma^2}
\end{align}
and $\widehat{\delta}_{\text{\tiny Gauss}} = \nicefrac{\delta}{3}$. For the $T$ sketching steps using the conversion from RDP to DP, it holds that 
\begin{equation}
    \widehat{\eps}_{\text{\tiny FastMix}} = \underset{1 < \alpha < \nicefrac{4\gamma}{5}}{\min} \ \left\{k_1T\cdot \overline{\phi}(\alpha;\gamma) + \frac{\log\left(\nicefrac{3}{\delta}\right) + (\alpha - 1)\log(1 - \nicefrac{1}{\alpha}) - \log(\alpha)}{\alpha - 1}\right\}
\end{equation}
and $\widehat{\delta}_{\text{\tiny FastMix}} = \nicefrac{\delta}{3}$, assuming that the event $\mathcal{B}$ holds. Then, using \lemmaref{lem:privacy_rare_event_adaptive_composition} with 
\begin{align}
    \pM_1 \gets \left\{\widetilde{m}_t, \widetilde{\lambda}_t\right\}^{T-1}_{t=0}, \ \pG_X \gets \left\{u\in\reals^{T}, v\in\reals^{T}: \ u_i \geq \widehat{m}_i, \ v_i \leq \lambda^{\mathsf{S}_{\mathrm{H},t }X}_{\min}\right\}, \ \pM_2 \gets \left\{\mathsf{S}_{\mathrm{G}, t}\widetilde{X}_t + \eta \xi_t\right\}^{T-1}_{t=0}
\end{align}
yields that overall these steps contributes $\nicefrac{2\delta}{3}$ to the final $\delta$ and $\frac{\eps}{3} + \widehat{\eps}_{\text{\tiny FastMix}}$ to the final $\eps$. 

Overall, we get that the algorithm is $\left(\widehat{\eps}, \widehat{\delta}\right)$-DP where 
\begin{equation}
    \widehat{\delta} = \frac{2\delta}{3} + \widehat{\delta}_{\text{\tiny Gauss}} = \delta
\end{equation}
and where $\widehat{\eps}$ satisfies
\begin{align}
    \label{eq:final_conversion}
    \widehat{\eps} &\leq \frac{\eps}{3} + \widehat{\eps}_{\text{\tiny FastMix}} + \widehat{\eps}_{\text{\tiny Gauss}}\\
    &\leq\hspace{-0.25em}\frac{\eps}{3}\hspace{-0.25em}+\hspace{-0.25em}\underset{1 < \alpha < \nicefrac{4\gamma}{5}}{\min}\hspace{-0.25em}\left\{k_1T\cdot \overline{\phi}(\alpha;\gamma) + \frac{\log\left(\nicefrac{3}{\delta}\right) + (\alpha - 1)\log(1 - \nicefrac{1}{\alpha}) - \log(\alpha)}{\alpha - 1}\right\} + \frac{\sqrt{2T\log\left(\nicefrac{3}{\delta}\right)}\textsc{C}}{\sigma} + \frac{T\textsc{C}^2}{2\sigma^2}\\
    &\leq\hspace{-0.25em}\frac{\eps}{3}\hspace{-0.25em}+\hspace{-0.25em}\underset{1 < \alpha < \nicefrac{4\gamma}{5}}{\min}\hspace{-0.25em}\left\{k_1T\cdot \overline{\phi}(\alpha;\gamma) + \frac{\log\left(\nicefrac{3}{\delta}\right)}{\alpha - 1}\right\} + \frac{\sqrt{2T\log\left(\nicefrac{3}{\delta}\right)}\textsc{C}}{\sigma} + \frac{T\textsc{C}^2}{2\sigma^2}.
\end{align} 
Now, for our choice of $\sigma$, the sum of the third and fourth terms equals $\nicefrac{\eps}{3}$. Moreover, by \eqref{eq:tcdp_conversion_finite_C} in \appendixref{app:tcdp_bound}, for our choice of $\gamma$ the second term is upper bounded by $\nicefrac{\eps}{3}$. Combining these bounds yields $\widehat{\eps} \leq \eps$ and concludes the proof. \hfill \qedsymbol

\subsubsection{Excess Empirical Risk:}
\label{app:proof_thm_linear_accuracy}
We follow the proof suggested in \citet{lev2026near} and seek to obtain an high probability bound on $\|\widehat{\theta}_T - \theta^{*}\|^2_{X^{\top}X}$, which corresponds to the excess empirical risk \citep[Lemma.~5]{wang_adassp}. 

Recall that in terms of notation, $\widetilde{m}_t$ denote the private version of $\widehat{m}_t$ and $\widetilde{\lambda}_t$ the private version of $\widehat{\lambda}_t$, where $\widehat{\lambda}_t \coloneq \lambda^{\mathsf{S}_{\mathrm{H},t} X}_{\min}.$ For any $a \geq 0$, we define $\overline{X}_a \coloneq (X^\top,\, a\brI_d)^\top$ and $M_a \coloneq X^\top X + a^2 \brI_d = \overline{X}_a^\top \overline{X}_a$. 

\paragraph{Step 1: Defining the High-Probability Event. }\mbox{}

We perform the analysis on an explicit high-probability event. For each
$t=0,\ldots,T-1$, define
\begin{align}
    \mathcal{B}_{1,t}(\chi) &\coloneq \left\{
        \max_{i \ne j}\abs{e_i^\top \mathsf{S}_{\mathrm{H},t}^\top \mathsf{S}_{\mathrm{H},t} e_j} \leq \chi\right\}, \\
    \mathcal{B}_{2,t}(\chi) &\coloneq \left\{\sup_{u \in \reals^n, v \in \reals^d}
        \frac{\left|u^\top\left(\mathsf{S}_{\mathrm{H},t}^\top \mathsf{S}_{\mathrm{H},t} - \brI_n\right)Xv\right|}
        {\norm{u}\norm{Xv}}
        \leq \frac{\chi}{2}, \inf_{v \in \reals^d}
        \frac{\norm{\mathsf{S}_{\mathrm{H},t} Xv}^2}{\norm{Xv}^2}
        \geq 1 - \chi\right\}.
\end{align}

By \citet[Lemma.~1]{pilanci_hessiansketch},
there exist constants $C_{\mathrm{H}}, c_{\mathrm{H}}, \bar{c}_{\mathrm{H}}$ such that for any $\chi \in (0, 1]$,
$\mathbb{P}\left(\mathcal{B}_{2,t}(\chi)\right) \geq 1-\bar{c}_{\mathrm{H}}\cdot \exp\left\{-\frac{c_{\mathrm{H}} k_2 \chi^2}{\log^4(n)}
\right\}$ provided that $k_2 \chi^2 \ge C_{\mathrm{H}} \cdot \mathrm{rank}(X)\cdot \log^4(n).$

Moreover, using $X=e_j$ for all $j\in [n]$ in \lemmaref{lem:nelson_prob}, there exist $C_{\mathrm{coh}}, c_{\mathrm{coh}}, \bar{c}_{\mathrm{coh}} > 0$ such that $\mathbb{P}\left(\mathcal{B}_{1,t}(\chi)\right) \geq 1-\frac{\varrho}{4T}$ provided that $k_2\chi^2 \ge C_{\mathrm{coh}} \left(1+\log\left(\frac{\bar{c}_{\mathrm{coh}} T n^2}{\chi\varrho}\right)\log\left(\frac{c_{\mathrm{coh}} T n^2}{\varrho}\right)\right).$ Therefore, if
\begin{equation}
    k_2\chi^2 \hspace{-0.15em}\ge\hspace{-0.15em} \max\Bigg\{\hspace{-0.15em}
        C_{\mathrm{H}} \cdot \mathrm{rank}(X)\cdot \log^4(n),\frac{\log^4(n)}{c_{\mathrm{H}}}\log\hspace{-0.15em}\left(\hspace{-0.15em}\frac{4T\bar{c}_{\mathrm{H}}}{\varrho}\hspace{-0.15em}\right), C_{\mathrm{coh}}\hspace{-0.15em}\left(1\hspace{-0.15em}+\hspace{-0.15em}\log\hspace{-0.15em}\left(\hspace{-0.15em}\frac{\bar{c}_{\mathrm{coh}} T n^2}{\chi\varrho}\hspace{-0.15em}\right)\log\hspace{-0.15em}\left(\hspace{-0.15em}\frac{c_{\mathrm{coh}} T n^2}{\varrho}\right)\hspace{-0.15em}\right)\hspace{-0.15em}\Bigg\},
\end{equation}
then $\mathbb{P}\left(\mathcal{B}_{1,t}^{\mathrm{c}}(\chi)
\right)\le \frac{\varrho}{4T}$ and $\mathbb{P}\left(\mathcal{B}_{2,t}^{\mathrm{c}}(\chi)\right)\le \frac{\varrho}{4T}$ for every $t=0,1,\ldots, T-1$. Now, define
\begin{align}
    \mathcal{B}_{\mathrm{tot}}(\chi) \coloneq \bigcap_{t=0}^{T-1}\left(\mathcal{B}_{1,t}(\chi)\cap \mathcal{B}_{2,t}(\chi)\right).    
\end{align}
By the union bound,
\begin{equation}
    \mathbb P\left(\mathcal{B}_{\mathrm{tot}}^{\mathrm{c}}(\chi)\right) \leq \sum_{t=0}^{T-1}\mathbb{P}\left(\mathcal{B}_{1,t}^{\mathrm{c}}(\chi)\right) + \sum_{t=0}^{T-1}\mathbb{P}\left(\mathcal{B}_{2,t}^{\mathrm{c}}(\chi)\right) \leq T\cdot \frac{\varrho}{4T} + T\cdot \frac{\varrho}{4T}
     = \frac{\varrho}{2},
\end{equation}
and hence $\mathbb P(\mathcal{B}_{\mathrm{tot}}(\chi))\ge 1-\frac{\varrho}{2}$. We note that the event $\mathcal{B}_{\mathrm{tot}}(\chi)$ implies 
\begin{equation}
    \lambda^{\mathsf{S}_{\mathrm{H},t}X}_{\min}
    \ge (1-\chi)\lambda^{X}_{\min} \ \ \ \mathrm{and} \ \ \ \max_{i\in[n]}
    \norm{X^{\top}\left(\mathsf{S}^{\top}_{\mathrm{H},t}\mathsf{S}_{\mathrm{H},t}-\brI_n\right)e_i}
    \le \frac{\chi}{2}\norm{X}_{\mathrm{op}}
\end{equation}
for all $t=0,\ldots,T-1$. This holds since for every $i\in[n]$, the event $\mathcal{B}_{2,t}(\chi)$ implies
\begin{equation}
    \norm{X^{\top}\left(\mathsf{S}^{\top}_{\mathrm{H},t}\mathsf{S}_{\mathrm{H},t}-\brI_n\right)e_i}
    = \sup_{\norm{v}=1}
    \abs{e_i^\top\left(\mathsf{S}^{\top}_{\mathrm{H},t}\mathsf{S}_{\mathrm{H},t}-\brI_n\right)Xv}\notag \leq \frac{\chi}{2}\sup_{\norm{v}=1}\norm{Xv}
     = \frac{\chi}{2}\norm{X}_{\mathrm{op}}.
\end{equation}

Conditioning on $\mathcal{B}_{\mathrm{tot}}(\chi)$, we next define for each
$t=0,\ldots,T-1$ the events
\begin{align}
    &\mathcal{A}_{1,t}(\chi) \hspace{-0.3em}\coloneq\hspace{-0.3em} \left\{\hspace{-0.5em}
    \begin{aligned}
        &\sup_{u \in \reals^{k_2+d},v \in \reals^d}
        \hspace{-0.35em}\frac{
            \left|u^\top\hspace{-0.3em}\left(\hspace{-0.2em}\frac{1}{k_1}\hspace{-0.2em}\left(\mathsf{S}_{\mathrm{G},t}, \xi_t\right)^\top\hspace{-0.2em} \left(\mathsf{S}_{\mathrm{G},t}, \xi_t\right)
            \hspace{-0.2em}-\hspace{-0.2em} \brI_{k_2+d}\right)\hspace{-0.3em}\begin{pmatrix}\mathsf{S}_{\mathrm{H}, t}X\\ \eta \brI_d\end{pmatrix}\hspace{-0.2em}v\right|
        }{
            \norm{u}\norm{\begin{pmatrix}\mathsf{S}_{\mathrm{H}, t}X\\ \eta \brI_d\end{pmatrix}v}
        }
        \hspace{-0.25em}\leq\hspace{-0.25em} \frac{\chi}{2}, \inf_{v \in \reals^d}
        \frac{\norm{(\mathsf{S}_{\mathrm{G},t}, \xi_t)\hspace{-0.3em}\begin{pmatrix}\mathsf{S}_{\mathrm{H}, t}X\\ \eta \brI_d\end{pmatrix}\hspace{-0.2em}v}^2}{k_1\norm{\begin{pmatrix}\mathsf{S}_{\mathrm{H}, t}X\\ \eta \brI_d\end{pmatrix}v}^2}
        \hspace{-0.3em}\geq\hspace{-0.3em} 1\hspace{-0.25em}-\hspace{-0.25em}\chi
    \end{aligned}
    \hspace{-0.75em}\right\},\\
    &\mathcal{A}_{2,t} \coloneq \left\{
        \widehat m_t \le \widetilde m_t \le
        \widehat m_t + \frac{12\chi T\max\left\{\log\left(\nicefrac{3T}{\delta}\right), \log\left(\nicefrac{16T}{\varrho}\right)\right\}}{\eps}
    \right\},\ \mathcal{A}_{3,t} \coloneq \left\{\mathsf z_{2,t}\le \tau\right\},\\
    &\mathcal{A}_{4,t} \coloneq \left\{
        \norm{\zeta_t}^2 \le
        \left(\sqrt{2d}+\sqrt{4\log(\nicefrac{8T}{\varrho})}\right)^2
    \right\}.
\end{align}

By \citet[Lemma.~1]{pilanci_hessiansketch}, there exist constants
$C_{\mathrm{G}}, c_{\mathrm{G}}, \bar{c}_{\mathrm{G}} > 0$ such that
$\mathcal{A}_{1,t}(\chi)$ holds w.p. at least $1-\frac{\varrho}{8T}$ provided that $k_1\chi^2 \ge \max\left\{
    C_{\mathrm{G}} d,\,
    \frac{1}{c_{\mathrm{G}}}\log\!\left(\frac{8T\bar{c}_{\mathrm{G}}}{\varrho}\right)
\right\}.$ Moreover, by \lemmaref{lem:excess_prob_Gaussian} applied with $\xi \gets \zeta_t$ and $\varrho \gets \nicefrac{\varrho}{8T}$ it holds that  
\begin{equation}
    \mathbb P(\mathcal{A}_{4,t}^{\mathrm{c}})\le \left(\frac{\varrho}{8T}\right)^2 \le \frac{\varrho}{8T},
\end{equation}
and by our choice of $\tau$ and since $\mathsf z_{2,t}\sim \mathrm{Lap}(0,1)$,
\begin{equation}
    \mathbb P(\mathcal{A}_{3,t}^{\mathrm{c}})\le \frac{\varrho}{8T}.
\end{equation}

Finally, under $\mathcal{B}_{\mathrm{tot}}$ and for our choice of $\tau$,
\begin{equation}
    \mathbb{P}\left(\mathcal{A}_{2,t}^{\mathrm{c}} \mid \mathcal{B}_{\mathrm{tot}}(\chi)\right) \le \frac{\varrho}{8T}
\end{equation}
which holds since conditioned on $\mathcal{B}_{\mathrm{tot}}(\chi)$, $\mathcal{A}_{2,t}$ is implied by the event $\{\mathsf z_{1,t}\in[-\tau,\tau]\}$. Now, define
\begin{equation}
    \mathcal{A}_{\mathrm{tot}}(\chi) \coloneq \mathcal{B}_{\mathrm{tot}}(\chi) \cap \bigcap_{t=0}^{T-1} \left( \mathcal{A}_{1,t}(\chi)\cap \mathcal{A}_{2,t}\cap \mathcal{A}_{3,t}\cap \mathcal{A}_{4,t} \right).
\end{equation}
Then,
\begin{align}
    \mathbb{P}\left(\mathcal{A}_{\mathrm{tot}}^{\mathrm{c}}(\chi)\right)
    &\leq \mathbb{P}\left(\mathcal{B}_{\mathrm{tot}}^{\mathrm{c}}(\chi)\right)
    + \mathbb{P}\left(\bigcup_{t=0}^{T-1}\bigl(\mathcal{A}_{1,t}^{\mathrm{c}}(\chi)\cup\mathcal{A}_{2,t}^{\mathrm{c}}\cup\mathcal{A}_{3,t}^{\mathrm{c}}\cup\mathcal{A}_{4,t}^{\mathrm{c}}\bigr) \middle| \mathcal{B}_{\mathrm{tot}}(\chi)\right)\notag \\
    &\leq \frac{\varrho}{2} + \sum_{t=0}^{T-1}\hspace{-0.15em}\left(\frac{\varrho}{8T} + \frac{\varrho}{8T} + \frac{\varrho}{8T} + \frac{\varrho}{8T}\right)\notag\\
    &=\varrho
\end{align}
and which implies $\mathbb P(\mathcal{A}_{\mathrm{tot}}(\chi))\ge 1-\varrho,$\    provided that 
\begin{align}
    k_1 \chi^2 &\geq \max\left\{C_{\mathrm{G}} d, \frac{1}{c_{\mathrm{G}}}\log\left(\frac{8T\bar{c}_{\mathrm{G}}}{\varrho}\right)\right\},\\ 
    k_2\chi^2 \hspace{-0.15em}&\ge\hspace{-0.15em} \max\Bigg\{\hspace{-0.15em}C_{\mathrm{H}} \hspace{-0.15em}\cdot\hspace{-0.15em} \mathrm{rank}(X)\hspace{-0.15em}\cdot\hspace{-0.15em} \log^4(n),\frac{\log^4(n)}{c_{\mathrm{H}}}\log\hspace{-0.15em}\left(\hspace{-0.15em}\frac{4T\bar{c}_{\mathrm{H}}}{\varrho}\hspace{-0.15em}\right), C_{\mathrm{coh}}\hspace{-0.15em}\left(2\hspace{-0.15em}+\hspace{-0.15em}\log\hspace{-0.15em}\left(\hspace{-0.15em}\frac{\bar{c}_{\mathrm{coh}} T n^2}{\chi\varrho}\hspace{-0.15em}\right)\log\hspace{-0.15em}\left(\hspace{-0.15em}\frac{c_{\mathrm{coh}} T n^2}{\varrho}\right)\hspace{-0.15em}\right)\hspace{-0.15em}\Bigg\},
\end{align}
for constants $\left(C_{\mathrm{G}}, c_{\mathrm{G}}, \bar{c}_{\mathrm{G}}, C_{\mathrm{H}}, c_{\mathrm{H}}, \bar{c}_{\mathrm{H}}, C_{\mathrm{coh}}, c_{\mathrm{coh}}, \bar{c}_{\mathrm{coh}}\right)$ and $\chi, \varrho\in (0, 1]$. 


\paragraph{Step 2: Extending Auxiliary Lemmas From \citet{lev2026near}:}\mbox{}

Recall that \lemmaref{lem:probability_pilanci_concatenated} extends \citet[Lemma.~8]{lev2026near} by showing that the concatenated sketch satisfies $\mathcal{C}(A,\chi,k)$ for a matrix $A$, approximation parameter $\chi$, and sketch dimension $k$. It remains to extend \citet[Lemmas~10--12]{lev2026near}. For completeness, we restate these lemmas in terms of $\mathcal{C}(A,\chi,k)$ and derive them from \citet{lev2026near}. We begin by recalling the definition of $\mathcal{C}(A,\chi,k)$:
\begin{equation}
    \mathcal{C}(A, \chi, k) \hspace{-0.15em}\coloneq\hspace{-0.15em} \left\{\hspace{-0.15em}S\hspace{-0.15em}\in\hspace{-0.15em} \reals^{k\times n}: \hspace{-0.15em}\sup_{u \in \reals^{d}, v \in \reals^{d}} \frac{\left|u^{\top}A^{\top}\hspace{-0.15em}\left(\frac{1}{k}S^{\top}S \hspace{-0.15em}-\hspace{-0.15em} \brI_{n}\right)\hspace{-0.15em} A v\right|}{\|A u\| \cdot \|A v\|} \hspace{-0.15em}\leq\hspace{-0.15em} \frac{\chi}{2} \bigcap \inf_{v \in \reals^{d}} \frac{\norm{SA v}^2}{k \cdot \|Av\|^{2}} \hspace{-0.15em}\geq\hspace{-0.15em} 1\hspace{-0.15em}-\hspace{-0.15em}\chi\hspace{-0.15em}\right\}.
\end{equation}
\begin{lem}{(Extension of \citet[Lemma.~10]{lev2026near}).} \ 
    \label{lem:extend_lem_10}
    Given $\chi \in \big(0, \nicefrac{1}{2}\big]$, $A \in \reals^{n\times d}, k\in \nats$, $B \in \mathcal{C}(A, \chi, k)$, $b \in \reals^{n}$ and $\theta_{0} \in \reals^{d}$ we have
    \begin{equation}
        \norm{\widehat{\theta}(A, b, B, \theta_0, \vec{0}_d) - \theta^{*}(A,b)}^2_{A^{\top}A} \leq \chi^2 \norm{\theta_0 - \theta^{*}(A,b)}^2_{A^{\top}A}.
    \end{equation}
\end{lem}
\begin{proof}
    As in \citet{lev2026near}, the result follows directly from the analysis in \citet[Appendix.~C]{pilanci_hessiansketch}, with $\theta_0$ playing the role of $x^t$ in \citet[Equation.~42]{pilanci_hessiansketch} and $B$ playing the role of the sketch $\mathsf{S}$. The desired guarantee then follows from the assumption $B\in \mathcal{C}(A,\chi,k)$, since the event in \citet[Equation.~43]{pilanci_hessiansketch} corresponds to the setting of $\mathcal{C}$, where the left and right vectors multiplying $\frac{1}{k}B^{\top}B-\brI_n$ are of the form $Av$. Hence, \citet[Equation.~44a]{pilanci_hessiansketch} holds under events of the form $B\in \mathcal{C}(A,\chi,k)$.
\end{proof}

\begin{lem}{(Extension of \citet[Lemma.~11]{lev2026near}).} \ 
    \label{lem:extend_lem_11}
    Given $\chi \in (0, 1]$, $A \in \reals^{n\times d}$ and $B \in \reals^{k\times n}$ such that $B \in \mathcal{C}(A, \chi, k)$ and $A^{\top}A\succ 0$ it holds  
    \begin{equation}
        \label{eq:PSD_Lemma_1}
        \frac{1}{1+\frac{\chi}{2}}\left(A^{\top}A\right)^{-1} \preceq \left(\frac{1}{k}(BA)^{\top}(BA)\right)^{-1} \preceq \frac{1}{1-\frac{\chi}{2}}\left(A^{\top}A\right)^{-1}
    \end{equation}
    and furthermore for any vector $v\in \reals^{d}$ 
    \begin{equation}
        \label{eq:PSD_Lemma_2}
        k^2\norm{\left((BA)^{\top}(BA)\right)^{-1}v}^2_{A^{\top}A} \leq \frac{1}{\left(1 - \frac{\chi}{2}\right)^{2}}\cdot \frac{\norm{v}^2}{\lambda_{\min}\left(A^{\top}A\right)}. 
    \end{equation}
\end{lem}
\begin{proof}
    The proof follows \citet[Lemma.~11]{lev2026near}, noting that \citet[Equation.~23]{lev2026near} remains valid under the event $B\in \mathcal{C}(A,\chi,k)$, since it only requires multiplying $\frac{1}{k}B^{\top}B$ from both sides by a vector of the form $Av$, thus covered by $\mathcal{C}(A,\chi,k)$. The rest of the proof is unchanged.
\end{proof}

\begin{lem}{(Extension of \citet[Lemma.~12]{lev2026near}).} \ 
    \label{lem:adaptation_no_clipping}
    Let $Y\hspace{-0.15em}\in\hspace{-0.15em} \reals^{n}$ and $X\hspace{-0.15em}\in\hspace{-0.15em} \reals^{n\times d}$ satisfying \(\textsc{C}_{X} \hspace{-0.15em}=\hspace{-0.15em} 1\), and let \(M \hspace{-0.15em}\coloneq\hspace{-0.15em} X^{\top}\hspace{-0.15em}X\). Suppose \(\{B_{t}\}_{t = 0}^{T - 1}\) and \(\{v_{t}\}_{t = 0}^{T - 1}\) satisfy \(B_{t} \hspace{-0.15em}\in\hspace{-0.15em} \mathcal{C}\left(X, \chi, k\right)\) and \(v_{t} \hspace{-0.15em}\in\hspace{-0.15em} \mathcal{A}_{4}(\varrho, d)\) for all \(t \hspace{-0.15em}\in\hspace{-0.15em} \{0,\ldots, T-1\}\) and for parameters \(\chi, \varrho \hspace{-0.15em}\in\hspace{-0.15em} (0, 1]\) and \(k \hspace{-0.15em}\in\hspace{-0.15em} \nats\). Assume that $\frac{\chi\sqrt{\kappa(M)}}{2 - \chi} \hspace{-0.15em}<\hspace{-0.15em} 1$ where $\kappa(M) \hspace{-0.15em}\coloneq\hspace{-0.15em} \frac{\lambda_{\max}(M)}{\lambda_{\min}(M)}.$ If \(\{\widehat{\theta}_{t}\}_{t = 0}^{T - 1}\) is defined according to the recursion
    \begin{equation}
        \widehat{\theta}_{t+1}\left(\widehat{\theta}_t, \sigma v_t\right) = \widehat{\theta}_t + \left(\frac{1}{k}X^{\top}B^{\top}_tB_tX\right)^{-1}\left(X^{\top}\left(Y - X\widehat{\theta}_t\right) + \sigma v_t\right)~, \quad \widehat{\theta}_{0} = \vec{0}_{d},
    \end{equation}
    for $\sigma \geq 0$, then
    \begin{equation}
        \norm{Y - X\widehat{\theta}_t}_{\infty} \leq \textsc{C}_Y + \left(1 + \left(\frac{\chi\sqrt{\kappa(M)}}{2 - \chi}\right)^t\right)\norm{\theta^{*}} + \frac{2 \sigma}{2 - \chi(1 + \sqrt{\kappa(M))}} \cdot \frac{\sqrt{2d} + \sqrt{4\log(\nicefrac{8T}{\varrho})}}{\lambda_{\min}(M)}.
    \end{equation}
\end{lem}
\begin{proof}
    The proof follows since \lemmaref{lem:extend_lem_10} and \lemmaref{lem:extend_lem_11} implies that \citet[Lemma.~11]{lev2026near} and \citet[Lemma.~10]{lev2026near} holds. The remaining steps follow as in \citet{lev2026near}.    
\end{proof}
\paragraph{Step 3: Adjusting the Overall Guarantee From \citet{lev2026near}:}\mbox{}

Throughout, we denote $\overline{X}_{\eta} \coloneq \left(X^{\top}, \eta \brI_d\right)^{\top}$. Before presenting the proof we recall that for any $\chi \in (0, 1]$, conditioned on $\mathcal{A}_{\mathrm{tot}}(\chi)$ it holds that 
\begin{equation}
    \label{eq:guarantee_eigenvalue}
    \left(1 - \chi\right)\lambda^{X}_{\min} \leq \widehat{\lambda}_t \leq \left(1+\chi\right)\lambda^{X}_{\min}, \ \ \widehat{\Delta}_t \leq \chi, \ \ \widehat{m}_t \leq \chi\sqrt{\lambda^{X}_{\max}} 
\end{equation}
and $\widetilde{\lambda}_t \hspace{-0.15em}\leq\hspace{-0.15em} \widehat{\lambda}_t, \ \widehat{m}_t \hspace{-0.15em}\leq\hspace{-0.15em} \widetilde{m}_t \hspace{-0.15em}\leq\hspace{-0.15em} \widehat{m}_t \hspace{-0.15em}+\hspace{-0.15em} \frac{12\chi T\cdot\max\left\{\log\left(\nicefrac{3T}{\delta}\right), \log\left(\nicefrac{16T}{\varrho}\right)\right\}}{\eps}$, all holds $\forall t\hspace{-0.15em}=\hspace{-0.15em}0,1,\ldots, T-1$. Consequently,
\begin{equation}
    \overline{m} \coloneq \underset{t=0,1,\ldots, T-1}{\max} \ {\widetilde{m}_t}\leq \chi\left(\sqrt{\lambda^{X}_{\max}} + \frac{12T\cdot \max\left\{\log\left(\nicefrac{3T}{\delta}\right), \log\left(\nicefrac{16T}{\varrho}\right)\right\}}{\eps}\right).     
\end{equation}
Additionally, $\mathcal{A}_{\mathrm{tot}}(\chi)$ holds under the next choice of parameters 
\begin{align}
    \label{eq:choices_k1_k_2}
    k_1 \chi^2 &\geq \max\left\{C_{\mathrm{G}} d, \frac{1}{c_{\mathrm{G}}}\log\left(\frac{8T\bar{c}_{\mathrm{G}}}{\varrho}\right)\right\},\\ 
    k_2\chi^2 \hspace{-0.15em}&\ge\hspace{-0.15em} \max\Bigg\{\hspace{-0.15em}C_{\mathrm{H}} \hspace{-0.15em}\cdot\hspace{-0.15em} \mathrm{rank}(X)\hspace{-0.15em}\cdot\hspace{-0.15em} \log^4(n),\frac{\log^4(n)}{c_{\mathrm{H}}}\log\hspace{-0.15em}\left(\hspace{-0.15em}\frac{4T\bar{c}_{\mathrm{H}}}{\varrho}\hspace{-0.15em}\right), C_{\mathrm{coh}}\hspace{-0.15em}\left(2\hspace{-0.15em}+\hspace{-0.15em}\log\hspace{-0.15em}\left(\hspace{-0.15em}\frac{\bar{c}_{\mathrm{coh}} T n^2}{\chi\varrho}\hspace{-0.15em}\right)\log\hspace{-0.15em}\left(\hspace{-0.15em}\frac{c_{\mathrm{coh}} T n^2}{\varrho}\right)\hspace{-0.15em}\right)\hspace{-0.15em}\Bigg\}.
\end{align}
Lastly, we assume that $\gamma = \gamma_1$, which puts the additional constraint $k_1T = \Omega\left(\log(\nicefrac{1}{\delta})\right)$. 

Starting with the no-clipping conditions (paralleling \citet[Appendix .~C, Step.~2]{lev2026near}), we follow the notation from \citet[Appendix.~C]{lev2026near} and define $\kappa^{X}(a) \hspace{-0.2em}\coloneq\hspace{-0.2em} \frac{\lambda^{X}_{\max} \hspace{-0.15em}+\hspace{-0.15em} a}{\lambda^{X}_{\min} \hspace{-0.15em}+\hspace{-0.15em} a}, \kappa(A) \hspace{-0.2em}\coloneq\hspace{-0.2em} \frac{\lambda_{\max}(A)}{\lambda_{\min}(A)}$. First, for $t=0$ no clipping occurs whenever $\textsc{C} \geq \textsc{C}_Y$. Then, for every $t\geq 1$, given $\chi \hspace{-0.2em}\in\hspace{-0.2em} (0,1]$ satisfying $\frac{\chi\sqrt{\kappa^{X}\left(\eta^2\right)}}{2 - \chi} \hspace{-0.2em}<\hspace{-0.2em} 1$, by \lemmaref{lem:adaptation_no_clipping} and \citet[Appendix.~C, Step 2]{lev2026near}
\begin{align}
    \norm{Y - X\widehat{\theta}_t}_{\infty} &\leq \textsc{C}_Y + \norm{\theta^{*}} + \frac{1}{1 - \frac{\chi\sqrt{\kappa^{X}\left(\eta^2\right)}}{2-\chi}}\cdot\frac{\sigma}{1 - \frac{\chi}{2}}\cdot\frac{\sqrt{2d} + \sqrt{4\log\left(\nicefrac{8T}{\varrho}\right)}}{\lambda^{X}_{\min} + \eta^2} + \left(\frac{\chi\sqrt{\kappa^{X}\left(\eta^2\right)}}{2 - \chi}\right)^{t}\norm{\theta^{*}}\\
    &\overset{(a)}\leq \textsc{C}_Y + \norm{\theta^{*}} + \frac{1}{1 - \frac{\chi\sqrt{\kappa^{X}\left(\eta^2\right)}}{2-\chi}}\cdot\frac{\sigma}{1 - \frac{\chi}{2}}\cdot\frac{\sqrt{2d} + \sqrt{4\log\left(\nicefrac{8T}{\varrho}\right)}}{\left(1 - \chi\right)\gamma} + \frac{\chi\sqrt{\kappa^{X}\left(\eta^2\right)}}{2 - \chi}\norm{\theta^{*}}\\
    &\overset{(b)}\leq \textsc{C}_Y + \norm{\theta^{*}} + \frac{1}{1 - \frac{\chi\sqrt{\kappa^{X}\left(\eta^2\right)}}{2-\chi}}\cdot\frac{\textsc{C}}{\left(1 - \chi\right)^2}\cdot\frac{\sqrt{2d} + \sqrt{4\log\left(\nicefrac{8T}{\varrho}\right)}}{\sqrt{k_1}} + \frac{\chi\sqrt{\kappa^{X}\left(\eta^2\right)}}{2 - \chi}\norm{\theta^{*}}\\
    &\overset{(c)}\leq \textsc{C}_Y + \norm{\theta^{*}} + \frac{\chi}{1 - \frac{\chi\sqrt{\kappa^{X}\left(\eta^2\right)}}{2 - \chi}}\cdot\frac{\textsc{C}}{\left(1 - \chi\right)^2}
    + \frac{\chi\sqrt{\kappa^{X}\left(\eta^2\right)}}{2 - \chi}\norm{\theta^{*}}
\end{align}
where (a) is due to \eqref{eq:guarantee_eigenvalue}, since
\begin{equation}
    \label{eq:inequality_eigenvalues_gamma}
    \left(1 - \chi\right)\gamma \leq \left(1 - \chi\right)\left(\widehat{\lambda}_t + \eta^2\right) \leq \left(1 - \chi\right)\widehat{\lambda}_t + \eta^2 \leq \frac{\widehat{\lambda}_t}{1 + \chi} + \eta^2 \leq \lambda^{X}_{\min} + \eta^2,
\end{equation}
(b) is by substituting $\gamma$ and $\sigma$, since $\frac{1 + \sqrt{1 + x}}{1 + \sqrt{1 + \frac{x}{4}}} \leq 2$, and (c) holds by \eqref{eq:choices_k1_k_2}, for sufficiently large (yet still constants) $C_{\mathrm{G}}, \frac{1}{c_{\mathrm{G}}}$ and $\overline{c}_{\mathrm{G}}$. This is below $\textsc{C}$ whenever
\begin{equation}
    \frac{\textsc{C}_Y}{1 - \frac{\chi}{\left(1 - \frac{\chi\sqrt{\kappa^{X}(\eta^2)}}{2-\chi}\right)(1-\chi)^2}} + \frac{1 + \frac{\chi\sqrt{\kappa^{X}(\eta^2)}}{2-\chi}}{1 - \frac{\chi}{\left(1 - \frac{\chi\sqrt{\kappa^{X}(\eta^2)}}{2-\chi}\right)(1-\chi)^2}}\norm{\theta^{*}} \leq \textsc{C}
\end{equation}
provided that $\frac{\chi}{\left(1 - \frac{\chi\sqrt{\kappa^{X}(\eta^2)}}{2-\chi}\right)(1-\chi)^2} < 1$. As in \citet[Appendix.~C]{lev2026near}, for $\chi$ below a sufficiently small constant less than $\frac{1}{4}$, attainable by increasing $k_1$ and $k_2$ by constant factors, this bound is less than $\textsc{C}$ whenever $\textsc{C} \hspace{-0.15em}\geq\hspace{-0.15em} a_0 \hspace{-0.15em}\cdot\hspace{-0.15em} \max\hspace{-0.15em}\left\{\textsc{C}_Y, \norm{\theta^{*}}\right\}$ and $k_1 \hspace{-0.15em}\geq\hspace{-0.15em} a_1\hspace{-0.15em}\cdot\hspace{-0.15em}\kappa^{X}(\eta^2)\hspace{-0.15em}\cdot\hspace{-0.15em} d, \ k_2 \hspace{-0.15em}\geq\hspace{-0.15em} a_1\hspace{-0.15em}\cdot\hspace{-0.15em}\kappa^{X}(\eta^2)\hspace{-0.15em}\cdot\hspace{-0.15em} d\hspace{-0.15em}\cdot\hspace{-0.15em} \log^4(n)$ for constants $a_0, a_1$. Now, to further simplify this argument, we note that 
\begin{equation}
    \kappa^{X}\left(\eta^2\right) = \frac{\lambda^{X}_{\max} - \lambda^{X}_{\min}}{\lambda^{X}_{\min} + \eta^2} + 1 \leq \frac{\lambda^{X}_{\max} - \lambda^{X}_{\min}}{\left(1 - \chi\right)\gamma} + 1 
\end{equation}
which holds by using \eqref{eq:inequality_eigenvalues_gamma}. 
Then, provided that $\chi\in (0, \nicefrac{1}{4}]$, following analysis similar to that presented in \citet{lev2026near} we get that the constraints on $k_1, k_2$ can be enforced by setting  
\begin{equation}
    k_1 \geq a_2\cdot \left(d \cdot \frac{\left(\lambda^{X}_{\max} -\lambda^{X}_{\min}\right)\eps}{\sqrt{\log(\nicefrac{1}{\delta})}}\right)^{\nicefrac{2}{3}}, \ \ \ k_2 \geq a_2\cdot \left(d \cdot \frac{\left(\lambda^{X}_{\max} -\lambda^{X}_{\min}\right)\eps}{\sqrt{\log(\nicefrac{1}{\delta})}}\cdot \log^4(n)\right)^{\nicefrac{2}{3}}.
\end{equation}
Overall, conditioned on $\mathcal{A}_{\mathrm{tot}}$ clipping does not occur  provided that $\textsc{C} \hspace{-0.15em}\geq\hspace{-0.15em} a_0 \cdot \max\hspace{-0.15em}\left\{\textsc{C}_Y, \norm{\theta^{*}}\right\}$ and 
\begin{align}
    k_1\chi^2 &\geq a_2\cdot \max\left\{d, g^{X}, \log\left(\frac{8T\bar{c}_{\mathrm{G}}}{\varrho}\right), \left(\sqrt{2d} + \sqrt{4\log\left(\nicefrac{8T}{\varrho}\right)}\right)^2\right\}, \\ 
    k_2\chi^2 &\hspace{-0.15em}\geq\hspace{-0.15em} a_2 \hspace{-0.15em}\cdot\hspace{-0.15em} \max\hspace{-0.15em}
    \Bigg\{\hspace{-0.15em}
        \mathrm{rank}(X)\log^4(n),
        \log^{\nicefrac{8}{3}}(n)g^{X}, \log^4(n)\log\hspace{-0.15em}\left(\hspace{-0.15em}\frac{4T\bar{c}_{\mathrm{H}}}{\varrho}\hspace{-0.15em}\right),2\hspace{-0.15em}+\hspace{-0.15em}\log\hspace{-0.15em}\left(\hspace{-0.15em}\frac{\bar{c}_{\mathrm{coh}} T n^2}{\chi\varrho}\hspace{-0.15em}\right)
        \hspace{-0.15em}\log\hspace{-0.15em}\left(\hspace{-0.15em}\frac{c_{\mathrm{coh}} T n^2}{\varrho}\hspace{-0.15em}\right)
    \hspace{-0.15em}\Bigg\}
\end{align}
for $\chi \leq \nicefrac{1}{4}$ and constants $a_0, a_2$, and recall that $g^{X}$ was defined in the theorem statement and is given by $g^{X} \coloneq \left(\frac{\eps d}{\sqrt{\log(\nicefrac{1}{\delta})}}\cdot \left(\lambda^{X}_{\max} - \lambda^{X}_{\min}\right)\right)^{\nicefrac{2}{3}}$. 

Now, given a vector $\theta$, let
\begin{equation}
    \widehat{\theta}_{t + 1}\left(\eta,\sigma;\theta\right)\coloneq\theta+\left(\frac{1}{k_1}\left(\mathsf{S}_{\mathrm{tot}, t}\overline{X}_{\eta}\right)^{\top}\left(\mathsf{S}_{\mathrm{tot}, t}\overline{X}_{\eta}\right)\right)^{-1}\left(X^{\top}(Y-X\theta)-{\eta}^{2}\theta+\sigma\zeta_{t}\right)\hspace{-0.15em}
\end{equation}
where $\mathsf{S}_{\mathrm{tot}, t} \coloneq \left(\mathsf{S}_{\mathrm{G}, t}, \xi_t\right)\begin{pmatrix}
    \mathsf{S}_{\mathrm{H}, t} & 0_{k_2\times d}\\
    0_{d\times n} & \brI_d
\end{pmatrix}$, $\widehat{\theta}_{t+1}\left(\eta, \sigma, \widehat{\theta}_t\right)$
is the update step from Line~\ref {line:update_step}, and we omitted the clipping provided that the previous conditions are met. Then, using \lemmaref{lem:extend_lem_11} and \lemmaref{lem:extend_lem_10} the analysis presented in \citet[Appendix.~C, Step.~3]{lev2026near} holds and yields
\begin{equation}
    \label{eq:excess_empirical_risk_step1}
    \norm{\widehat{\theta}_{T}(\eta, \sigma) \hspace{-0.15em}-\hspace{-0.15em} \theta^{*}}^2_{X^{\top}X} \hspace{-0.25em}\leq\hspace{-0.25em} O\hspace{-0.25em}\left(\hspace{-0.2em}\frac{\sigma^2 \max\left\{d, \log(\nicefrac{T}{\varrho})\right\}}{\lambda^{X}_{\min} \hspace{-0.15em}+\hspace{-0.15em} \eta^2}\hspace{-0.2em}\right) \hspace{-0.25em}+\hspace{-0.25em} \left(\sqrt{2}\chi\right)^{2T}\hspace{-0.3em}\kappa^{X}\hspace{-0.2em}\left(\eta^2\right)\hspace{-0.2em}\left(\lambda^{X}_{\min} \hspace{-0.2em}+\hspace{-0.2em} \eta^2\right)\hspace{-0.15em}\norm{\theta^{*}}^2 \hspace{-0.2em}+\hspace{-0.2em} \eta^2 \hspace{-0.15em}\norm{\theta^{*}}^2. 
\end{equation}
Then, substituting
\begin{align}
    &\gamma \hspace{-0.15em}=\hspace{-0.15em}\gamma_1 \hspace{-0.15em}=\hspace{-0.15em} \Theta\hspace{-0.15em}\left(\hspace{-0.15em}\frac{\sqrt{k_1T\log(\nicefrac{1}{\delta})}}{\eps}\hspace{-0.15em}\right), \ \sigma^2 \hspace{-0.15em}=\hspace{-0.15em} \Theta\hspace{-0.15em}\left(\hspace{-0.15em}\frac{T\log(\nicefrac{1}{\delta})\textsc{C}^2}{\eps^2}\hspace{-0.15em}\right), \ \widehat{m}_t \hspace{-0.15em}\leq\hspace{-0.15em} \chi\sqrt{\lambda^{X}_{\max}}, \ \widetilde{\lambda}_t \leq \widehat{\lambda}_t\\
    &\left(1 - \chi\right)\lambda^{X}_{\min} \hspace{-0.15em}\leq\hspace{-0.15em} \widehat{\lambda}_t \hspace{-0.15em}\leq\hspace{-0.15em} \left(1 \hspace{-0.15em}+\hspace{-0.15em} \chi\right)\lambda^{X}_{\min}, \ \widehat{m}_t \hspace{-0.15em}\leq\hspace{-0.15em} \widetilde{m}_t \hspace{-0.15em}\leq\hspace{-0.15em} \widehat{m}_t \hspace{-0.15em}+\hspace{-0.15em} \frac{12\chi T\cdot \max\left\{\log\left(\nicefrac{3T}{\delta}\right), \log\left(\nicefrac{16T}{\varrho}\right)\right\}}{\eps}, \\ 
    &\qquad\qquad\qquad\qquad\qquad \widehat{\lambda}_t + \eta^2 = \begin{cases}
        \gamma\left(1 + 2\widetilde{m}_t\right) \ &\mathrm{if} \ \eta^2 > 0\\ 
        \widehat{\lambda}_t \ &\mathrm{otherwise}
    \end{cases}, \\ 
    &\qquad\qquad \lambda^{X}_{\min} + \eta^2 \leq \frac{\widehat{\lambda}_t}{1 - \chi} + \eta^2 \leq \frac{\widehat{\lambda}_t + \eta^2}{1 - \chi}, \ \ \ \lambda^{X}_{\min} + \eta^2 \geq \frac{\widehat{\lambda}_t}{1 + \chi} + \eta^2 \geq \frac{\widehat{\lambda}_t + \eta^2}{1 + \chi}
\end{align}
 in \eqref{eq:excess_empirical_risk_step1} and using the notation $\overline{m} \coloneq \underset{t=0,1,\ldots, T-1}{\max} \widetilde{m}_t$ yields
\begin{align}
    &\norm{\widehat{\theta}_{T}(\eta, \sigma) \hspace{-0.15em}-\hspace{-0.15em} \theta^{*}}^2_{X^{\top}X} \hspace{-0.2em}\\
    &\leq\hspace{-0.2em} O\hspace{-0.15em}\left(\hspace{-0.15em}\frac{\sigma^2 \max\left\{d, \log(\nicefrac{T}{\varrho})\right\}}{\widehat{\lambda}_t \hspace{-0.15em}+\hspace{-0.15em} \eta^2}\hspace{-0.15em}\right) \hspace{-0.2em}+\hspace{-0.2em} \left(\sqrt{2}\chi\right)^{2T}\hspace{-0.2em}\kappa^{X}\hspace{-0.2em}\left(\eta^2\right)\hspace{-0.2em}\left(\hspace{-0.15em}\frac{\widehat{\lambda}_t \hspace{-0.15em}+\hspace{-0.15em} \eta^2}{1 - \chi}\hspace{-0.15em}\right)\hspace{-0.15em}\norm{\theta^{*}}^2 \hspace{-0.15em}+\hspace{-0.15em} \eta^2\norm{\theta^{*}}^2, \ \mathrm{for\ all\ } t=0,1,\ldots, T-1, \\
    &\leq\hspace{-0.2em} \begin{cases}
            \hspace{-0.25em}O\hspace{-0.25em}\left(\frac{\sigma^2 \max\left\{d, \log(\nicefrac{T}{\varrho})\right\}}{\gamma}\right) + \left(\sqrt{2}\chi\right)^{2T} \kappa^{X}(\eta^2)\frac{\gamma\left(1 + 2\overline{m}\right)}{1 - \chi}\norm{\theta^{*}}^2 + \eta^2\norm{\theta^{*}}^2 \qquad\qquad\qquad\qquad\qquad &\mathrm{if} \ \lambda^{X}_{\min} \hspace{-0.35em}<\hspace{-0.35em} \frac{\gamma\left(1+2\overline{m}\right)}{1-\chi} \\ 
            \hspace{-0.25em}O\hspace{-0.25em}\left(\frac{\sigma^2 \max\left\{d, \log(\nicefrac{T}{\varrho})\right\}}{\lambda^{X}_{\min}}\right) + \left(\sqrt{2}\chi\right)^{2T} \kappa^{X}(0)\lambda^{X}_{\min}\frac{1+\chi}{1-\chi}\norm{\theta^{*}}^2   &\mathrm{otherwise}
        \end{cases}\\
    &\leq\hspace{-0.2em} \begin{cases}
            \hspace{-0.25em}O\hspace{-0.25em}\left(\hspace{-0.25em}\frac{\sqrt{T\max\left\{d, \log(\nicefrac{T}{\varrho}), g^{X}\right\}\hspace{-0.1em}\log(\nicefrac{1}{\delta})}\left(\hspace{-0.1em}\textsc{C}^2_Y \hspace{-0.1em}+\hspace{-0.1em} \norm{\theta^{*}}^2\hspace{-0.1em}\right)\left(\hspace{-0.1em}1 \hspace{-0.1em}+\hspace{-0.1em} 2\overline{m}\right)}{\eps}\hspace{-0.3em}\left(\hspace{-0.2em}1\hspace{-0.25em}+\hspace{-0.25em} (\sqrt{2}\chi)^{2T}\hspace{-0.15em}\kappa^{X}(\eta^2) \hspace{-0.25em}+\hspace{-0.25em}\frac{1}{1\hspace{-0.1em}+\hspace{-0.1em} 2\overline{m}}\right) \hspace{-0.25em}-\hspace{-0.25em} \lambda^{X}_{\min}\hspace{-0.25em}\left(\hspace{-0.2em}\textsc{C}^2_Y \hspace{-0.2em}+\hspace{-0.2em} \norm{\theta^{*}}^2\hspace{-0.2em}\right)\hspace{-0.25em}\right)     &\mathrm{if} \ \lambda^{X}_{\min} \hspace{-0.35em}<\hspace{-0.35em} \frac{\gamma\left(1\hspace{-0.1em}+\hspace{-0.1em}2\overline{m}\right)}{1\hspace{-0.1em}-\hspace{-0.1em}\chi} \\ 
            \hspace{-0.25em}O\hspace{-0.25em}\left(\hspace{-0.2em}\frac{T \max\left\{d, \log(\nicefrac{T}{\varrho})\right\}\hspace{-0.1em}\log(\nicefrac{1}{\delta})\left(\hspace{-0.1em}\textsc{C}^2_Y \hspace{-0.1em}+\hspace{-0.1em} \norm{\theta^{*}}^2\hspace{-0.1em}\right)}{\eps^2\lambda^{X}_{\min}} \hspace{-0.2em}+\hspace{-0.2em} \left(\sqrt{2}\chi\right)^{2T} \hspace{-0.2em}\kappa^{X}(0)\lambda^{X}_{\min}\norm{\theta^{*}}^2\hspace{-0.2em}\right)   &\mathrm{otherwise}
        \end{cases}\\
    &\leq\hspace{-0.2em} \begin{cases}
            \hspace{-0.25em}O\hspace{-0.25em}\left(\frac{\sqrt{T\max\left\{d, \log(\nicefrac{T}{\varrho}), g^{X}\right\}\log(\nicefrac{1}{\delta})}\left(\textsc{C}^2_Y + \norm{\theta^{*}}^2\right)\left(1 + 2\overline{m}\right)}{\eps}\left(1 \hspace{-0.2em}+\hspace{-0.2em} \frac{1}{1+2\overline{m}}\right) \hspace{-0.2em}-\hspace{-0.2em} \lambda^{X}_{\min}\hspace{-0.2em}\left(\textsc{C}^2_Y \hspace{-0.2em}+\hspace{-0.2em} \norm{\theta^{*}}^2\right)\right)     \qquad\qquad\quad&\mathrm{if} \ \lambda^{X}_{\min} \hspace{-0.35em}<\hspace{-0.35em} \frac{\gamma\left(1 + 2\overline{m}\right)}{1 - \chi} \\ 
            \hspace{-0.25em}O\hspace{-0.25em}\left(\frac{T \max\left\{d, \log(\nicefrac{T}{\varrho})\right\}\log(\nicefrac{1}{\delta})\left(\textsc{C}^2_Y + \norm{\theta^{*}}^2\right)}{\eps^2\lambda^{X}_{\min}} + 8^{-T}\lambda^{X}_{\min}\norm{\theta^{*}}^2 \right)  &\mathrm{otherwise}
        \end{cases}
\end{align}
where we have used the lower bound $\widehat{\lambda}_t + \eta^2 = \gamma(1 + 2\widetilde{m}_t) \geq \gamma$ for the case where $\eta^2 > 0$ and the last step is for $\chi \hspace{-0.25em}\leq\hspace{-0.25em} \frac{1}{4}$ and $\frac{\chi\sqrt{\hspace{-0.1em}\kappa^{X}\hspace{-0.1em}(\eta^2)}}{2-\chi}\hspace{-0.25em}<\hspace{-0.25em}1$. In the first case, the optimal $T$ is constant; in the second, 
$T^{\star}\hspace{-0.25em}=\hspace{-0.25em}O\hspace{-0.3em}\left(\hspace{-0.2em}\log\hspace{-0.2em}\left(\hspace{-0.2em}\frac{(\eps\lambda_{\min}^{X})^2}{\max\left\{d,\log(\nicefrac{1}{\varrho})\hspace{-0.1em}\right\}\hspace{-0.2em}\log(\nicefrac{1}{\delta})}\hspace{-0.2em}\right)\hspace{-0.25em}\right).$ Substituting this,
$\overline{m}\hspace{-0.25em}\leq\hspace{-0.25em}O\hspace{-0.2em}\left(\hspace{-0.2em}\chi\hspace{-0.2em}\left(\hspace{-0.25em}\sqrt{\lambda_{\max}^{X}}\hspace{-0.2em}+\hspace{-0.2em}\frac{T\max\left\{\log(\nicefrac{T}{\delta}),\log(\nicefrac{T}{\varrho})\hspace{-0.1em}\right\}}{\eps}\hspace{-0.2em}\right)\hspace{-0.25em}\right)$, and noting that the asymptotics are unchanged under $\lambda_{\min}^{X}\hspace{-0.25em}<\hspace{-0.25em}\gamma(\hspace{-0.1em}1\hspace{-0.1em}+\hspace{-0.1em}2\overline{m})$ completes the proof. \hfill \qedsymbol

\subsubsection{Runtime Argument:}
\label{app:proof_thm_linear_runtime}
The algorithm runtime is dominated by $T$ steps of \FastMix (splitted between \texttt{MultipleFastMixing} and Line.~\ref{line:fast-sketch-X}) and $T$ times of solving a linear regression problem of size $k_1\times d$ (Line.~\ref{line:update_step}). The complexity of the former was established in \theoremref{thm:Privacy_First}, and the complexity of the latter is given by $O\left(k_1d^2 + d^3\right)$ via standard arguments. Summing these together yields 
\begin{equation}
    \label{eq:runtime_asymptotic}
    O\left(T\cdot \left(k_1k_2d + nd\log(n) + k_1d^2 + k_2d^2 + d^3\right)\right).
\end{equation}
Substituting the previous values of $k_1, k_2 $ and $T^{\star}$ yields the stated guarantee. \hfill \qedsymbol
\section{Coherence for SRHT}
\label{app:srht_coherence}
Recall that for SRHT $\norm{S_{\mathrm{f}}e_i}^2 = 1$ for all $i\in[n]$. Then, $\widehat{\Delta}$ (formally defined in Line~\ref{line:aux_computes_FastMix}) is the \emph{coherence} of $S_{\mathrm{f}}$. By Cauchy-Schwarz and the normalization of the columns of $S_{\mathrm{f}}$, it can be shown to be upper-bounded by $1$. Moreover, by Welch's lower bound \citep[Theorem~1]{welch1974lower} (see also the discussion in \citet[Section~4.1]{haikin2021asymptotic}),
\begin{equation}
    \widehat{\Delta} \ge \sqrt{\frac{n-k}{k(n-1)}}.
\end{equation}
In particular, when $\frac{n}{k} \gg 1$, the optimal coherence scale permitted by Welch's bound is $\Theta\left(k^{-\nicefrac{1}{2}}\right)$, while for $n\approx k+1$ it might get as small as $\Theta\left(k^{-1}\right)$. For the SRHT, as shown in \appendixref{app:proof_fihm}, w.p. at least $1-\varrho$ there exists constants $(C, c, \bar{C})$ such that $\widehat{\Delta}\le \chi$ for every $\chi\in(0,1]$ satisfying
\begin{equation}
    k\chi^2 \ge C\left(2 + \log\left(\frac{cn^2}{\chi\varrho}\right)\log\left(\frac{\bar c n^2}{\varrho}\right)\right).
\end{equation}
Hence, the resulting upper bound is of the form $\widehat{\Delta}=\widetilde{O}\left(\frac{\log\left(\nicefrac{n}{\varrho}\right)}{\sqrt{k}}\right)$; Relative to Welch's lower bound, the SRHT coherence bound degrades by a factor of roughly $\log(\nicefrac{n}{\varrho})$ whenever $n\gg k$. 
\section{Algorithms}
\label{app:algos}

\begin{algorithm}[H]
    \caption{Multiple Fast Mixing}
    \label{alg:gauss_fast_mix_eta} 
    \begin{algorithmic}[1]
        \Require Dataset $X\in\reals^{n\times d}$, row norm bound $\textsc{C}_X$, parameters $\gamma,\tau,\omega, k, T$.
        \State Set $\eta = 0$.
        \For {$t=0,1,\ldots, T-1$}: 
            \State Sample $\mathsf{S}_{\mathrm{H}, t}\simiid \mathrm{SRHT}(k, n)$. 
            \State Compute $\widetilde{X}_t \gets \mathsf{S}_{\mathrm{H}, t}X$.
            \State Compute $\widehat{m}_t \hspace{-0.15em}\gets\hspace{-0.15em} \underset{i\in[n]}{\max} \norm{\widetilde{X}^{\top}_t\hspace{-0.15em}\mathsf{S}_{\mathrm{H}, t}e_i \hspace{-0.15em}-\hspace{-0.15em} x_i}, \widehat{\Delta}_t \hspace{-0.15em}\gets\hspace{-0.15em}  \textsc{C}_X \hspace{-0.15em}\cdot\hspace{-0.15em} \underset{i,j: \ i\ne j}{\max} \abs{e^{\top}_i\mathsf{S}^{\top}_{\mathrm{H}, t}\mathsf{S}_{\mathrm{H}, t}e_j}, \widehat{\lambda}_t \hspace{-0.15em}\gets\hspace{-0.15em} \lambda^{\widetilde{X}_t}_{\min}$. 
            \State Compute $\widetilde{m}_t \hspace{-0.15em}\gets\hspace{-0.15em}\max\left\{0, \widehat{m}_t \hspace{-0.15em}+\hspace{-0.15em} \omega \widehat{\Delta}_t (\tau \hspace{-0.15em}-\hspace{-0.15em} \mathsf{z}_{1,t})\right\}$ for $\mathsf{z}_{1,t} \hspace{-0.15em}\sim\hspace{-0.15em} \mathrm{Lap}(0, 1)$. 
            \State Compute $\widetilde{\lambda}_t \hspace{-0.15em}\gets\hspace{-0.15em} \max\left\{0, \widehat{\lambda}_t \hspace{-0.15em}-\hspace{-0.15em} \omega\textsc{C}_X \left(\textsc{C}_X + 2\widetilde{m}_t\right) \left(\tau \hspace{-0.15em}-\hspace{-0.15em} \mathsf{z}_{2,t}\right)\right\}$ for $\mathsf{z}_{2,t} \hspace{-0.15em}\sim\hspace{-0.15em} \mathrm{Lap}(0, 1)$.
            \State Set $\widetilde{\eta}_t \gets \sqrt{\max\left\{0, \gamma\textsc{C}_{X}\left(\textsc{C}_{X} + 2\widetilde{m}_t\right) - \widetilde{\lambda}_t\right\}}$.
            \State If $\widetilde{\eta}_t \geq \eta$ update $\eta \gets \widetilde{\eta}_t$. 
        \EndFor 
        \State \textbf{Output:} $\eta, \ \{\widetilde{X}_{t}\}^{T-1}_{t=0}$. 
    \end{algorithmic}
\end{algorithm}

\section{Experimental Details}
\label{app:exp_details}
The experiments were run on 12th Gen Intel(R) Core(TM) i7-1255U and implemented in Python 3.14. We used five real datasets and two synthetic datasets. The plots throughout the paper are for the normalized (divided by $n$) excess empirical risks, as they match the performance metric used in our accuracy theorem. In all cases, we normalized the training data so that the maximum \( \ell_2 \)-norm of any training sample was 1 (namely, $\norm{x_i}^2 \leq 1$ and $\abs{y_i} \leq 1$ for all $i =1,\ldots, n$, so $\textsc{C}_X = \textsc{C}_Y = 1$). Moreover, to allow compatibility with the Hadamard transform, we subsample (randomly, without replacement) the datasets to a number of samples which is a power of $2$. The airline dataset was subsampled to $4194304$ points, to allow for reasonable solution times. For performing the Hadamard transform, we have used the optimized implementation from \url{https://github.com/falconn-lib/ffht}.

We report the excess empirical risk. First, we calculate the normalized mean squared error (MSE) for the train set, computed as the squared error in predicting \( y_i \) via \( x_i^{\top} \widehat{\theta} \), averaged over the train set, and we call this quantity train MSE. The excess empirical risk is defined as the difference between this quantity and the error attained by the optimal solution $\theta^{*}$. The results are averaged over $30$ independent trials, and we report both the empirical means and 95\% confidence intervals, calculated via $\pm 1.96 \cdot \frac{\mathrm{std}}{\sqrt{\# \mathrm{runs}}}$. For running times, we measure wall-clock time via the function \texttt{perf\_counter\_ns} from Python's \texttt{time} library, and report the average across the experiments as well as $95$\% confidence intervals.  

Throughout the experiments, we fixed the failure probability to $\varrho = \nicefrac{\delta}{10}$, so the system failure probability is only slightly larger than the additive slack in the DP definition, $\delta$. Moreover, we set $\delta = \frac{1}{n^2}$. For IHM, we used $k = 6 \cdot \max\left\{d, \log(\nicefrac{4T}{\varrho})\right\}, T=4$ and $\textsc{C} = \textsc{C}_Y$, following similar choices in \citet{lev2026near}, where since the comparison for \FastIHM is done for relatively higher range of $\eps$ we have increased $T$ from $3$ as was simulated in \citet{lev2026near} to $T=4$. For FastIHM, we used $k_1 = 6 \cdot \max\left\{d, \log(\nicefrac{4T}{\varrho})\right\}, T=4$ and $\textsc{C} = \textsc{C}_Y$.
For both schemes, we omitted the term $-\eta^2\widehat{\theta}_t$ from the computation of $\widetilde{G}_t$, as it improves performance by making the iterates converge closer to $\theta^{*}$ than to $\theta^{*}(\eta^2)$. This does not affect the privacy analysis, so we omit this term in the empirical evaluation. Whenever we used the Gaussian mechanism, we computed the noise level using the analytic, tight formula from \citet{balle2018improving}.

Our synthetic datasets aim to address two edge dataset settings. The first case aims to examine the performance of our algorithm in a setting where $\lambda^{X}_{\min} \gg 1$, and under which, as predicted by \theoremref{thm:main_fihm}, the \FastIHM should perform similarly to the IHM. To that end, we generate covariates that are distributed \iid \ and uniformly on the $ d$-dimensional sphere. In that case, we note that 
\begin{equation}
    X^{\top}X = \sum^{n}_{i=1}x_ix^{\top}_i \approx \frac{n}{d}\brI_d
\end{equation}
so $\lambda^{X}_{\min} \approx \frac{n}{d}$, which corresponds to its maximal attainable value since for any matrix $X$ the next bounds hold $d\cdot\lambda^{X}_{\min} \leq \mathrm{Tr}(X^{\top}X) \leq n$. Then, we generate responses according to the linear model 
\begin{equation}
    \label{eq:linear_model_synthetic_data}
    y_i = x^{\top}_i \theta_0 + \sigma \mathsf{z}_i
\end{equation}
where $\theta_0$ is a randomly generated unit vector, $\mathsf{z}_i \simiid \pN(0, 1)$ and $\sigma^2 = 0.1$. The second dataset aims to test the performance of the algorithm in a poorly conditioned setting (termed \emph{synthetic correlated}), where usually the computational advantage of the IHS becomes prominent (see, for example, experiments done in \citet{pilanci_hessiansketch, lacotte2020optimal}). To that end, we generate a design matrix built from vectors sampled as $\sx_i \simiid \pN(\vec{0}_d, \Sigma)$ and where $\Sigma_{ij} = 2\cdot (0.99)^{\abs{i-j}}$ and responses from the same linear model \eqref{eq:linear_model_synthetic_data}. For both cases, we have fixed $n= 2^{19}$ and $d = 2^{5}$.
\section{Additional Experiments}
\label{app:additional_exps}
We now provide additional experiments for a setting where the ratio $\nicefrac{n}{d}$ is extremely large, representing a setting where sketching is most beneficial in terms of runtime. In particular, we use the airline dataset \citep{openml_airlines_depdelay_10m_42728}, subsampled to $n=2^{22}$ samples, leads to $\nicefrac{n}{d} = \nicefrac{2^{22}}{4} = 2^{20}$. We sweep across values of $\frac{k_2}{\max\left\{d, \log(\nicefrac{4T}{\varrho})\right\}} \in [4, 1200]$ and fix $T = 4, k_1 = 6\cdot \max\left\{d, \log(\nicefrac{4T}{\varrho})\right\}$ and $k = 6\cdot \max\left\{d, \log(\nicefrac{4T}{\varrho})\right\}$. As seen in \figureref{fig:simulations_airline}, in this regime, \FastIHM allows for attaining close-to-optimal performance at a significantly lower runtime, demonstrating the most favorable regime for sketching.  

\begin{figure}[!htbp]
    \centering
    \caption*{$\lambda^{X}_{\max} \approx 356246,\ \lambda^{X}_{\min} \approx 0.685$}
    \includegraphics[width=0.5\linewidth]{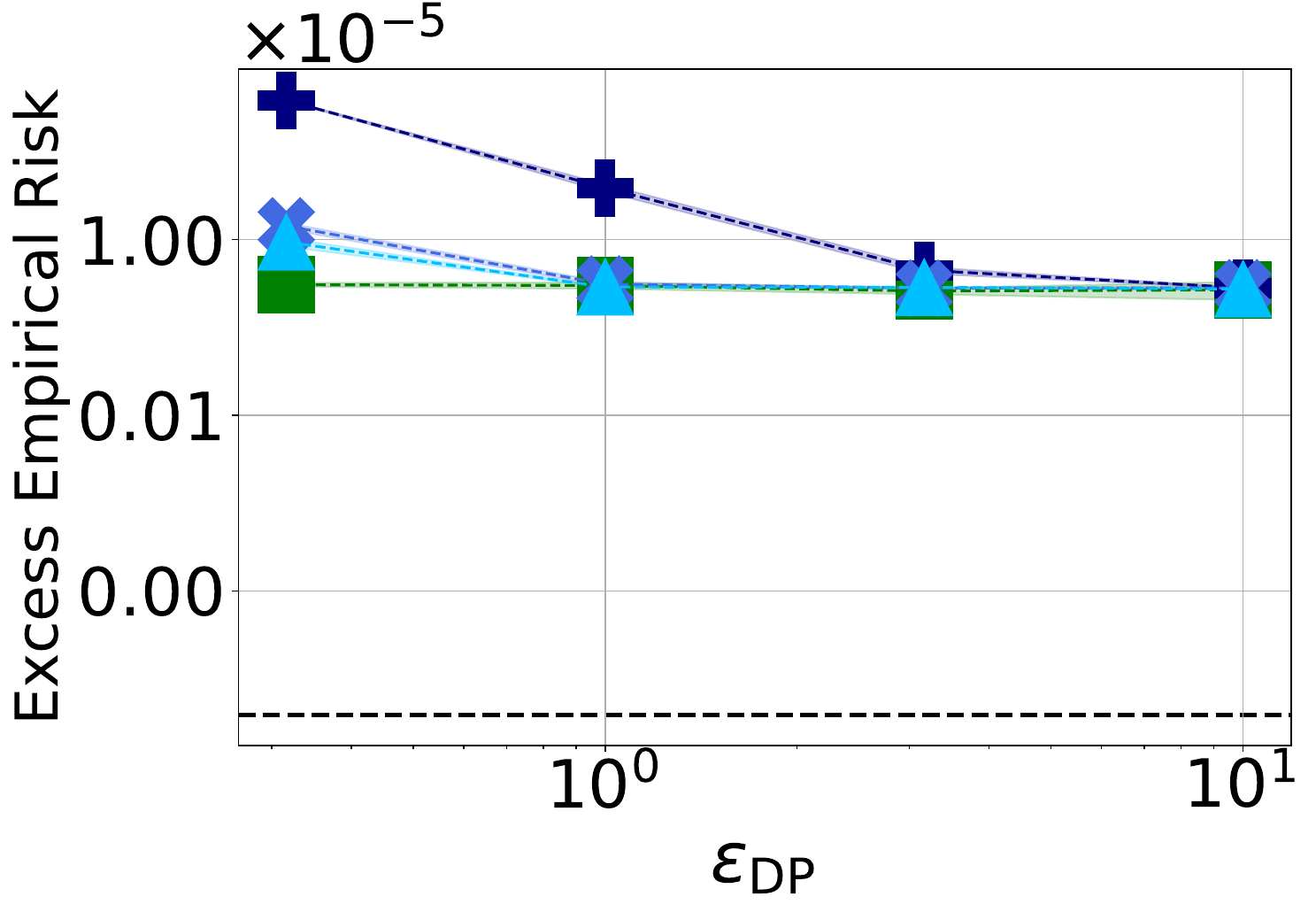}
    {\footnotesize
        \setlength{\tabcolsep}{2pt}%
        \renewcommand{\arraystretch}{0.9}%
        \begin{tabular*}{\linewidth}{@{\extracolsep{\fill}} c c c@{}}
        \textcolor{FastIHMColorA}{\rule{12pt}{1.6pt}$\pmb{+}$}
        $\mathrm{FastIHM}\hspace{-0.2em}:\hspace{-0.2em}
        k_2 = 4 \cdot \max\{d,\hspace{-0.1em}\frac{4T}{\varrho}\hspace{-0.1em}\}$
        &
        \textcolor{FastIHMColorB}{\rule{12pt}{1.6pt}$\pmb{\times}$}
        $\mathrm{FastIHM}\hspace{-0.2em}:\hspace{-0.2em}
        k_2 = 600 \cdot \max\{d,\hspace{-0.1em}\frac{4T}{\varrho}\hspace{-0.1em}\}$
        &
        \textcolor{FastIHMColorC}{\rule{12pt}{1.6pt}$\bm{\blacktriangle}$}
        $\mathrm{FastIHM}\hspace{-0.2em}:\hspace{-0.2em}
        k_2 = 1200 \cdot \max\{d,\hspace{-0.1em}\frac{4T}{\varrho}\hspace{-0.1em}\}$
        \\[-0.125em]
        \multicolumn{3}{@{}c@{}}{%
        \begin{tabular*}{0.36\linewidth}{@{\extracolsep{\fill}} c c@{}}
        \textcolor{Green}{\rule{12pt}{1.6pt}$\blacksquare$}$\ \mathrm{IHM}$
        &
        \textcolor{Black}{\hdashrule[0.5ex]{12pt}{0.9pt}{1.5pt 1pt}}%
        $\ \mathrm{FastIHS}: k = 6\cdot \max\{d,\frac{4T}{\varrho}\}$
        \end{tabular*}
        }
        \end{tabular*}
    }

    {
    \centering
    \footnotesize
    \setlength{\tabcolsep}{4.2pt}
    \renewcommand{\arraystretch}{1.03}
    \begin{adjustbox}{max width=\textwidth}
    \begin{tabular}{l
                    c
                    S[table-format=1.3]
                    S[table-format=1.3]
                    S[table-format=1.3]
                    S[table-format=1.3]}
    \toprule
    
    & {$(r_1, r_2, r_3)$}
    & {Fast IHS}
    & {Fast IHM: $r_1$}
    & {Fast IHM: $r_2$}
    & {Fast IHM: $r_3$} \\
    \midrule
   Airline
    & $(4,600,1200)$ & 41.95 & 16.71 & 15.26 & 14.18 \\
    \bottomrule
    \end{tabular}
    \end{adjustbox}
   }
    \caption{Accuracy-runtime trade-off for the Airline dataset.}
    \label{fig:simulations_airline}
\end{figure}\hspace{0.01\textwidth}%
\section{Dataset Statistics}
\label{app:statistics}
The next table summarizes key statistics about the different datasets used for conducting the simulations in \sectionref{s:sims}. 
\begin{table}[H]
  \centering
  \small
  \setlength{\tabcolsep}{6pt}
  \renewcommand{\arraystretch}{1.2}
  \resizebox{\linewidth}{!}{\footnotesize 
  \begin{tabular}{|l|c|c|c|c|c|c|c|c|c|}
    \toprule
    \textbf{Dataset} &
    $n$ &
    $d$ &
    $\lambda^{X}_{\min}$ &
    $\lambda^{X}_{\max}$ &
    $\frac{1}{n}\hspace{-0.2em}\norm{Y \hspace{-0.2em}-\hspace{-0.2em} X\theta^{*}}^2$ &
    $\frac{1}{n}\hspace{-0.2em}\norm{Y}^2$ & 
    $\frac{n}{d}$\\
    \midrule
    Beijing \citep{Beijing_dataset}  & 16384 & 9 & {\small \(< 10^{-3}\)} & 4435.30 & 0.0006 & 0.015 & 1310.7\\ \hline
    YearMSD \citep{year_prediction_msd_203} & 131072 & 90 & {\small \(< 10^{-3}\)} & 464.78 & 0.004 & 0.988 & 1456.3 \\ \hline
    Black Friday \citep{openml_black_friday_41540} & 131072 & 81 & {\small \(< 10^{-3}\)} & 40022.54 & 0.022 & 0.281 & 1618.2 \\ \hline
    Synthetic: $\lambda^{X}_{\min} \approx \frac{n}{d}$ & 524288 & 32 & 16166.42 & 16637.76 & 0.034 & 0.045 & 16384 \\ \hline
    Synthetic: Correlated & 524288 & 32 & 3.27 & 18723.40 & 0.0006 & 0.039 & 16384 \\ \hline
    Rossman \citep{rossmann_store_sales_kaggle} & 524288 & 28 & {\small \(< 10^{-3}\)} & 8390.50 & 0.0249 & 0.265 & 18724.5 \\ \hline
    Airline \citep{openml_airlines_depdelay_10m_42728} & 4194304 & 4 & 0.685 & 356246.3 & 0.005 & 0.006 & 1048576 \\ \hline
    \bottomrule
  \end{tabular}
  }

  \caption{Key parameters from the datasets simulated in this work.}
  \label{tab:dataset_stats_app}
\end{table}

\end{document}